\documentclass[lettersize,onecolumn]{IEEEtran}
\usepackage{algorithmic}
\usepackage{algorithm}
\usepackage{array}
\usepackage{textcomp}
\usepackage{stfloats}
\usepackage{url}
\usepackage{verbatim}
\IEEEoverridecommandlockouts
\usepackage{amsmath,amsfonts,amsthm,amssymb,bbm}
\usepackage{cite}

\usepackage{balance}
\usepackage{mathrsfs}
\usepackage{xcolor}
\usepackage{tikz}
\usepackage{pgfplots}
\usepackage{colortbl}
\usepackage{graphicx}
\usepackage{svg}
\graphicspath{figures}
\usetikzlibrary{positioning,quotes,angles,patterns,intersections,pgfplots.fillbetween,calc}
\usepackage{tkz-euclide}
\usepackage{mathtools}
\usepackage{url}
\usetikzlibrary{decorations.pathmorphing}
\usepackage{enumitem}
\usepackage{stfloats}
\pgfplotsset{compat=1.17} 
\usepackage{hyperref}
\usepackage{cleveref}
\usepackage{comment}
\usepackage{multicol}
\usepackage{subcaption}

\def \extended {1}

\usepackage[normalem]{ulem}
\DeclareMathOperator*{\argmin}{arg\,min}

\theoremstyle{plain} 
\newtheorem{theorem}{Theorem}
\newtheorem{corollary}{Corollary}

\newtheorem{prop}{Proposition}
\theoremstyle{definition} \newtheorem{remark}{Remark}
\theoremstyle{definition} 
\usetikzlibrary{calc,shapes.geometric}

\crefname{assumption}{Assumption}{Assumptions}
\crefname{remark}{Remark}{Remarks}
\crefname{prop}{Proposition}{Propositions}
\crefname{theorem}{Theorem}{Theorems}
\crefname{corollary}{Corollary}{Corollaries}
\crefname{section}{Section}{Sections}
\crefname{appendix}{Appendix}{Appendices}
\crefname{equation}{Equation}{Equations}
\crefname{figure}{Figure}{Figures}
\crefname{lemma}{Lemma}{Lemmas}
\crefname{table}{Table}{Tables}

\begin{document}

\title{Lower Bounds on the MMSE of Adversarially Inferring Sensitive Features} 



\author{%
  \IEEEauthorblockN{Monica Welfert\IEEEauthorrefmark{1}, Nathan Stromberg\IEEEauthorrefmark{1}, Mario Diaz\IEEEauthorrefmark{2}, and Lalitha Sankar\IEEEauthorrefmark{1}}
  \\\IEEEauthorblockA{\IEEEauthorrefmark{1}Arizona State University, USA, \IEEEauthorrefmark{2}IIMAS, Mexico\\
    Email: \{mwelfert, nstrombe, lsankar\}@asu.edu, mario.diaz@sigma.iimas.unam.mx}
    \thanks{Part of this paper has been accepted for presentation at ISIT 2025.}
\thanks{Manuscript received October 26, 2023; revised December 8, 2023.}
}

\markboth{Journal of \LaTeX\ Class Files,~Vol.~1, No.~2, April~2025}%
{Shell \MakeLowercase{\textit{et al.}}: A Sample Article Using IEEEtran.cls for IEEE Journals}

\IEEEpubid{0000--0000~\copyright~2025 IEEE}

\maketitle

\begin{abstract} 
We propose an adversarial evaluation framework for sensitive feature inference based on minimum mean-squared error (MMSE) estimation with a finite sample size and linear predictive models. Our approach establishes theoretical lower bounds on the true MMSE of inferring sensitive features from noisy observations of other correlated features. These bounds are expressed in terms of the empirical MMSE under a restricted hypothesis class and a non-negative error term. The error term captures both the estimation error due to finite number of samples and the approximation error from using a restricted hypothesis class. For linear predictive models, we derive closed-form bounds, which are order optimal in terms of the noise variance, on the approximation error for several classes of relationships between the sensitive and non-sensitive features, including linear mappings, binary symmetric channels, and class-conditional multi-variate Gaussian distributions. We also present a new lower bound that relies on the MSE computed on a hold-out validation dataset of the MMSE estimator learned on finite-samples and a restricted hypothesis class. Through empirical evaluation, we demonstrate that our framework serves as an effective tool for MMSE-based adversarial evaluation of sensitive feature inference that balances theoretical guarantees with practical efficiency.
\end{abstract}


\section{Introduction}
\IEEEPARstart{E}{stimating} the minimum mean-squared error (MMSE) of one random variable given another is a fundamental problem in information theory, signal processing, communications, and machine learning. In adversarial settings, the MMSE quantifies the extent to which an adversary can infer a sensitive variable $S$ from an observed feature vector $X$. Understanding these inference capabilities is crucial in assessing the vulnerabilities of shared data representations.

Consider a healthcare company that collects patient data (see \cref{fig:adversarial-evaluation}), including demographics, medical history, and genetic information. Due to regulatory and ethical considerations, the company wants to release a dataset containing only a subset of features while ensuring that a sensitive attribute—such as a genetic predisposition to a disease ($S$)—cannot be easily inferred from the released data. To achieve this, they introduce noise into the raw feature set $X \in \mathbb{R}^d$, yielding $X^\sigma \in \mathbb{R}^d$, before making it available for research or machine learning applications. However, an adversary with access to $X^\sigma$ may still attempt to infer $S$, raising the question: how effectively can an adversary reconstruct the sensitive feature despite the noise? More generally, $X$ can be statistics (e.g., from the US Census) or low-dimensional (e.g., last layer of a deep model) representations of user data; on the other hand, $S$ can be legally protected features, or the presence or absence of individuals in a dataset.  

To address this concern, the company assembles an internal red team, tasked with 
evaluating if the release $X^\sigma$ is informative about $S$.
Such a red team employs a predictive learning model trained on the available $(X^\sigma,S)$ pairs provided to it to estimate the extent to which $S$ can still be inferred. 
Their goal is to establish guarantees on how well an adversary can infer $S$ from $X^\sigma$, providing insight into whether the current noise mechanism effectively reduces adversarial inference.

In general, one cannot and should not make assumptions on adversarial capabilities. Thus, the aforementioned red team should not be viewed as an adversary, rather as a mechanism to provide guarantees on the performance of the strongest adversary. 
To this end, we consider MMSE as a metric for adversarial inference and propose a framework to obtain non-vacuous lower bounds on the inference possible by the strongest adversary, where the latter is capable of implementing any function and knows the joint distribution of the sensitive and disclosed data. 
Such lower bounds are obtained by the red team using a limited number of 
$(X^\sigma,S)$ example pairs (provided by the healthcare company) and using a restricted hypothesis class for predicting $S$ from $X^\sigma$. In fact, there are two possible ways the red team can obtain these lower bounds: (i) using training MSE loss, and (ii) using validation MSE. We explore both in this paper. Our bounds explicitly depend on both the strength of the hypothesis class and the number of samples.

Our primary contribution is establishing lower bounds on the true MMSE of $S$ given the noisy observed $X^\sigma$ when the red team has access to a restricted hypothesis class $\mathcal{H}$, a finite number of training and validation samples $n$, and $m$, respectively. The resulting bounds take the following form:
\begin{equation}    \mathrm{MMSE}(S \vert X^{\sigma}) \geq \mathrm{MSE}(\hat{h}) - \epsilon^\mathcal{H}_{n,m},
\label{eq:general-mmse-lower-bound}
\end{equation}
where $\mathrm{MSE}(\hat{h})$ is the empirical mean-squared error (MSE) computed using samples from either a training set or validation set and an empirically trained model $\hat{h}$ from a hypothesis class $\mathcal{H}$, and
$\epsilon^\mathcal{H}_{n,m}$ is a positive number depending on the size of the training set and/or validation set and the parameters of  $\mathcal{H}$. This work substantially extends our prior results presented in \cite{ISITextenver}.
We list below our contributions while highlighting novelty
relative to \cite{ISITextenver}:

\begin{enumerate}[leftmargin=*]
    \item We introduce a lower bound of the form in \eqref{eq:general-mmse-lower-bound} using the training MSE and decompose $\epsilon^\mathcal{H}_n$ ($m=0$) into two components: (i) a finite sample term dependent on $n$ and (ii) an approximation error term dependent on the parameters of the chosen hypothesis class $\mathcal{H}$ (\cref{thm:mmse-bound-general}). While a version of this result appeared in \cite{ISITextenver}, we strengthen the bound on the finite sample term using large deviation tools, including a new Bernstein-style bound (\cref{prop:eps-c-bounds}) that improves upon the Hoeffding-style bound used in \cite{ISITextenver}.
    \label{contribution-training-bound}
    \item Given that a large model may overfit on the training data, we introduce a novel alternative lower bound of the form in \eqref{eq:general-mmse-lower-bound} using the validation MSE rather than the training MSE and decompose $\epsilon^\mathcal{H}_{n,m}$ into two components: (i) a finite sample and generalization term dependent on $n,m$ and the complexity of $\mathcal{H}$ and (ii) the same approximation error term dependent on the parameters of $\mathcal{H}$ as in the training bound (\cref{thm:mmse-bound-general-val}). We bound the generalization component of the finite sample and generalization term using state-of-the-art (SOTA) generalization bounds \cite{wilson2025deep} (\cref{prop:eps-g-bound}). 
    \item When $\mathcal{H}$ is taken to be the hypothesis class of the sigmoid function composed with a linear function, we obtain a closed-form bound on the approximation error in terms of the statistics of $X^\sigma$ and the optimizer of the true $\mathrm{MMSE}(S|X^\sigma)$ (\cref{prop:linear-bound-epsilon-a}). While this result was presented in \cite{ISITextenver}, we include it here for completeness and now provide a full proof, which was omitted in the prior work.
    \item Under Gaussian noise and binary $S$, we explicitly compute the bound on the approximation error to highlight its dependence on the noise parameter for the following relationships between $S$ and $X$: (i) linear relationship (\cref{prop:delta-a-linear-optimal-example}), (ii)  class-conditional Gaussians with identical covariances (\cref{prop:delta-a-linear-optimal-example}), (iii) binary symmetric channel (\cref{thm:delta-a-bsc-example}), and (iv) class-conditional Gaussians with different covariances (\cref{thm:CCG-vector-general} and  \cref{cor:CCG-vector-diag-cov}). These results were originally presented in \cite{ISITextenver} and are included here for completeness. In contrast to the prior version, we now provide full proofs. Additionally, we extend \cref{cor:CCG-vector-diag-cov} to make explicit the dependence of the  approximation error bound on the dimension of the data.
    \item  {Our final contribution significantly expands upon the empirical illustrations presented in \cite{ISITextenver} by incorporating two additional experimental settings:

(i) Class-conditional Gaussian setting, where we examine the effect of the feature dimension and  sample size on the tightness of our bounds, and also evaluate the performance of neural networks as estimators; and

(ii) Class-conditional Gaussian mixture setting, where we compare the MMSE lower bounds obtained from linear models and neural networks as mixture complexity increases, measured by the number of Gaussian modes per class. We also evaluate the performance of bounds derived from training MSE (or empirical MMSE) versus those based on validation MSE in this setting.


} 
\end{enumerate}

Following each result, we provide either a proof sketch or a complete proof, as appropriate; detailed proofs and additional experimental results can be found in the appendices.

\begin{figure}
    \centering
    \includegraphics[width=0.5\linewidth]{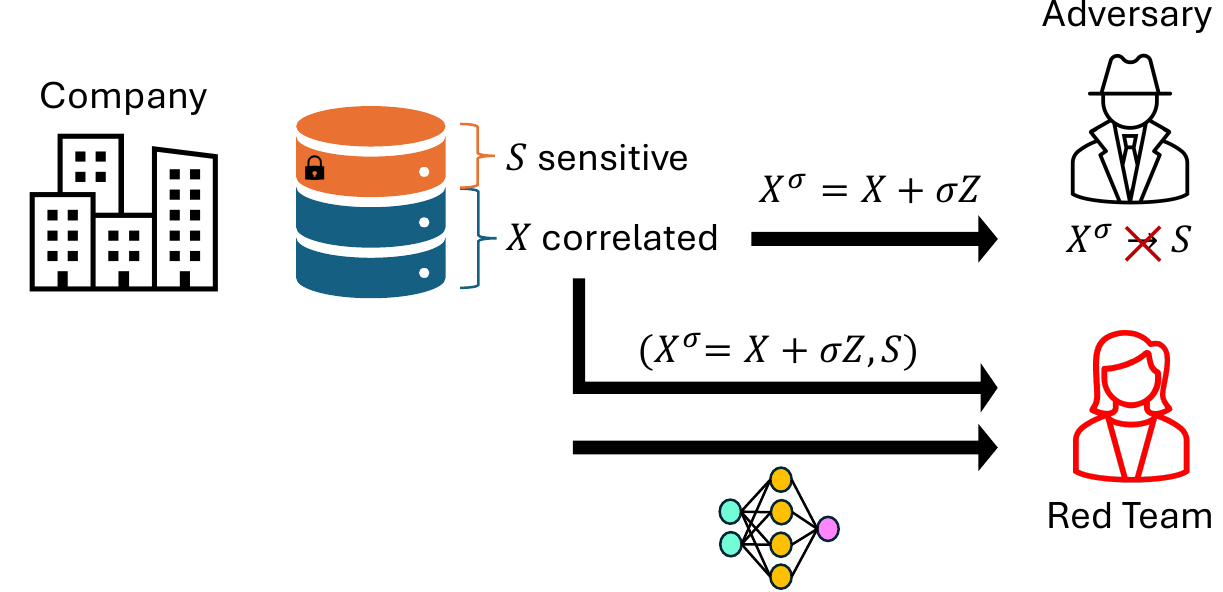}
    \caption{A company aims to release a sanitized version of user data by adding noise to the original features $X$, yielding $X^\sigma$, while protecting a sensitive attribute $S$. To evaluate the risk of adversarial inference, an internal red team is tasked with estimating how well $S$ can be predicted from $X^\sigma$ using provided samples of $(X^\sigma,S)$ pairs and a learning model, providing guarantees on the effectiveness of the noise mechanism.}
    \label{fig:adversarial-evaluation}
\end{figure}

\subsection{Related Work} 

The MMSE achievable by optimal estimation of a random variable given another one plays a key role in statistics and communications \cite{scharf1991statistical,biglieri2007mimo}, and it is closely related to fundamental information-theoretic concepts \cite{guo2005mutual}. As such, the problem addressed in this paper is related to other fundamental problems in information theory and statistics. In the special case when $S = X$, i.e., estimating $X$ from $X^\sigma$, our setting is closely related to the problem of estimation in Gaussian channels as studied in \cite{guo2005mutual,guo2011estimation,wu2011functional}. Zieder \emph{et al.} \cite{ziederpoincare2022} developed lower bounds on the MMSE for a class of additive noise channels from $X$ to $X^\sigma$, including Gaussian, in terms of Poincaré constants. Klartag and Ordentlich \cite{Ordentlich2025} later developed a simpler bound for the Gaussian additive channel also in terms of Poincaré constants. In \cite{jeong2025}, Jeong \emph{et al.} extend and unify different versions of the classical Ziv-Zakai bound. In constrast to these approaches, we problem we consider here differs in two ways: (i) by considering finite samples to lower bounds, and (ii) requires estimating another variable $S \neq X$.

Recent results have focused on methods to estimate MMSE, including using linear \cite{scharf1991statistical}, kernel-based \cite{peinado2017statistical} and polynomial approximations of the conditional expectation~\cite{alghamdi2019mutual,alghamdi2021polynomial,Alghamdi2024-measuring-information}. These works are closely related to ours; 
in fact, a key component of our lower bound includes a term that we refer to as the \textit{approximation error} for any chosen hypothesis class relative to the true estimator in an MMSE sense. The abovementioned papers focus on approximating the true estimator or obtaining bounds on our approximation error. In this paper, we first present upper bounds on this error for specific $(X,S)$ distributions and later consider a general empirical methodology. 




Our work is motivated by the problem of obfuscating sensitive features by releasing noisy non-sensitive features. It is worth mentioning that while we use this obfuscation as a motivation, our results are broad and apply to a variety of communications and signal processing settings where MMSE of an unknown feature needs to be estimated from observed data. That said, in the context of obfuscation or privacy, there are several notions of data privacy designed to capture the risks posed by a variety of adversaries. For example, differential privacy quantifies the membership inference capabilities of an adversary when making queries on a dataset \cite{kairouz2017composition}, maximal $\alpha$-leakage quantifies the capacity of an adversary to infer any (randomized) function of the private attribute \cite{issa2017operational,liao2019tunable}, probability of correctly guessing quantifies the probability of an adversary to guess the private attribute \cite{asoodeh2018estimation}, to name a few notions. In these settings, our work is broadly related to the problem of data-driven estimation of information (leakage) measures, see, e.g., \cite{shamir2010learning,jiao2015minimax,wu2016minimax,diaz2019robustness,belghazi2018mutual,chan2019neural,sreekumar2021non} and references therein. There is also a recent research effort dedicated to understanding the performance of an adversary with practical computational capabilities \cite{jagielski2020auditing,nasr2021adversary}. From this perspective, our results compare the MMSE performance of a finite-capacity predictive model using a restricted hypothesis class and finitely many samples to that of the strongest adversary capable of implementing any function to predict $S$ from $X$ and with knowledge of the joint distribution of the private and disclosed data.



Finally, for the problem of estimating MMSE, existing approaches rely on expressive but computationally intensive models, such as neural networks \cite{diaz2021neural}. However, precise quantification of the bounds using neural networks requires bounding a quantity called the Barron's constant \cite{barron1993universal}, which is generally infeasible to compute. Our approach using simpler computable models for adversarial evaluation (e.g., linear models) circumvents this issue. 


\subsection{Outline}
The remainder of the paper is organized as follows. In \cref{sec:setup}, we set up the problem including relevant definitions. In  \cref{section:mainresults}, we present general lower bounds on $\mathrm{MMSE}(S\vert X^\sigma)$ 
following which, in  \cref{sec:linear-eps-a-bounds}, we present upper bounds on the $\epsilon^\mathcal{H}_{n,m}$ term in \eqref{eq:general-mmse-lower-bound} for specific $(X,S)$ distribution models and linear hypothesis class. Finally, in  \cref{sec:empirical}, we illustrate our results and conclude in  \cref{sec:conclusion}.

\section{Problem Setup}
\label{sec:setup}
Given jointly distributed random variables $X\in\mathbb{R}^{d}$ and $S\in\mathbb{R}$, the minimum mean square error in estimating $S$ given $X$ is defined as
\begin{equation}
\label{eq:DefMMSE}
    \mathrm{MMSE}(S \vert X) \coloneqq \inf_{h \textnormal{ meas.}} \mathbb{E}\left[(S-h(X))^{2}\right],
\end{equation}
where the infimum is taken over all (Borel) measurable functions $h:\mathbb{R}^{d}\to\mathbb{R}$. The infimum in \eqref{eq:DefMMSE} is attained by the conditional expectation of $S$ given $X$, i.e.,
\begin{equation}
\label{eq:MMSEMinimizer}
    \mathrm{MMSE}(S \vert X) = \mathbb{E}\left[(S-\eta(X))^{2}\right],
\end{equation}
where $\eta(X) \stackrel{\textnormal{a.s.}}{=} \mathbb{E}\left[S \vert X\right]$. Note that if $S \stackrel{\textnormal{a.s.}}{=} h_{0}(X)$ for some function $h_{0}:\mathbb{R}^{d}\to\mathbb{R}$, then $\mathrm{MMSE}(S \vert X) = 0$. Also, note that if $X$ and $S$ are independent, then the MMSE is maximal and $\mathrm{MMSE}(S \vert X) = \mathbb{E}\left[(S-\mathbb{E}[S])^{2}\right]$.

There are many metrics for adversarial evaluation. We motivate the use of MMSE in this adversarial context by the following simple example. Consider a binary sensitive feature $S\in\{0,1\}$ and let $h: \mathbb{R}^d \rightarrow \{0,1\}$ be a predictor of $S$ given $X$; for this case, we can show that the MMSE serves as a lower bound on the probability of error in estimating $S$ from $X$ as follows:
\begin{align*}
    P_{\textnormal{error}}(S \vert X) &= \inf_{h:\mathbb{R}^{d}\to\{0,1\}} \mathbb{E}\left[\mathbbm{1}_{S \neq h(X)}\right] \nonumber\\
    &= \inf_{h:\mathbb{R}^{d}\to\{0,1\}} \mathbb{E}\left[(S-h(X))^{2}\right] \nonumber\\
    &\geq \inf_{h \textnormal{ measurable}} \mathbb{E}\left[(S-h(X))^{2}\right] \nonumber\\
    \label{eq:PerrorMMSE} &= \mathrm{MMSE}(S \vert X).
\end{align*}
Thus, for binary $S$, any lower bound on $\mathrm{MMSE}(S \vert X)$ gives rise to a lower bound on $P_{\textnormal{error}}(S \vert X)$. This observation further illustrates the importance of studying lower bounds for the MMSE since accuracy ($1-P_{\textnormal{error}}$) is an oft-used measure in machine learning (ML), including adversarial ML. 

Assume that $X\in\mathbb{R}^{d}$ are usable features/statistics that are correlated with a sensitive attribute $S\in[0,1]$. For example, in a healthcare company, $X$ could represent patient biomarkers while $S$ corresponds to a protected health status (e.g., presence of a genetic disorder). 
To limit the risk of sensitive feature inference, the company may choose to release only a \emph{sanitized} version of $X$. In precise terms, we have a Markov chain $S- X - X^\sigma$ where $S$ represents sensitive information, $X$ represents non-sensitive information, and $X^\sigma$ is a noisy version of $X$. In this work, we focus on the additive Gaussian mechanism, a popular sanitization method in privacy-aware data sharing, frequently studied in the information-theoretic and differential privacy literature, e.g., \cite{dwork2014algorithmic,asoodeh2016privacy,asoodeh2016information}. Given $\sigma>0$, we define
\begin{equation}
    X^{\sigma} \coloneqq X + \sigma Z,
\end{equation}
where $Z \sim \mathcal{N}(0,I)$ is independent of $X$ and $S$.

We define $\eta^\sigma$ as the conditional expectation of $S$ given $X^\sigma$, i.e.,
\begin{equation}
\label{eq:DefEta-general}
    \eta^\sigma(x) \coloneqq \mathbb{E}[S\vert X^\sigma=x], \quad x \in \mathbb R^d.
\end{equation} 
In the case when $S$ is binary, i.e., $S \in \{0,1\}$, let ${f}_{i}$ and  ${f}_{i}^{\sigma}$ be the (conditional) density of $X|S=i$ and $X^\sigma|S=i$, respectively, for $i \in \{0,1\}$. Note that  ${f}_{i}^{\sigma}(x) = (f_{i} \ast K_{\sigma})(x)$, where $\ast$ is the convolution operator and $K_{\sigma}$ is the density of $\sigma Z$. Then the conditional expectation becomes
\begin{equation}
\label{eq:DefEta}
    \eta^\sigma(x) = \frac{pf^\sigma_{1}(x) }{pf^\sigma_{1}(x) + \bar{p}f^\sigma_{0}(x)} = s(\theta^\sigma(x)), \quad x \in \mathbb R^d,
\end{equation}
where $p = \mathbb{P}(S = 1)$, $\bar{p} \coloneqq 1- p$, $s(z) = 1/(1+e^{-z})$ is the sigmoid function, and
\begin{equation}
\label{eq:DefTheta}
    \theta^\sigma(x) = \log\left(\frac{pf^\sigma_{1}(x)}{\bar{p}f^\sigma_{0}(x)} \right).
\end{equation}

To evaluate the effectiveness of this sanitization, a red team within the company is tasked with simulating an adversary attempting to infer $S$ from the noisy $X^\sigma$ using a class of learning models and a dataset of finite size.
Let $\mathcal{H}$ be the hypothesis class associated with a collection of functions $h_\mathcal{H}:\mathbb{R}^{d}\to[0,1]$. We define the MMSE restricted to the hypothesis class $\mathcal{H}$ as follows:
\begin{equation}
    \mathrm{MMSE}^\mathcal{H}(S\vert X^\sigma) \coloneqq \min_{h\in\mathcal{H}} \mathbb E[(S-h(X^\sigma))^2],
    \label{eq:mmse-restricted-h}
\end{equation}
and denote $h^*_\mathcal{H}$ as the attainer of the $\mathrm{MMSE}^\mathcal{H}(S\vert X^\sigma)$, i.e.,
\begin{equation}
    h^*_\mathcal{H} \coloneqq \argmin_{h\in\mathcal{H}} \mathbb E[(S-h(X^\sigma))^2] .
    \label{eq:opt-h}
\end{equation}
Given a dataset $\mathcal{D}=\mathcal{D}_\text{train} \cup \mathcal{D}_\text{val}$ composed of a training set $\mathcal{D}_\text{train} = \{(X^{\sigma}_{i},S_{i})\}_{i=1}^{n}$ and a validation set $\mathcal{D}_\text{val} = \{(X^{\sigma}_{i},S_{i})\}_{i=1}^{m}$ of i.i.d. samples, the empirical mean-squared error (MSE) of a certain model $h \in \mathcal{H}$ is given by:
\begin{equation}
\label{eq:emp-mse}
    \mathrm{MSE}_{y}(h) \coloneqq  \frac{1}{|\mathcal{D}_y|} \sum_{(X^\sigma,S) \in \mathcal{D}_y} (S - h(X^{\sigma}))^{2},
\end{equation}
where $y \in \{\text{train}, \text{val}\}$ and $|\mathcal{D}|$ denotes the cardinality of the set $\mathcal{D}$. We denote a trained model from a hypothesis class $\mathcal{H}$ as 
\begin{equation}
    \hat{h}_\mathcal{H} \coloneqq \argmin_{h\in\mathcal{H}} \, \mathrm{MSE}_\textup{train}(h),
    \label{eq:trained-model}
\end{equation}
i.e., $\hat{h}_\mathcal{H}$ attains the minimum MSE over the training dataset $\mathcal{D}_\text{train}$. As we wish to give guarantees on the obfuscation provided by adding Gaussian noise to $X$, our goal is therefore to establish a (probabilistic) lower bound of the form
\begin{equation}
\label{eq:MainResultPrototype}
   \mathrm{MMSE}(S \vert X^{\sigma}) \geq \mathrm{MSE}_y(\hat{h}_\mathcal{H}) - \epsilon^\mathcal{H}_{y} ,
\end{equation}
where $y \in \{\text{train},\text{val}\}$ and $\epsilon^\mathcal{H}_y$ is a positive number depending on the size of the dataset $\mathcal{D}_y$ and the parameters of the hypothesis class $\mathcal{H}$.

\noindent\textbf{Notation:} For functions, we use $\lVert \cdot \rVert_{2}$ to denote the $2$-norm with respect to the distribution of $X^\sigma$, i.e.,
$\lVert h \rVert_{2}^{2} = \mathbb{E}[ \lvert h(X^\sigma) \rvert^{2}]$. 

\section{General Lower Bounds}\label{section:mainresults}

By definition, $\mathrm{MMSE}(S \vert X^{\sigma})$ represents the smallest expected square-loss achievable by \emph{any} measurable function. This distinguishes our bound in \eqref{eq:MainResultPrototype} from classical statistical learning results, such as Rademacher complexity bounds, where the expected loss is minimized only over a specific hypothesis class $\mathcal{H}$. This distinction necessitates explicit consideration of the approximation error, i.e., how well functions in a hypothesis class $\mathcal{H}$ can approximate the true conditional expectation. The following result provides a lower bound for $\mathrm{MMSE}(S \vert X^{\sigma})$ that depends on three components: (i) the empirical training MSE $\mathrm{MSE}_\textup{train}(\hat{h}_\mathcal{H})$ of the trained model $\hat{h}_\mathcal{H}$ as defined in \eqref{eq:emp-mse} and \eqref{eq:trained-model}, respectively; (ii)
the approximation error incurred when estimating $\eta^\sigma$ in \eqref{eq:DefEta-general} using functions from the chosen hypothesis class $\mathcal{H}$; and (iii)
the finite sample error, which accounts for the uncertainty in learning $h^*_\mathcal{H}$, i.e., the optimal function in $\mathcal{H}$ defined in \eqref{eq:opt-h}, using only the dataset $\mathcal{D}_\text{train}$ of finite size $n$.

\begin{theorem}
Let $n = |\mathcal{D}_\textup{train}|$, and $\mathcal{H} \subseteq \{f:\mathbb R^d \to [0,1] \}$. If $S\in [0,1]$, then
\label{thm:mmse-bound-general}
\begin{equation}
\label{eq:2LNNMainResultRepresentation}
    \mathrm{MMSE}(S \vert X^{\sigma}) \ge   \mathrm{MSE}_\textup{train}(\hat{h}_\mathcal{H}) - \epsilon_C - \epsilon_A,
\end{equation}
where 
\begin{equation}
    \epsilon_C \coloneqq \mathrm{MSE}_\textup{train}({h}^*_\mathcal{H}) - \mathrm{MMSE}^\mathcal{H}(S \vert X^\sigma).
    \label{eq:eps-c}
\end{equation}
denotes the deviation of  $\mathrm{MSE}_\textup{train}({h}^*_\mathcal{H})$ from its mean, and
\begin{equation}
    \epsilon_A \coloneqq \lVert \eta^\sigma - h^*_\mathcal{H} \rVert^2_{2}
    \label{eq:eps-a}
\end{equation}
denotes the error resulting from approximating $\eta^\sigma$ by $h^*_\mathcal{H}$.
\end{theorem}

\begin{proof}[Proof sketch]\let\qed\relax 
We decompose the difference of the empirical training error $\mathrm{MSE}_\textup{train}(\hat{h}_\mathcal{H})$ and the true $\mathrm{MMSE}(S|X^\sigma)$ into a finite sample error term and an approximation error term. We do so by adding and subtracting the MMSE resulting from using a restricted hypothesis class but with an infinite number of samples as follows:
\begin{equation}
\label{eq:train-decomp}
     \underset{\text{finite sample error}}{\underbrace{\mathrm{MSE}_\textup{train}(\hat{h}_\mathcal{H}) - \mathrm{MMSE}^\mathcal{H}(S \vert X^\sigma)}} + \underset{\text{approximation error}}{\underbrace{\mathrm{MMSE}^\mathcal{H}(S \vert X^\sigma) - \mathrm{MMSE}(S \vert X^\sigma)}},
\end{equation}
where $\mathrm{MMSE}_\mathcal{H}(S|X^\sigma)$ is defined in \eqref{eq:mmse-restricted-h}.
We then bound the finite sample error by $\epsilon_C$, bounding $\mathrm{MSE}_\textup{train}(\hat{h}_\mathcal{H})$ by $\mathrm{MSE}_\textup{train}({h}^*_\mathcal{H})$. This  eliminates the need for a generalization-style bound that involve the size of the hypothesis class. Proof details are in Appendix  \ref{appendix:thm1-proof}.
\end{proof}

\begin{remark}
    A meaningful lower bound on the true MMSE requires non-zero values for $\mathrm{MSE}_\text{train}(\hat{h}_\mathcal{H})$ while requiring  small values for both $\epsilon_C$ and $\epsilon_A$. We explore this dynamic in the sequel. In particular, in the following proposition, we bound $\epsilon_C$ defined in \eqref{eq:eps-c} using large deviations results. 
\end{remark}



\begin{prop}
    Let $n = |\mathcal{D}_\textup{train}|$, $\mathcal{H} \subseteq \{f:\mathbb R^d \to [0,1] \}$, and $\delta \in (0,1)$. If $S \in [0,1]$, then with probability at least $1-\delta$, the following concentration bounds hold for $\epsilon_C$. 
\begin{enumerate}
    \item[(i)] A Hoeffding-style bound gives:
\begin{equation}
    \epsilon_C \le \sqrt{\frac{\log(1/\delta)}{2n}}.
    \label{eq:eps-c-bound-hoeffdings}
\end{equation}
\item[(ii)] Alternatively, a tighter Bernstein-style bound that depends on the sample variance of $W = (S-h^*_\mathcal{H}(X^\sigma))^2$ yields
\begin{equation}
    \epsilon_C \le \sqrt{\frac{2\textup{Var}_n(W)\log(2/\delta)}{n}}+\frac{7\log(2/\delta)}{3(n-1)},
    \label{eq:eps-c-bound-bernstein}
\end{equation}
where
\begin{equation}
    \textup{Var}_n(W) \coloneqq \frac{1}{n(n-1)}\sum_{1\le i < j \le n} (W_i - W_j)^2
\end{equation}
is the sample variance of $W$. 
\end{enumerate}
\label{prop:eps-c-bounds}
\end{prop}
\begin{proof}\let\qed\relax 
The bound in \eqref{eq:eps-c-bound-hoeffdings} follows directly from applying Hoeffding's inequality \cite[Sec.~4.2]{shalev2014understanding} since $W \in [0,1]$. The bound in \eqref{eq:eps-c-bound-bernstein} follows from an application of an empirical Bernstein's inequality \cite[Thm. 4]{maurer-bernstein-2009} again since $W \in [0,1]$.
\end{proof}

\begin{remark}
The bound on $\epsilon_C$ in \eqref{eq:eps-c-bound-bernstein} is generally tighter than that in \eqref{eq:eps-c-bound-hoeffdings}, especially for small $n$. However, it depends on knowledge of $h^*_\mathcal{H}$, which may limit its practical applicability. 
\end{remark}

It is useful to note that the approximation error term $\epsilon_A$ is a statistical quantity; thus, it requires either a large number of examples to estimate with high fidelity or a sufficiently rich hypothesis class $\mathcal{H}$ to make it small. 
In our adversarial evaluation setting, the red team evaluator gets to choose the hypothesis class $\mathcal{H}$. This choice determines how large or small one can make the approximation error.  Choosing a large hypothesis class, e.g., neural networks of a chosen width and depth, can significantly reduce $\epsilon_A$; however,  such a larger hypothesis class has the potential to overfit the data, thereby driving $\mathrm{MSE}_\text{train}(\hat{h}_\mathcal{H})$ 
to 0, and leading to vacuous lower bounds. 
To address this, we consider using the validation loss $\mathrm{MSE}_\text{val}(\hat{h}_\mathcal{H})$, defined in  \eqref{eq:emp-mse} with $y=\text{val}$, which generally remains non-zero even when the training loss is minimized. Establishing a lower bound using the validation MSE, however, requires generalization bounds. The following theorem provides such a lower bound.

\begin{theorem}
Let $n = |\mathcal{D}_\textup{train}|$, $m = |\mathcal{D}_\textup{val}|$, and $\mathcal{H} \subseteq \{f:\mathbb R^d \to [0,1] \}$. If $S\in [0,1]$, then
\label{thm:mmse-bound-general-val}
\begin{equation}
\label{eq:mmse-lower-bound-val-error}
    \mathrm{MMSE}(S \vert X^{\sigma})\geq \mathrm{MSE}_\textup{val}(\hat{h}_\mathcal{H}) - \Tilde{\epsilon}_C - \epsilon_G - \epsilon_C- \epsilon_A,
\end{equation}
where
\begin{equation}
    \Tilde{\epsilon}_C \coloneqq \mathrm{MSE}_\textup{val}(\hat{h}_\mathcal{H}) - \mathbb E[(S-\hat{h}_\mathcal{H}(X^\sigma))^2]
    \label{eq:eps-c-tilde}
\end{equation}
denotes the deviation of $\mathrm{MSE}_\textup{val}(\hat{h}_\mathcal{H})$ from its mean,
\begin{equation}
    \epsilon_G \coloneqq \mathrm{MSE}_\textup{train}(\hat{h}_\mathcal{H}) - \mathbb{E}[(S-\hat{h}_\mathcal{H}(X^\sigma))^2].
    \label{eq:eps-g}
\end{equation}
denotes the generalization error of the trained model $\hat{h}_\mathcal{H}$, $\epsilon_C$ is defined in \eqref{eq:eps-c} and $\epsilon_A$ is defined in \eqref{eq:eps-a}.
\end{theorem}

\begin{proof}[Proof sketch]\let\qed\relax 
We decompose the difference of the empirical validation error $\mathrm{MSE}_\textup{val}(\hat{h}_\mathcal{H})$ and the true $\mathrm{MMSE}(S|X^\sigma)$ into a finite sample and generalization error term and an approximation error term. We do so by adding and subtracting the MMSE that results from using a restricted hypothesis class but with an infinite number of samples as follows:
\begin{equation}
\label{eq:val-decomp}
     \underset{\text{finite sample + generalization error}}{\underbrace{\mathrm{MSE}_\textup{val}(\hat{h}_\mathcal{H}) - \mathrm{MMSE}^\mathcal{H}(S \vert X^\sigma)}} + \underset{\text{approximation error}}{\underbrace{\mathrm{MMSE}^\mathcal{H}(S \vert X^\sigma) - \mathrm{MMSE}(S \vert X^\sigma)}},
\end{equation}
where $\mathrm{MMSE}_\mathcal{H}(S|X^\sigma)$ is defined in \eqref{eq:mmse-restricted-h}. We then further decompose the finite sample + generalization term into the following terms by adding and subtracting the mean of $\mathrm{MSE}_\textup{val}(\hat{h}_\mathcal{H})$: (i) the concentration term $\epsilon_C$ defined in \eqref{eq:eps-c}, and (ii) a second term that we bound by the sum of another concentration term $\Tilde{\epsilon}_C$ defined in \eqref{eq:eps-c-tilde} and the generalization error term $\epsilon_G$ defined in \eqref{eq:eps-g}. The proof details are in Appendix \ref{appendix:mmse-bound-general-val-proof}.
\end{proof}

\begin{remark}
Both the training loss bound in  \cref{thm:mmse-bound-general} and the validation loss bound in  \cref{thm:mmse-bound-general-val} share the same approximation error term, $\epsilon_A$, defined in \eqref{eq:eps-a}. However, unlike the bound in  \cref{thm:mmse-bound-general}, which depends solely on the training loss, the bound in  \cref{thm:mmse-bound-general-val} is based on the validation loss of the \emph{trained} model. As a result, it depends on: (i) the size of the training set $\mathcal{D}_\text{train}$ through the generalization error term $\epsilon_G$ defined in \eqref{eq:eps-g} and the large deviation term ${\epsilon_C}$ defined in \eqref{eq:eps-c},  and (ii) the size of the validation set $\mathcal{D}_\text{val}$  through the large deviation term $\Tilde{\epsilon}_C$. 
Leveraging the validation loss rather than the training loss enables the use of richer hypothesis classes. However, this flexibility comes at the cost of the additional generalization term ${\epsilon}_G$, which captures the complexity of the model $\hat{h}_\mathcal{H}$. Consequently, there exists a fundamental tradeoff between the training loss bound in  \cref{thm:mmse-bound-general} and the validation loss bound in  \cref{thm:mmse-bound-general-val}. While the former offers a simpler bound that depends only on the training data, the latter accommodates more expressive hypothesis classes at the expense of an additional complexity-dependent generalization term.
\end{remark}

There are various generalization bounds that can be used to control the generalization error term 
$\epsilon_G$ in \eqref{eq:eps-g}. Traditional generalization bounds, based on measures such as Rademacher complexity \cite{bartlett2002rademacher} or VC dimension \cite{vapnik1998adaptive}, often yield loose or vacuous estimates, particularly for large models like neural networks. More recently, tighter bounds have been derived using Kolmogorov complexity \cite{wilson2025deep}, which quantifies a model’s complexity by the length (in bits) of the shortest program that represents it under a given coding scheme. We present one such bound in the following proposition.

\begin{prop}
Let $n = |\mathcal{D}_\textup{train}|$, $\mathcal{H} \subseteq \{f:\mathbb R^d \to [0,1] \}$, and $\delta \in (0,1)$. If $S\in [0,1]$, then with probability at least $1-\delta$
\begin{equation}
    \epsilon_G \le \sqrt{\frac{C( \hat{h}_\mathcal{H})\log(2) +2\log(C( \hat{h}_\mathcal{H}))+\log(1/\delta)}{{2n}}},
    \label{eq:eps-g-bound}
\end{equation}
where $C(h)$ denotes the number of bits required to represent hypothesis $h$ using some pre-specified coding.
\label{prop:eps-g-bound}
\end{prop}
\begin{proof}\let\qed\relax 
    The result follows by applying the generalization bound from \cite[Thm. 3.1]{wilson2025deep}, together with the Kolmogorov complexity bound given in \cite[eq. (4)]{wilson2025deep}.
\end{proof}

Leveraging \cref{prop:eps-c-bounds,prop:eps-g-bound}, we can now clarify the $\epsilon_{(\cdot)}$ terms in \eqref{eq:mmse-lower-bound-val-error} of  \cref{thm:mmse-bound-general-val}. The concentration term $\epsilon_C$ in \eqref{eq:mmse-lower-bound-val-error} can be bounded using  \cref{prop:eps-c-bounds}. The same proposition can also be used to bound $\Tilde{\epsilon}_C$ by substituting $n$ with $m$ and replacing 
$h^*_\mathcal{H}$ with $\hat{h}_\mathcal{H}$. Importantly, the Bernstein-style bound in \eqref{eq:eps-c-bound-bernstein} therefore applies to $\Tilde{\epsilon}_C$ without requiring knowledge of $h^*_\mathcal{H}$, since it depends only on the trained model 
$\hat{h}_\mathcal{H}$. Therefore, the Bernstein bound is preferable and should be used whenever bounding $\Tilde{\epsilon}_C$. Finally, when combining bounds on $\epsilon_C$, $\Tilde{\epsilon}_C$ and $\epsilon_G$, applying the union bound where each individual bound is assured with the confidence parameter $\delta/3$ ensures that the total failure probability remains bounded by $\delta$.

\begin{remark}
If $\mathcal{D}_\text{train}$ is sufficiently large and the relationship between the $S$ and $X$ is complex enough that the trained model  $\hat{h}_\mathcal{H}$ does not overfit—i.e., $\mathrm{MSE}_\text{train}(\hat{h}_\mathcal{H})$ remains nonzero even for a large hypothesis class $\mathcal{H}$—then the training loss bound in  \cref{thm:mmse-bound-general} is likely tighter than the validation loss bound in  \cref{thm:mmse-bound-general-val}, as the latter incurs an additional generalization penalty. Conversely, if $\hat{h}_\mathcal{H}$ overfits the training data and reducing the hypothesis class $\mathcal{H}$ would make 
$\epsilon_A \coloneqq \|\eta^\sigma - h^*_\mathcal{H}\|$ excessively large, the validation loss bound provides an alternative. However, since most generalization bounds tend to be vacuous unless the dataset size ($n,m$) is large enough to offset the complexity term $C(\hat{h}_\mathcal{H})$ in \eqref{eq:eps-g-bound}, the validation bound may not yet offer a tighter alternative to the training bound. 
\end{remark}

\section{Lower Bounds for Linear Hypothesis Class}
\label{sec:linear-eps-a-bounds}

We now focus our attention on the approximation error term $\epsilon_A$ defined in \eqref{eq:eps-a}. In order to bound this term, we consider the simple setting where $\mathcal{H}$ is the hypothesis class of the sigmoid function composed with linear functions. In the following result, we show that we can obtain a closed-form bound on $\epsilon_A$ in terms of the statistics of $\theta^\sigma(X^\sigma)$ and $X^\sigma$.
\begin{prop}
    Let $\mathcal{H}_{L}$ denote the hypothesis class associated with linear functions $\theta_L:\mathbb{R}^{d}\to\mathbb{R}$ of the form
\begin{equation}
\label{eq:DefHk}
    \theta_L(x) = a^\mathrm{T}x + b,
\end{equation}
where $a\in\mathbb{R}^{d}$ and $b\in\mathbb{R}$. Then there exists $\theta_L^* = \argmin_{\theta_L \in \mathcal{H}_L} \| \theta^\sigma - \theta_L \|^2_2$ with parameters
\begin{align}
a^* &= \textup{Var}(X^\sigma)^{-1}\textup{Cov}(X^\sigma,\theta^\sigma(X^\sigma)), \label{eq:theta-l-weight} 
 \\ 
b^* &= \mathbb E[\theta^\sigma(X^\sigma)] - (a^*)^T\mathbb E[X^\sigma]. \label{eq:theta-l-bias}
\end{align}
such that
\begin{align}
    \epsilon_A
    &\le \frac{1}{4} \left[\textup{Var}(\theta^\sigma(X^\sigma)) - \textup{Cov}(\theta^\sigma(X^\sigma),X^\sigma)\textup{Var}(X^\sigma)^{-1}\textup{Cov}(X^\sigma,\theta^\sigma(X^\sigma)) \right].
    \label{eq:delta-a-decomp}
\end{align}
\label{prop:linear-bound-epsilon-a}
\end{prop}
\begin{proof}\let\qed\relax 
Using the fact that $\eta^\sigma(x) = s(\theta^\sigma(x))$,
 we can write $h_\mathcal{H}^*(x) = s(\theta_L^*(x))$ for 
 \begin{equation}
     \theta_L^* = \argmin_{\theta_L \in \mathcal{H}_L} \mathbb E\left[ (\theta^\sigma(X^\sigma) - \theta_L(X^\sigma))^2\right].
     \label{eq:opt_theta-L}
 \end{equation}
 We can then bound $\epsilon_A$ as follows:
\begin{align}
    \epsilon_A = \|s \circ \theta^\sigma - s \circ \theta_L^* \|^2_2 \le \frac{1}{4}\| \theta^\sigma - \theta_L^* \|^2_2,
    \label{eq:approx-error-bound}
\end{align}
where the inequality follows from the fact that $s$ is ${1}/{4}$-Lipschitz. Solving \eqref{eq:opt_theta-L},  $\theta_L^*$ has the form
\begin{align}
    \theta_L^*(x) = (a^*)^\mathrm{T} x + b^*,
    \label{eq:opt-theta-l-form}
\end{align}
where $a^*$ and $b^*$ are given in \eqref{eq:theta-l-weight} and \eqref{eq:theta-l-bias}, respectively.
Substituting \eqref{eq:opt-theta-l-form} into \eqref{eq:approx-error-bound} yields
\begin{align}
    \| \theta^\sigma - \theta_L^* \|^2_2 &= \mathbb{E}\left[(\theta^\sigma(X^\sigma) - \theta_L^*(X^\sigma))^2\right] \nonumber \\
    &= \text{Var}(\theta^\sigma(X^\sigma))  - \text{Cov}(\theta^\sigma(X^\sigma),X^\sigma)\text{Var}(X^\sigma)^{-1}\text{Cov}(X^\sigma,\theta^\sigma(X^\sigma)).
\end{align} 
\end{proof}

We can therefore bound $\epsilon_A$ by approximating $\theta^\sigma$ using the linear function $\theta^*_L$. We now present several examples for various increasingly complex relationships between $S$ and $X$ to gain some intuition about the behavior of $\epsilon_A$ as a function of $\sigma$. 
\begin{prop}
If $S \sim \mathrm{Ber}(p)$ and 
\\(i) $X = aS+b$ for some $a,b \in \mathbb{R}^d$ or \\
(ii) $(X\vert S=s) \sim N(\mu_s,\Sigma)$ for mean $\mu_s \in \mathbb{R}^d$ and positive definite covariance matrix $\Sigma \in \mathbb{R}^{d\times d}$, \\
then $\epsilon_A=0$. 
\label{prop:delta-a-linear-optimal-example}
\end{prop}
\begin{proof}[Proof sketch]\let\qed\relax 
For setting (ii), computing $\theta^\sigma$ in \eqref{eq:DefTheta} yields
\[\theta^\sigma(x) = (\mu_1-\mu_0)^\mathrm{T}\Tilde{\Sigma}^{-1}x + \frac{1}{2}\left(\mu_0^\mathrm{T}\Tilde{\Sigma}^{-1}\mu_0-\mu_1^\mathrm{T}\Tilde{\Sigma}^{-1}\mu_1\right) + \log\left(\frac{p}{\bar{p}} \right),\]
where $\Tilde{\Sigma} \coloneqq \Sigma + \sigma^2 I$. Since $\theta^\sigma$ linear in $x$, it can therefore be recovered exactly by $\theta^*_L$. Moreover, since setting (i) is a special case of setting (ii), the same conclusion holds for setting (i) as well. See 
\if \extended 0%
    \cite[Appendix~C]{ISITextenver}
    \fi%
    \if \extended 1%
    Appendix \ref{appendix:proof-delta-a-linear-optimal}
    \fi for a complete proof.
\end{proof}
\begin{remark}
    Once Gaussian noise is added to $X$, the setting (i) in  \cref{prop:delta-a-linear-optimal-example} reduces to a special case of setting (ii), for which it is known that a linear separator is optimal.
    Consequently, in the training bound of \cref{thm:mmse-bound-general}, the only gap between $\mathrm{MSE}_\textup{train}(\hat{h}_\mathcal{H})$ and  $\mathrm{MMSE}(S \vert X^{\sigma})$  arises from the concentration term $\epsilon_C$, which can be made arbitrarily small with a sufficiently large training set. In the validation bound of  \cref{thm:mmse-bound-general-val}, the gap between $\mathrm{MSE}_\textup{val}(\hat{h}_\mathcal{H})$ and  $\mathrm{MMSE}(S \vert X^{\sigma})$ is due to the remaining concentration terms 
$\epsilon_C$ and $\Tilde{\epsilon}_C$, as well as the generalization error term 
$\epsilon_G$. All of these terms can be made arbitrarily small with sufficiently large training and validation sets, provided the number of training samples is sufficient to offset the complexity of the learned model $\hat{h}_\mathcal{H}$.
\end{remark}

It is of more interest to consider settings in which the optimal predictor is nonlinear, in order to better understand how the approximation error $\epsilon_A$ depends on the noise parameter $\sigma$, and how this, in turn, influences the lower bound on the MMSE. The simplest example we consider is a binary symmetric channel, which is very relevant in communication systems.

\begin{theorem}[Binary Symmetric Channel]
If $S \sim \text{Ber}(p)$, $N \sim \text{Ber}(p_N)$,  and $X = S \oplus N$, then
\begin{align}
   \epsilon_A \le \epsilon^\text{BSC}_A \coloneqq &\frac{3(1-2p_N)^2}{4\sigma^2} - \frac{(1-2p_N)^2}{4(q\overline{q}+\sigma^2)}  + \frac{(1-2p_N)^2\left(5-4p_N+4p_N^2+14q-2q^2 \right)}{8\sigma^4} \nonumber \\
    & \quad + \frac{2(1-2p_N)^2\left(p_N\overline{p_N}-q\overline{q} \right)}{4\sigma^2\left(q\overline{q}+\sigma^2 \right)} + \mathcal{O}\left(\frac{C_1(p,p_N)(1-2p_N)^2}{\sigma^6} \right) +\mathcal{O}\left(\frac{C_2(p,p_N)(1-2p_N)^2}{\sigma^4\left(q\overline{q}+\sigma^2 \right)} \right)
    \label{eq:epsilon-a-bsc}
\end{align}
for some constants $C_1(p,p_N)$ and $C_2(p,p_N)$, where $q \coloneqq p_N\overline{p}+\overline{p_N}p$.
\label{thm:delta-a-bsc-example}
\end{theorem}
\begin{proof}[Proof sketch]\let\qed\relax 
Computing $\theta^\sigma$ in \eqref{eq:DefTheta} yields  
\begin{equation}
    \theta^\sigma(x)=\log\left(\frac{p}{\overline{p}}\right)+\log\left(\frac{p_N + \overline{p_N} e^{(2x-1)/2\sigma^2}}{\overline{p_N} + {p_N} e^{(2x-1)/2\sigma^2}} \right).
    \label{eq:bsc-theta-sig}
\end{equation}
We use a series expansion for the second term in \eqref{eq:bsc-theta-sig} in addition to the binomial theorem and the series expansion for the exponential function in order to derive approximations of the expectations in \eqref{eq:delta-a-decomp}. Proof details are in
\if \extended 0%
    \cite[Appendix~D]{ISITextenver}.
    \fi%
    \if \extended 1%
    Appendix \ref{appendix:prrof-bsc-example}.
    \fi 
\end{proof}

\begin{remark}
For large noise levels $\sigma$, the approximation error bound $\epsilon^\text{BSC}_A$ decays as $\mathcal{O}(1/\sigma^2)$. 
Given large enough training and validation sets, the convergence of $\mathrm{MMSE}(S \vert X^{\sigma})$ to either $\mathrm{MSE}_\textup{train}(\hat{h}_\mathcal{H})$ in \eqref{eq:2LNNMainResultRepresentation} or $\mathrm{MSE}_\textup{val}(\hat{h}_\mathcal{H})$ in \eqref{eq:mmse-lower-bound-val-error} is dominated by the noise term $\mathcal{O}(1/\sigma^2)$. This behavior has an intuitive explanation: as the noise level $\sigma$ increases, the class-conditional distributions increasingly resemble pure Gaussians rather than Gaussian mixtures, causing the optimal estimator $\theta^\sigma$ in \eqref{eq:bsc-theta-sig} to approach linearity and eventually become constant in the limit, as shown in  \cref{fig:bsc-noise-density-progression}. Consequently, a linear function provides an increasingly good approximation. Additionally, as the class-conditional distributions overlap more substantially, the true $\mathrm{MMSE}(S \vert X^{\sigma})$, the empirical training error $\mathrm{MSE}_\textup{train}(\hat{h}_\mathcal{H})$ and the empirical validation error $\mathrm{MSE}_\textup{val}(\hat{h}_\mathcal{H})$ grow with $\sigma$, reflecting diminishing ability to recover the sensitive attribute $S$ from the noised features $X^\sigma$. While larger $\sigma$ provides stronger adversarial evaluation guarantees, the choice of noise level should ultimately be dictated by the utility of the noised features $X^\sigma$. In this setting, our analysis shows that a linear model serves as an effective adversarial evaluation tool for moderate to large $\sigma$, providing reliable estimates of adversarial inference of sensitive features with minimal computational overhead. 
\end{remark}

\begin{figure}
    \centering
    \begin{subfigure}{0.33\textwidth}
        \centering
        \includegraphics[width=\textwidth,page=2]{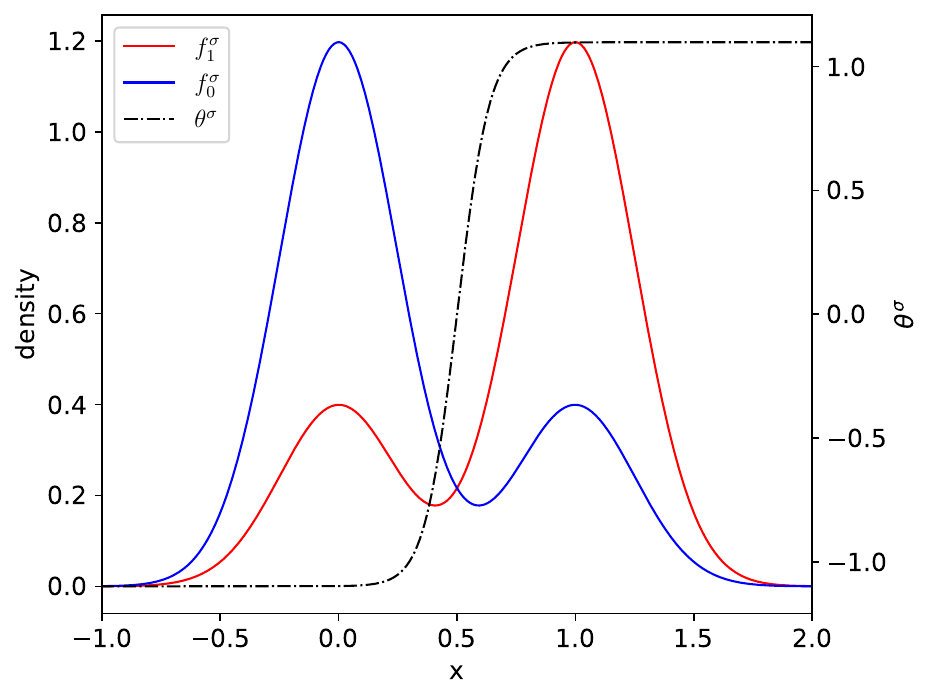}
        \caption{$\sigma = 0.25$}
    \end{subfigure}\hfill\hfill%
    \begin{subfigure}{0.33\textwidth}
        \centering
        \includegraphics[width=\textwidth,page=3]{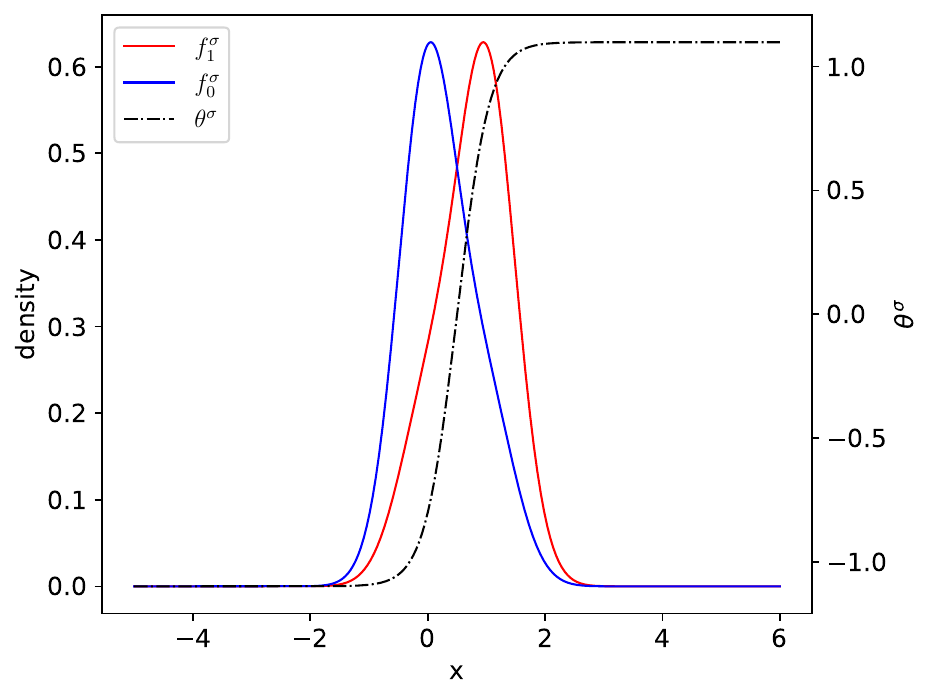}
        \caption{$\sigma = 0.5$}
    \end{subfigure}\hfill\hfill%
    \begin{subfigure}{0.33\textwidth}
        \centering
    \includegraphics[width=\textwidth,page=7]{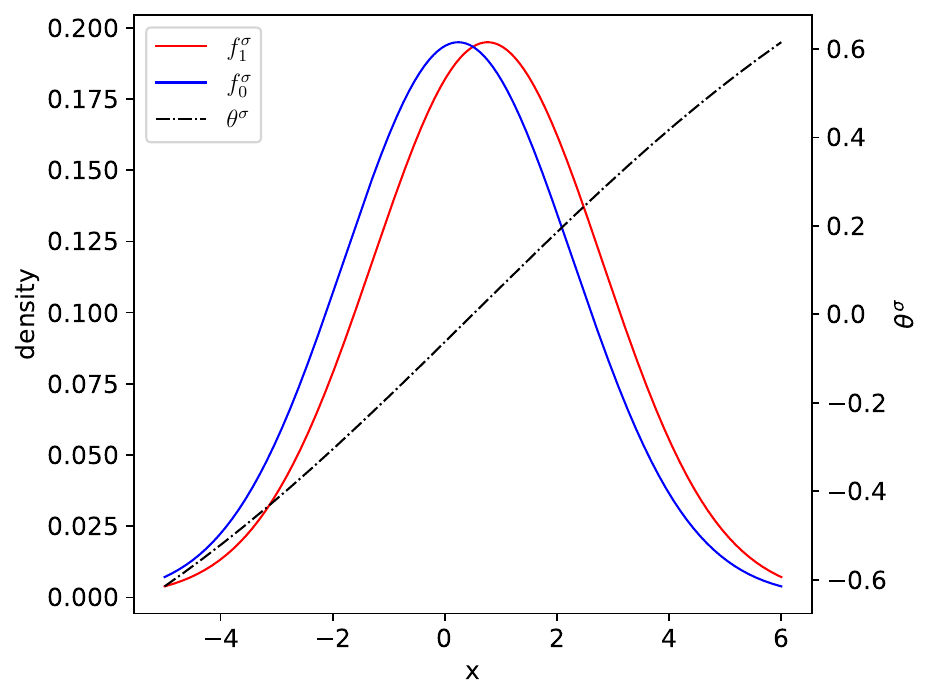}
        \caption{$\sigma = 2$}
    \end{subfigure}
    \caption{An illustration of how the class-conditional densities $f_i^\sigma$
  of $X^\sigma\vert S=i$, for $i\in\{0,1\}$, and the optimal estimator $\theta^\sigma$ evolve in the binary symmetric channel setting of Theorem~\ref{thm:delta-a-bsc-example}, as the noise level $\sigma$ increases. Here, $p=1/2$  and $p_N=1/4$. As $\sigma$ grows, the conditional densities progressively resemble pure Gaussians rather than Gaussian mixtures. This leads $\theta^\sigma$  to become increasingly linear and ultimately converge to a constant function.}
    \label{fig:bsc-noise-density-progression}
\end{figure}

Our next example examines the case where $X$ follows a class-conditional Gaussian distribution with distinct covariance matrices. 
For this setting, we can derive exact closed-form expressions for each term in the bound given in \eqref{eq:delta-a-decomp}, rather than relying on approximations.

\begin{theorem}[Class Conditional Vector Gaussian] Suppose $S \sim \text{Ber}(p)$ and $(X|S=s) \sim \mathcal{N}(\mu_s,\Sigma_s)$ for $s \in \{0,1\}$, where $\mu_s \in \mathbb R^d$ and $\Sigma_s \in \mathbb R^{d\times d}$ is a positive definite covariance matrix. Let $\Tilde{\Sigma}_s \coloneqq \Sigma_s + \sigma^2 I$ for $\sigma > 0$, $\mu_d \coloneqq \mu_1-\mu_0$,
\begin{align}
    A &= \frac{1}{2}\left(\Tilde{\Sigma}_0^{-1} - \Tilde{\Sigma}_1^{-1}\right), \label{eq:ccg-theta-sig-a-def} \\
    b &= \Tilde{\Sigma}_1^{-1}\mu_1 - \Tilde{\Sigma}_0^{-1}\mu_0, \label{eq:ccg-theta-sig-b-def} \\
    c & =\frac{1}{2}\mu_0^\mathrm{T}\Tilde{\Sigma}_0^{-1}\mu_0 - \frac{1}{2}\mu_1^\mathrm{T}\Tilde{\Sigma}_1^{-1}\mu_1+ \frac{1}{2}\log\left(\frac{|\Tilde{\Sigma}_0|}{|\Tilde{\Sigma}_1|} \right)+ \log\left( \frac{p}{\overline{p}}\right), \label{eq:ccg-theta-sig-c-def} \\ 
    M_s &\coloneqq \sum_{j=1}^d \lambda_j^{(s)}+b^\mathrm{T}\mu_s + \mu_s^\mathrm{T}A\mu_s + c, \label{eq:cond-exp-ccg} \\
    V_s &\coloneqq \sum_{j=1}^d 2(\lambda_j^{(s)})^2+(u_j^{(s)})^2,
    \label{eq:cond-var-ccg}
\end{align}
where $\lambda_j^{(s)}$, $j \in \{1,\dots,d\}$, are the eigenvalues of $\Tilde{\Sigma}_y^{1/2}A\Tilde{\Sigma}_y^{1/2}$ with corresponding eigenvectors as the columns of a matrix $Q_y$, i.e., $Q_s^\mathrm{T}\Tilde{\Sigma}_s^{1/2}A\Tilde{\Sigma}_s^{1/2}Q_s = \text{diag}(\lambda_1^{(s)},\dots,\lambda_d^{(s)})$, and
\begin{equation}
    u^{(s)} = (u_1^{(s)},\dots,u_d^{(s)})^\mathrm{T} = Q_s^\mathrm{T}\Tilde{\Sigma}_s^{1/2}(b+2A\mu_s).
\end{equation}
Then, 
\begin{equation}
    \epsilon_A \le \epsilon^\text{CCG}_A \coloneqq \frac{1}{4}\left(\textup{VAR}_1 - \textup{COV}*\textup{VAR}_2^{-1}*\textup{COV}^\mathrm{T}\right),
\end{equation}
where
\begin{align}
    &\textup{VAR}_1 = pV_1 + \overline{p}V_0 + p\overline{p}\left(M_1-M_0\right)^2, \label{eq:var-theta-sig-ccg}\\
    &\textup{COV}
    = 2p\mu_1^\mathrm{T}A\Sigma_1 + 2\overline{p}\mu_0^\mathrm{T}A\Sigma_0+ 2\sigma^2 \left[p\mu_1+\overline{p}\mu_0\right]^\mathrm{T} A  + p\overline{p}\left[\text{tr}(A(\Sigma_1 - \Sigma_0)) + \mu_1^\mathrm{T}A\mu_1 - \mu_0^\mathrm{T}A\mu_0 \right]\mu_d^\mathrm{T} \nonumber\\
    & \qquad \qquad + b^\mathrm{T}\left[p\Sigma_1 + \overline{p}\Sigma_0 +\sigma^2I+ p\overline{p}\mu_d\mu_d^\mathrm{T} \right]  \label{eq:cov-theta-x-ccg} \\
    &\textup{VAR}_2 = p\Sigma_1 + \overline{p}\Sigma_0 +p\overline{p}\mu_d\mu_d^\mathrm{T} + \sigma^2 I \label{eq:var-x-sig-ccg}.
\end{align}
\label{thm:CCG-vector-general}
\end{theorem}
\begin{proof}[Proof sketch]\let\qed\relax
It is well-known that the optimal separator is quadratic \cite{hastie_elements-of-statistical-learning}, i.e.,
\begin{equation}
    \theta^\sigma(x) = x^\mathrm{T}Ax + b^\mathrm{T}x + c,
    \label{eq:optimal_ccg}
\end{equation}
where $A$, $b$, and $c$ are defined in \eqref{eq:ccg-theta-sig-a-def} and \eqref{eq:ccg-theta-sig-b-def}, \eqref{eq:ccg-theta-sig-c-def} respectively.
We then compute each term in \eqref{eq:delta-a-decomp} as in \eqref{eq:var-theta-sig-ccg}, \eqref{eq:cov-theta-x-ccg} and \eqref{eq:var-x-sig-ccg} by conditioning on $S$ and using moments of quadratic forms of Gaussian random vectors. A complete proof can be found in
\if \extended 0%
    \cite[Appendix~E]{ISITextenver}.
    \fi%
    \if \extended 1%
    Appendix \ref{appendix:proof-CCG-vector-general}.
    \fi 
\end{proof}

To gain more intuition about how the class-conditional distribution parameters $\mu_s$, $\Sigma_s$, $s \in \{0,1\}$ and noise parameter $\sigma$ affect $\epsilon_A$, we consider the special case when $\Sigma_s = \sigma_s^2 I$.

\begin{corollary}
    Suppose $S \sim \text{Ber}(p)$ and $(X|S=s) \sim \mathcal{N}(\mu_s,\sigma_s^2 I)$ for $s \in \{0,1\}$, where $\mu_s \in \mathbb R^d$ and $\sigma_s > 0$. Then for $\sigma > 0$,
    
\begin{equation}
    \epsilon_A \le \epsilon_A^\text{CCG} = \frac{(\sigma_1^2-\sigma_0^2)^2(q_1 + q_2\sigma^2+q_3\sigma^4+2d\sigma^6)}{16(r_1+r_2\sigma^2+r_3\sigma^4+r_4\sigma^6+r_5\sigma^8+\sigma^{10})},
    \label{eq:CCG-vector-diag-cov-bound}
\end{equation}
where 
\begin{subequations}
\begin{align}
    q_1 & = \lVert \mu_1 - \mu_0 \rVert_2^4\left(p^2\overline{p}\sigma_1^2 +p\overline{p}^2\sigma_0^2 \right) + 2d\sigma_0^6 - d^2p^3(\sigma_1^2-\sigma_0^2)^3 +2p\overline{p}(2+d)\sigma_0^2\sigma_1^2\lVert \mu_1 - \mu_0 \rVert_2^2 \nonumber \\
    & \quad + p^2\left(d(5d-2)\sigma_0^4\sigma_1^2   - 2d(2d+1)\sigma_0^2\sigma_1^4+d(2+d)\sigma_1^6 - 2d(d-1)\sigma_0^6\right) \nonumber \\
    & \quad +p\left(d(d-4)\sigma_0^6 - 2d(d-1)\sigma_0^4\sigma_1^2 + d(2+d) \sigma_0^2\sigma_1^4 \right), \\
    q_2 & = p\overline{p}\lVert \mu_1 - \mu_0 \rVert_2^2\left(\lVert \mu_1 - \mu_0 \rVert_2^2 + 2(2+d)(\sigma_0^2+\sigma_1^2) \right) - p^2d(d-4)(\sigma_1^2-\sigma_0^2)^2 \nonumber \\
    & \quad +p\left(d(d-10)\sigma_0^4 -2d(d-4)\sigma_0^2\sigma_1^2+d(2+d)\sigma_1^4 \right)+ 6d\sigma_0^4, \\
    q_3 & = 2p\overline{p}(2+d)\lVert \mu_1 - \mu_0 \rVert_2^2+6d(p\sigma_1^2+\overline{p}\sigma_0^2), \\
    r_1 & = \sigma_0^4\sigma_1^4\left(p\overline{p}\lVert \mu_1 - \mu_0 \rVert_2^2+p\sigma_1^2+\overline{p}\sigma_0^2 \right), \\
    r_2 & = \sigma_0^2\sigma_1^2 \left(2p\overline{p}(\sigma_1^2+\sigma_0^2)\lVert \mu_1 - \mu_0 \rVert_2^2 +2p\sigma_1^4 +2\overline{p}\sigma_0^4+ 3\sigma_0^2\sigma_1^2 \right), \\
    r_3 & = p\overline{p}(\sigma_0^4 + 4\sigma_0^2\sigma_1^2 + \sigma_1^4)\lVert \mu_1 - \mu_0 \rVert_2^2 + p\sigma_1^6 + 3(1+p)\sigma_0^2\sigma_1^4 + 3(2-p)\sigma_0^4\sigma_1^2+ \overline{p}\sigma_0^6, \\
    r_4 & = 2p\overline{p}(\sigma_1^2+\sigma_0^2)\lVert \mu_1 - \mu_0 \rVert_2^2 + (2p+1)\sigma_1^4 + (3-2p)\sigma_0^4 +6\sigma_0^2\sigma_1^2, \\
    r_5 & = p\overline{p}\lVert \mu_1 - \mu_0 \rVert_2^2 + (3-p)\sigma_0^2 + (2+p)\sigma_1^2.
\end{align}
\end{subequations}
Additionally,
\begin{equation}
    \epsilon_A^\text{CCG} = \frac{c_1d^2 + c_2 d + c_3}{c_4},
    \label{eq:eps-ccg-d}
\end{equation}
where
\begin{subequations}
\begin{align}
    c_1 &= p\bar{p}(\sigma_1^2-\sigma_0^2)^4(\sigma^2 + p\sigma_1^2 + \bar{p}\sigma_0^2), \\
    c_2 & = 2(\sigma_1^2-\sigma_0^2)^2\left((\sigma_0^2+\sigma^2)^3 +p(1-p)\|\mu_1-\mu_0\|^2(\sigma_0^2+\sigma^2)(\sigma_1^2+\sigma^2) + p^2(\sigma_1^2-\sigma_0^2)^2(\sigma_1^2+\sigma_0^2+2\sigma^2)\right. \nonumber \\
    & \qquad \left. + p(\sigma_0^2+\sigma^2)(\sigma_1^2-\sigma_0^2)(\sigma_1^2+2\sigma_0^2+3\sigma^2) \right) ,\\
    c_3 &= p\bar{p}\|\mu_1-\mu_0\|^2(\sigma_1^2-\sigma_0^2)^2\left(4(\sigma_0^2+\sigma^2)(\sigma_1^2+\sigma^2)+\|\mu_1-\mu_0\|^2(p\sigma_1^2+\bar{p}\sigma_0^2+\sigma^2) \right), \\
    c_4 & =  16(r_1+r_2\sigma^2+r_3\sigma^4+r_4\sigma^6+r_5\sigma^8+\sigma^{10}).
\end{align}
\label{eq:coeffs-eps-ccg-d}
\end{subequations}
    \label{cor:CCG-vector-diag-cov}
\end{corollary}
\begin{proof}[Proof sketch]\let\qed\relax
Substituting $\Sigma_s = \sigma^2_s I$, and consequently $\Tilde{\Sigma}_s = (\sigma^2_s+\sigma^2) I$, into \cref{thm:CCG-vector-general}, we then simplify and manipulate the resulting expressions to obtain the result in \eqref{eq:CCG-vector-diag-cov-bound}. 
Proof details can also be found in 
\if \extended 0%
    \cite[Appendix~F]{ISITextenver}.
    \fi%
    \if \extended 1%
    Appendix \ref{appendix:proof-CCG-vector-diag-cov}.
    \fi
\end{proof}
\begin{remark}
    For large values of $\sigma$, $\epsilon_A^{\text{CCG}}$ in \eqref{eq:CCG-vector-diag-cov-bound} decays as $\mathcal{O}(1/\sigma^{4})$. Similar to \cref{fig:bsc-noise-density-progression}, in \cref{fig:ccg-noise-density-progression} in Appendix \ref{sec:additional_plots}, we illustrate how the optimal $\theta^\sigma$ reduces in curvature from a quadratic to a linear function with increasing $\sigma$.  Note that, without loss of generality, assuming $\sigma_1 > \sigma_0$ results in $c_i > 0$ for all $i \in \{1,2,3,4\}$, as defined in \eqref{eq:coeffs-eps-ccg-d}. Consequently, $\epsilon_A^\textup{CCG}$ grows quadratically with the dimension $d$ of $X^\sigma$. Intuitively, this occurs because the difference between a quadratic and a linear surface becomes more pronounced in higher dimensions, making a linear hypothesis an increasingly poor approximation of the true $\theta^\sigma$. This observation motivates the consideration of richer hypothesis classes, such as neural networks.  
\end{remark}

\subsection{Limitations of Derived Bounds on $\epsilon_A$ for Linear Hypothesis Class}
While the derived bounds $\epsilon_A^\text{BSC}$ and $\epsilon_A^\text{CCG}$ on $\epsilon_A$ in \cref{thm:delta-a-bsc-example,thm:CCG-vector-general}, respectively, provide a way to express the approximation error in terms of the parameters of the corresponding distributions and reveal the behavior of $\epsilon_A$ for a sigmoid-composed-with-linear hypothesis class, they serve as loose approximations—particularly when $\sigma$ is small, as illustrated in \cref{fig:bound-tightness}. This results in trivial lower bounds on the true MMSE, rendering them uninformative. 

This limitation arises from the inability to effectively leverage the contraction property of the sigmoid function, highlighting the need for a more refined approach to bounding $\epsilon_A$. Consequently, in the next section, rather than relying on these loose upper bounds, we directly use the exact formulation:
$\epsilon_A = \|\eta^\sigma - h^*_\mathcal{H}\|^2_2$ to illustrate results.

\begin{figure}
    \centering
    \begin{subfigure}{0.48\textwidth}
        \centering
        \includegraphics[width=\textwidth]{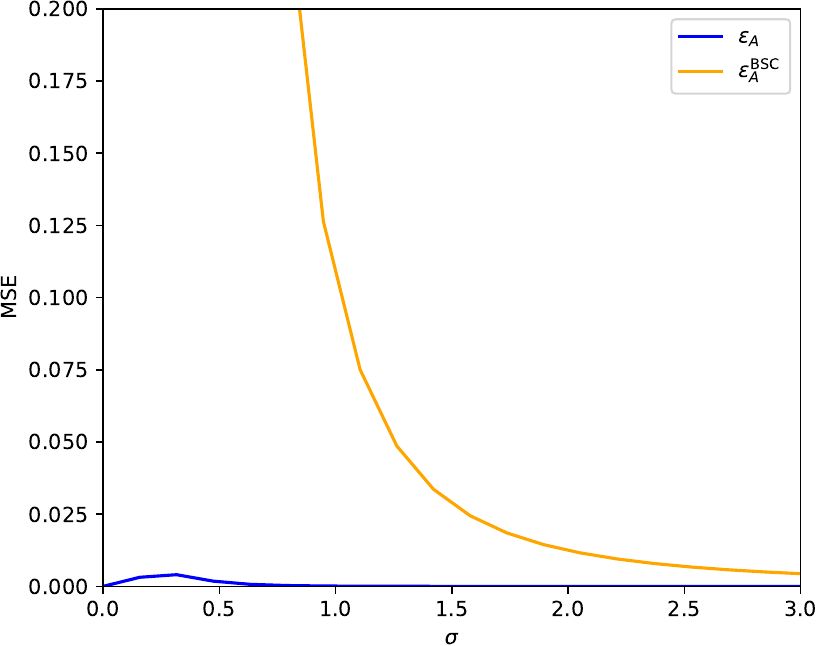}
        \caption{Binary Symmetric Channel}
    \end{subfigure}\hfill\hfill%
    \begin{subfigure}{0.48\textwidth}
        \centering
        \includegraphics[width=\textwidth]{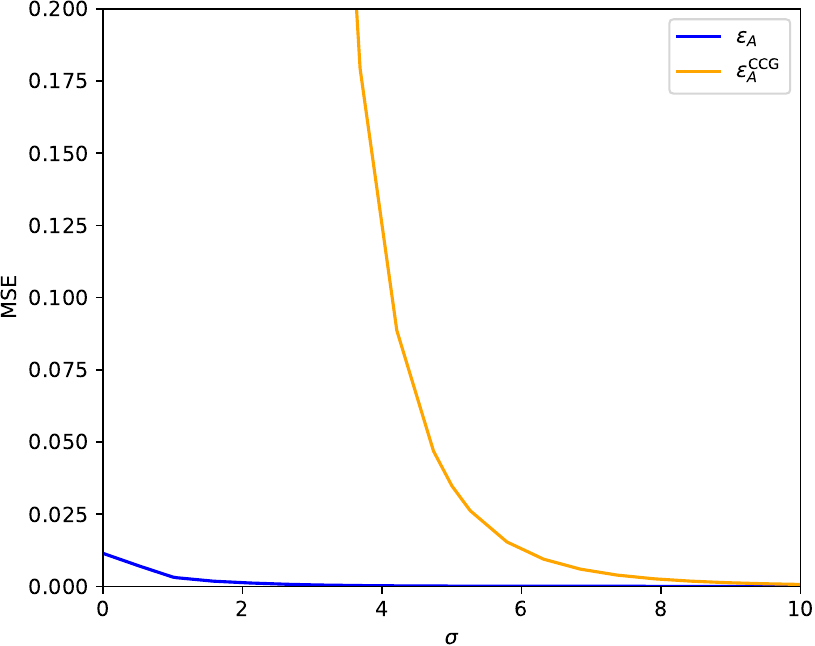}
        \caption{Class-conditional Gaussian}
    \end{subfigure}
    \caption{
    Comparison of the bounds $\epsilon_A^\text{BSC}$ (using the first four terms of the series in \eqref{eq:epsilon-a-bsc}) and $\epsilon_A^\text{CCG}$ (from \eqref{eq:eps-ccg-d}) against the true value of $\epsilon_A$ in the binary symmetric channel (BSC) and class-conditional Gaussian (CCG) settings, corresponding to \cref{thm:delta-a-bsc-example} and \cref{cor:CCG-vector-diag-cov}, respectively. Both bounds become tighter for large $\sigma$ but are loose for small $\sigma$. The true $\epsilon_A$ is computed using the closed-form expression for $\eta^\sigma$, with two independent sets of 1M samples of $(X^\sigma, S)$: one to estimate the statistics used in the closed-form expression of the optimal linear estimator $\theta^*_L$ defined in \eqref{eq:opt-theta-l-form} in order to obtain $h^*_\mathcal{H}$, and another to estimate the expectation. 
    In the BSC case, $\epsilon_A = 0$ when $\sigma = 0$, since $X$ is Bernoulli and $\eta^\sigma(x)$ therefore only needs to be learned at $x=0$ and $x=1$. A sigmoid-composed-with-linear model can exactly match these values, yielding zero approximation error. The parameters for the BSC setting are $p = 1/4$ and $p_N = 1/4$; for the CCG setting, they are $p = 1/4$, $\sigma_0^2 = 1$, $\sigma_1^2 = 3$, $\mu_0 = -1$, $\mu_1 = 1$, and $d = 1$.}
    \label{fig:bound-tightness}
\end{figure}


\section{Illustration of Results}
\label{sec:empirical}


We now illustrate our MMSE bounds for 
the following distributional relationship between $X$ and $S$, where $S \in \{0,1\}$:
(i) binary symmetric channel, 
(ii) class-conditional Gaussians with different covariances, and
(iii) class-conditional mixtures of Gaussians. The models in (i) and (ii) capture simple but clear relationships between $S$ and $X$ where mathematical analysis was possible as detailed in the previous section; in contrast, (iii) captures a more complex and more realistic per-class multi-modal distributions. 

For each $(X,S)$ relationship listed above, to empirically verify the lower bounds presented, we estimate several key terms through simulation. We begin by observing that the following quantities are population-based: $\eta^\sigma$, $h^*_\mathcal{H}$, $\epsilon_A$, and the true $\mathrm{MMSE}(S|X^\sigma)$.  To obtain a high-confidence empirical estimate of these statistical quantities, we sample 2M tuples $(X^\sigma, S)$, of which we use 1M to estimate $h^*_\mathcal{H}$, the optimal estimator for the restricted hypothesis class. We consider two possible hypothesis classes: (i) logistic model, i.e., sigmoid activation composed with a linear model, and (ii) single hidden layer neural network (SHL-NN) with ReLU and sigmoid activations for the hidden layer and output layer, respectively, for $d_w$ number of hidden neurons (width). For the logistic model, we use these 1M tuples to estimate the statistical quantities of the closed-form expression for the optimal linear model $\theta^*_L$ in \eqref{eq:opt-theta-l-form}. For the SHL-NN models, using these 1M tuples, we minimize the square loss via gradient descent (full hyperparameters are listed in \cref{appendix:experimental_setup}) to learn $h^*_\mathcal{H}$.
Using the remaining 1M samples, we estimate the statistical quantities $\mathrm{MMSE}(S|X^\sigma)$ and $\epsilon_A$; to this end, we also use the knowledge of the distributions to derive the closed-form expression for $\eta^\sigma$. The $h^*_\mathcal{H}$ learned from the first 1M samples is used in estimating $\epsilon_A$ using the equality in \eqref{eq:approx-error-bound}. We additionally use $h^*_\mathcal{H}$ to calculate the (Bernstein-based) bound on $\epsilon_C$ in \eqref{eq:eps-c-bound-bernstein} which results in the tightest bounds. Note that the sample variance term in \eqref{eq:eps-c-bound-bernstein} is computed over the train set of length $n$. For all bounds, we set $\delta = 0.05$ and average all results over 30 runs.

To learn the optimal finite sample restricted hypothesis estimator $\hat{h}_\mathcal{H}$, we follow the same training protocol as described above for $h^*_\mathcal{H}$ except for a much smaller  number of samples $n$ (e.g., $n=500$) with learning rate and training epochs viewed as hyperparameters (detailed in \cref{appendix:experimental_setup}). The resulting training error for a chosen number of samples $n$ 
yields $\mathrm{MSE}_\text{train}(\hat h_\mathcal{H})$. Finally, the code for generating our results can be found at \cite{code}.




\vspace{3pt}
\noindent \textbf{Binary Symmetric Channel:} We first illustrate in \cref{fig:BSC} the BSC setting described in \cref{thm:delta-a-bsc-example} for parameter choices $p=1/4$ and $p_N=1/4$. We explore the effect of $p_N$ in \cref{fig:bsc_pn} in \cref{sec:additional_plots}. We choose $n=500$ for which $\epsilon_C\approx0.05$ for a bound with probability 95\%.
The black solid curve is the empirical estimate of the true MMSE using 1M samples of $(X,S)$, while the orange dashed curve is the empirical MMSE, $\text{MSE}_\text{train}(\hat h_\mathcal{H})$, for a linear model with sigmoid activation trained with $n=500$ samples. The solid blue curve is the training bound provided in \cref{thm:mmse-bound-general} computed using the (million sample-based) empirical estimates of $\epsilon_A$ and $\mathrm{MMSE}(S\vert X^\sigma)$.
We note that the bound provided is quite tight and that the majority of the gap is due to $\epsilon_C$ and could be improved with a larger sample size. 


\begin{figure}
    \centering
    \begin{subfigure}{0.48\linewidth}
    \centering
        \includegraphics[width=\linewidth]{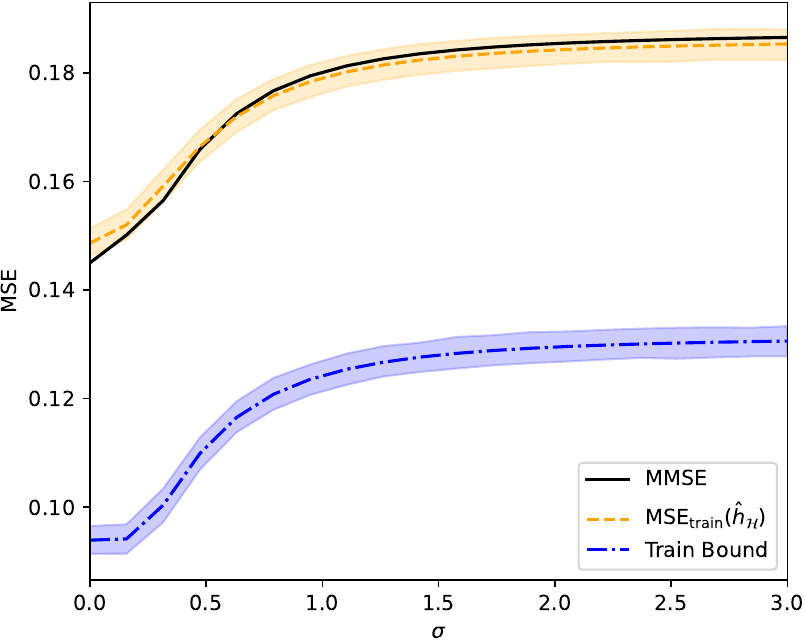}
        \caption{Binary Symmetric Channel for $(X,S)$}
        \label{fig:BSC}
    \end{subfigure}\hfill%
    \begin{subfigure}{0.48\linewidth}
    \centering
    \includegraphics[width=\linewidth]{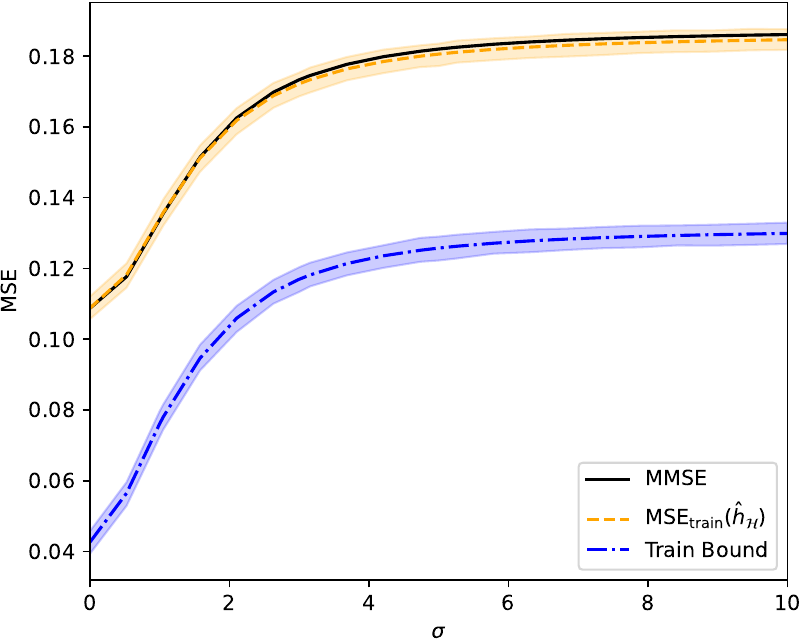}
    \caption{Class-conditional Univariate Gaussian ($d=1$)}
    \label{fig:ccg}
    \end{subfigure}
    \caption{Bound on the MMSE as a function of the noise deviation $\sigma$. For the BSC, the parameters are: $p=1/4$ and  $p_N=1/4$, while for the CCG, the parameters are $p=1/4, \sigma_0^2 = 1, \sigma_1^2 = 3, \mu_0 = -1, \mu_1 = 1$ and $d=1$. Here $\mathrm{MSE}_\text{train}(\hat h_\mathcal{H})$ is calculated with $n=500$ samples. Note that most of the gap between the true MMSE and the given bound comes from $\epsilon_C$ and a linear evaluation model is sufficient for bounding MMSE even for settings where a non-linear model would be optimal.}
\end{figure}
\vspace{3pt}
\noindent\textbf{Class-Conditional Gaussian with Different Covariances:} We next consider the one-dimensional class-conditional Guassian setting as described in \cref{cor:CCG-vector-diag-cov} with parameters $p=1/4, d=1, \sigma_0^2 = 1, \sigma_1^2 = 3, \mu_0 = -1, \mu_1 = 1$.
While a linear model is not optimal in this setting for small $\sigma$, we are still able to effectively bound the true MMSE for every noise level. Note that the loosest bound is at $\sigma=0$, where the approximation error is the largest, following the results of \cref{cor:CCG-vector-diag-cov}.
Despite this, we see in \cref{fig:ccg} that a linear adversarial evaluation model provides a meaningful bound for reasonable values of $\sigma$. Again, a majority of the gap is made up by $\epsilon_C$.

We explore the effect of the data dimensionality $d$ in this dataset for both linear and neural network model classes with hidden layer width $d_w=10$ in \cref{fig:ccg_d_nn_vs_linear}. {To ensure that the means scale appropriately in higher dimensions, we set them such that $\lVert \mu_1 - \mu_0 \rVert^2 = 4$.}
{When comparing against simple neural networks, we observe that as the data dimension $d$ increases, linear models provide better bounds and remain non-vacuous up to $d=20$. In contrast, neural networks begin to overfit very quickly, resulting in vacuous bounds beyond $d>7$. Increasing the number of training samples can help mitigate this overfitting, as illustrated in \cref{fig:ccg_d_n}. With only 500 samples, the training-based bounds for neural networks are very loose and quickly become vacuous. However, with 20,000 samples, the bounds remain tight across all values of $d$. It is also worth noting that even with 20k samples, the bounds degrade as the dimension $d$ increases, due to the approximation error $\epsilon_A$ growing with $d$, as indicated in the linear case by \eqref{eq:CCG-vector-diag-cov-bound} in \cref{cor:CCG-vector-diag-cov} and illustrated for neural networks in \cref{fig:epsilon_a_dimension} in \cref{sec:additional_plots}. Increasing the hidden layer width $d_w$ can help reduce $\epsilon_A$, but doing so requires a larger number of samples $n$ to avoid overfitting—particularly at higher data dimensions $d$, where larger models are more prone to overfitting, as shown in \cref{fig:linear-vs-nn-bounds-dimension} in \cref{sec:additional_plots}.
}

\begin{figure}
    \centering
    \begin{subfigure}[t]{0.48\linewidth}
    \centering
    \includegraphics[width=\linewidth]{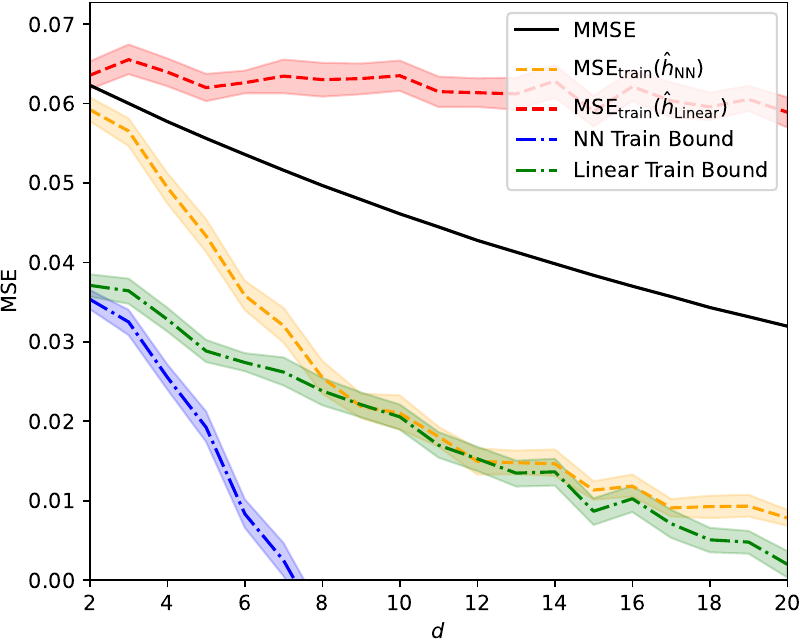}
    \caption{Linear models do not overfit to this simple Gaussian dataset and provide non-vacuous bounds for moderate dimensional datasets even with data limited to $n=1$k. On the other hand, neural networks overfit and thus bounds are vacuous for $d >7$.}
    \label{fig:ccg_d_nn_vs_linear}
    \end{subfigure}\hfill\hfill%
    \begin{subfigure}[t]{0.48\linewidth}
    \centering
    \includegraphics[width=\linewidth]{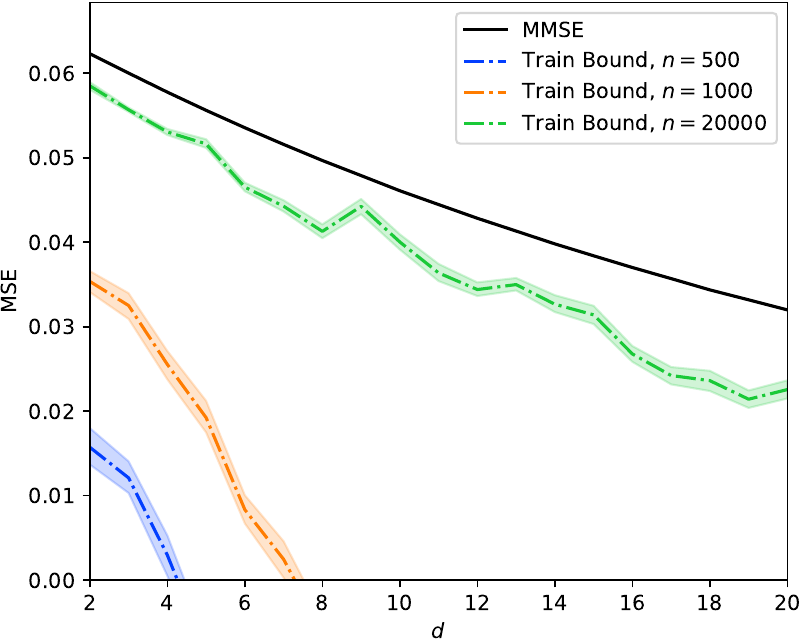}
    \caption{Simple neural networks with only $d_w = 10$ hidden units overfit to small samples, but provide tight bounds provided enough samples. Indeed for 20k samples the neural network provides very tight bounds even for relatively large $d$.}
    \label{fig:ccg_d_n}
    \end{subfigure}
    \caption{Comparison of linear models and single hidden layer neural networks with $d_w=10$ on the class-conditional Gaussian dataset as a function of dimension $d$ for different number of finite samples $n$. }
    \label{fig:ccg_d}
\end{figure}

\vspace{3pt}
\noindent\textbf{Class-Conditional Mixture of Gaussians:} 
While a linear model suffices for class-conditional Gaussian distributions, it is unclear whether it remains effective when the class distributions are multi-modal. To investigate this, we consider a variant of the classic Gaussian XOR dataset, where $X\vert S$ is modeled as a 2D Gaussian mixture; specifically,
\begin{equation}
    X\vert S \sim \sum_{k=1}^{n_m} \frac{1}{n_m} \mathcal{N}\left(\mu_k,\frac{1}{n_m^2}I\right),
    \label{eq:xor-dist}
\end{equation} 
where $\mu_k$ are the component means, evenly spaced on a circle of fixed radius (set to 2 in \cref{fig:xor-sample-3-modes,fig:xor-sample-4-modes}), and $n_m$ is the number of Gaussian components—or modes—per class. The modes are interleaved between classes to form a dataset that is intentionally challenging for linear predictors.
To ensure a consistent scale of difficulty across different values of $n_m$, we scale the variance of each component as $1/n_m^2$ and adjust the standard deviation of the additive noise as $\sigma/n_m$ for a fixed noise level $\sigma$. This ensures that the true MMSE remains on the same order across datasets with varying numbers of modes. In our experiments, we consider datasets with $n_m=3$ and $n_m=4$ modes per class and set $\sigma=2$ and $p=1/2$. Visualizations of $X$ and the corresponding noisy $X^\sigma$ are provided in \cref{fig:xor-sample-3-modes,fig:xor-sample-4-modes} for $n_m=3$ and $n_m=4$, respectively. Since $\eta^\sigma$ can infer $S$ from $X$ with high confidence, we choose a relatively large noise level ($\sigma = 2$) to ensure meaningful obfuscation of the sensitive variable. In \cref{sec:additional_plots}, we present visualizations (\cref{fig:xor-sample-3-modes-radius-1,fig:xor-sample-4-modes-radius-1}) and results (\cref{fig:xor_train_r_1,fig:xor_val_r_1}) for a setting with reduced separability, achieved by using a smaller radius (set to 1), which requires less noise ($\sigma = 0.5$) to attain a comparable level of obfuscation. The results in this less separable case remain qualitatively similar.


\begin{figure}
    \centering
    \begin{subfigure}{0.48\linewidth}
    \centering
    \includegraphics[width=\linewidth]{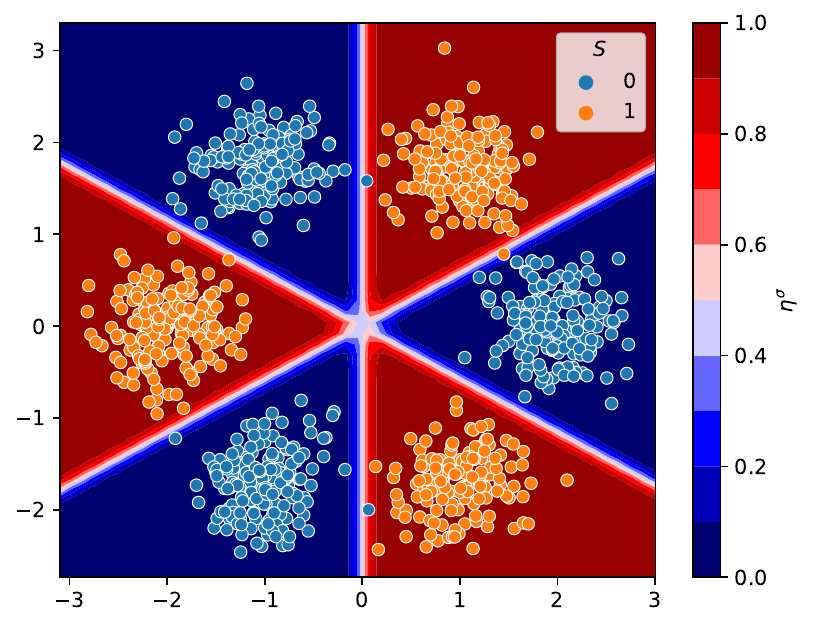}
    \caption{$X$}
    \label{fig:xor-sample-3-no-noise}
    \end{subfigure}\hfill\hfill%
    \begin{subfigure}{0.48\linewidth}
    \centering
    \includegraphics[width=\linewidth]{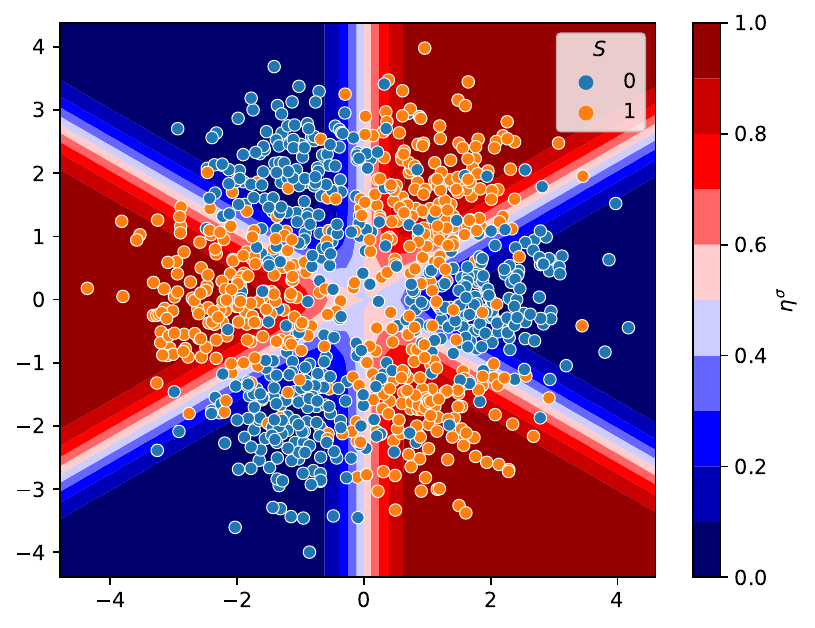}
    \caption{$X^\sigma$}
    \label{fig:xor-sample-3-noise}
    \end{subfigure}
    \caption{A sample of (a) $X$ and (b) $X^\sigma$ with $\sigma = 2$ from our class-conditional mixture dataset with 3 modes per class, where the modes are placed on a circle of radius 2.
    The heat maps (i.e., contours) of $\eta^\sigma$ (the color legend is shown to the right of the figure) are derived from the true statistics of the data.}
    \label{fig:xor-sample-3-modes}
\end{figure}

\begin{figure}
    \centering
    \begin{subfigure}{0.48\linewidth}
    \centering
    \includegraphics[width=\linewidth]{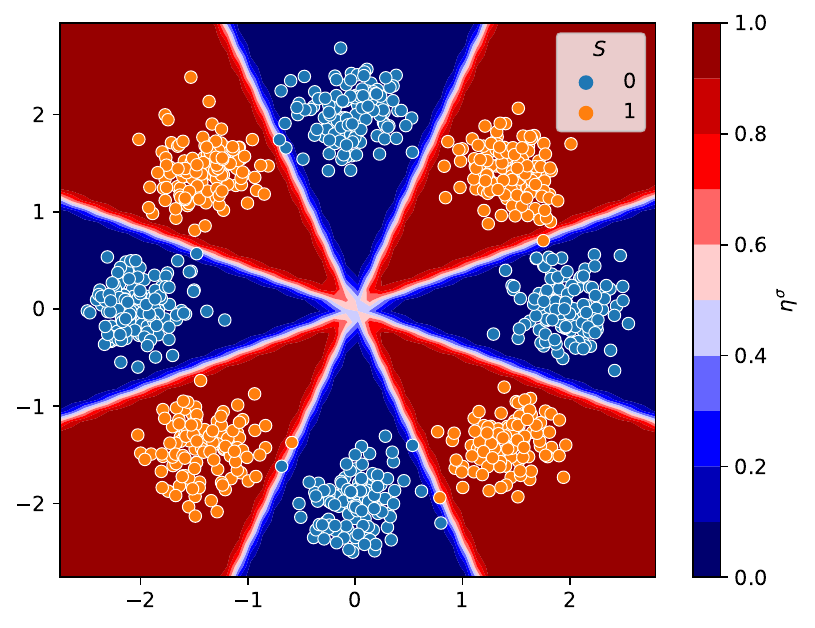}
    \caption{$X$}
    \label{fig:xor-sample-4-no-noise}
    \end{subfigure}\hfill\hfill%
    \begin{subfigure}{0.48\linewidth}
    \centering
    \includegraphics[width=\linewidth]{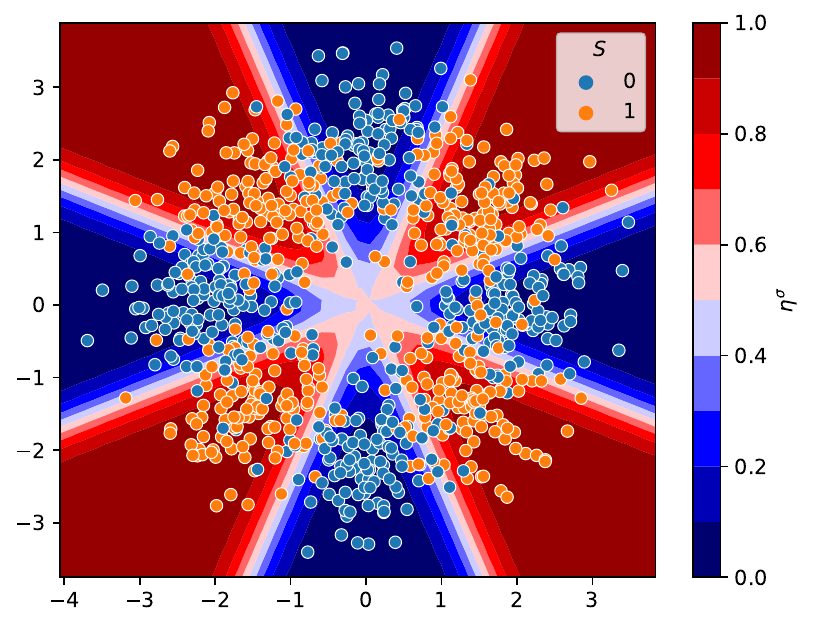}
    \caption{$X^\sigma$}
    \label{fig:xor-sample-4-noise}
    \end{subfigure}
    \caption{A sample of (a) $X$ and (b) $X^\sigma$ with $\sigma = 2$ from our class-conditional mixture dataset with 4 modes per class, where the modes are placed on a circle of radius 2.
    The heat maps (i.e., contours) of $\eta^\sigma$ (the color legend is shown to the right of the figure) are derived from the true statistics of the data.}
    \label{fig:xor-sample-4-modes}
\end{figure}

\begin{figure}
    \centering
    \begin{subfigure}[t]{0.48\linewidth}
    \includegraphics[width=\linewidth]{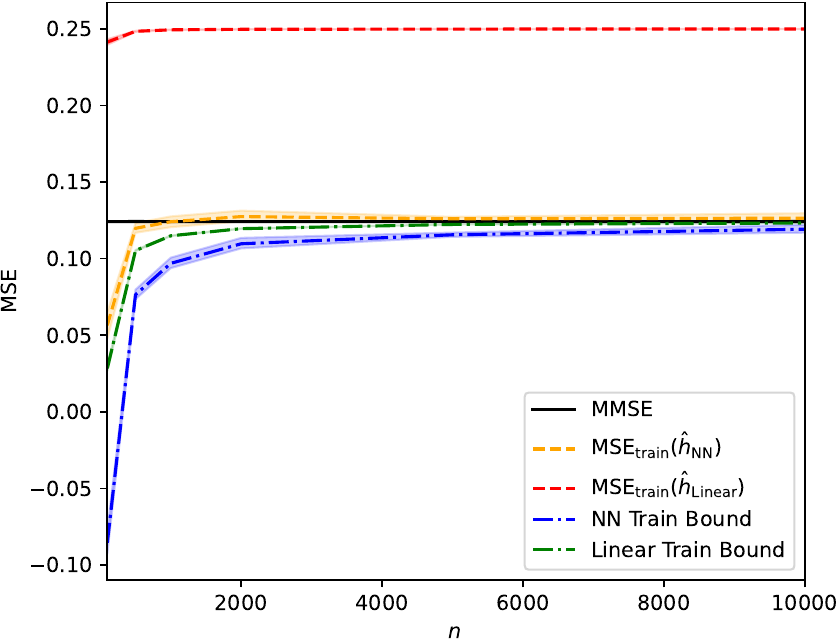}
    \caption{3 modes per class.}
    \label{fig:xor_train_3}
    \end{subfigure}\hfill\hfill%
    \begin{subfigure}[t]{0.48\linewidth}
    \includegraphics[width=\linewidth]{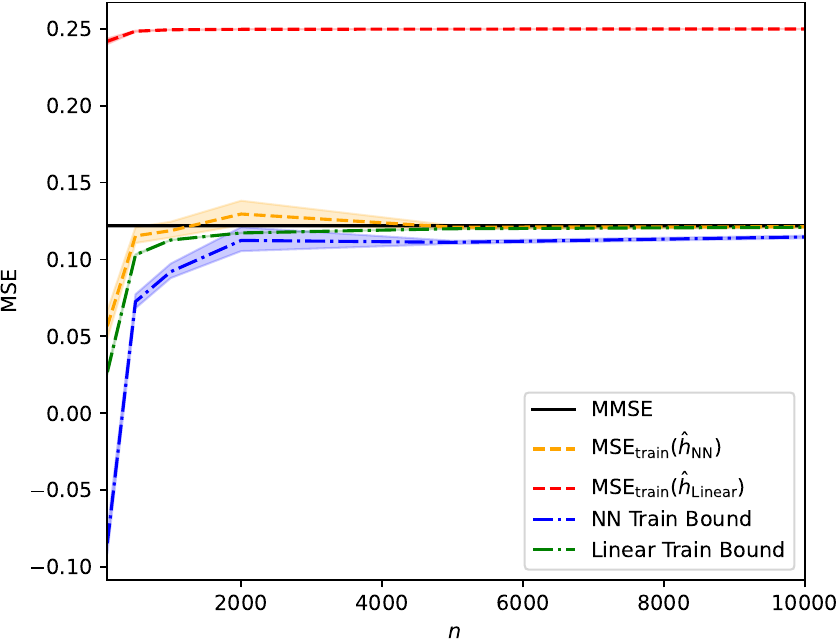}
    \caption{4 modes per class.}
    \label{fig:xor_train_4}
    \end{subfigure}
    \caption{Training MSE-based lower bounds for the class-conditional mixture of Gaussians dataset for radius 2, $\sigma=2$, and $d_w = 10$. For both settings ($n_m=3$ and $n_m=4$), linear models yield tighter bounds for small $n$, despite having a significantly larger approximation error $\epsilon_A$. Neural networks, while initially affected by overfitting, exhibit much smaller approximation error and achieve competitive bounds as $n$ increases.}
    \label{fig:xor_train}
\end{figure}
\begin{figure}
    \centering
    \begin{subfigure}[t]{0.48\linewidth}
    \includegraphics[width=\linewidth]{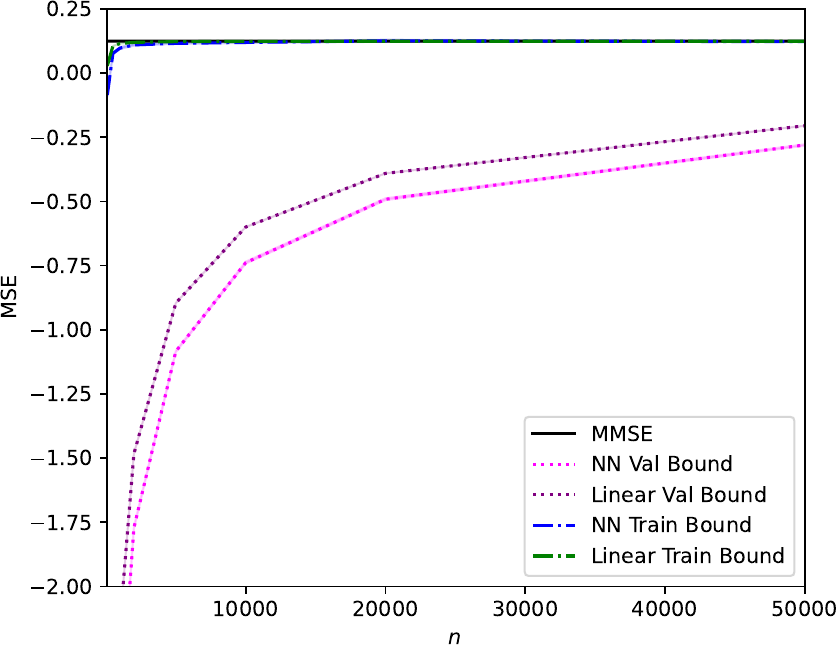}
    \caption{3 modes per class.}
    \label{fig:xor_val_3}
    \end{subfigure}\hfill\hfill%
    \begin{subfigure}[t]{0.48\linewidth}
    \includegraphics[width=\linewidth]{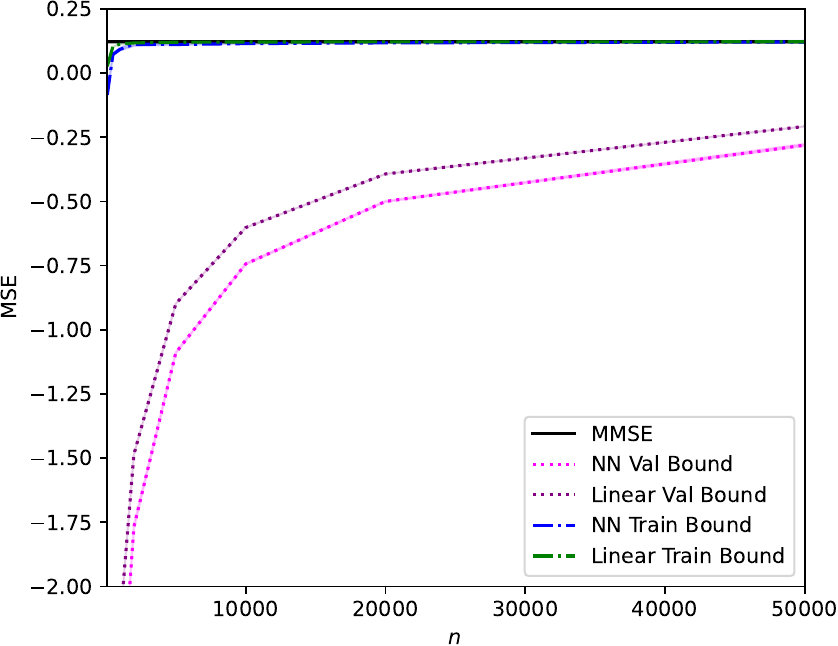}
    \caption{4 modes per class.}
    \label{fig:xor_val_4}
    \end{subfigure}
    \caption{Comparison of training MSE- and validation MSE-based lower bounds for the class-conditional mixture of Gaussians dataset with radius 2, $\sigma=2$ and the neural network hidden-layer width $d_w = 10$. Results are shown for both linear models and shallow neural networks. Despite using as many as $n=50$k training and $m=1$k validation samples,  the validation bounds remain vacuous due to the looseness of the generalization error bound on $\epsilon_G$ in \eqref{eq:eps-g-bound}, which accounts for the complexity of the hypothesis class.}
    \label{fig:xor_val}
\end{figure}

When data is limited, the issue of overfitting becomes evident for the neural network. As shown in \cref{fig:xor_train}, using only $n=500$ training samples leads to very loose lower bounds for the shallow neural network, primarily because the training MSE is so low. In contrast, the linear model does not overfit, so the gap between the true MMSE and the lower bound is primarily driven by the finite-sample error $\epsilon_C$ in \eqref{eq:eps-c}. As $n$ increases, the bounds for both models improve and begin to approach the true MMSE. Similar behavior is observed in the less separable, radius-1 setting, as shown in \cref{fig:xor_train_r_1}.

In this more complex dataset, however, the linear model exhibits a significantly larger approximation error $\epsilon_A$ of approximately 0.1258, compared to only 0.0006 for the neural network (for the 3-mode case). While the optimal linear model can be computed in closed form via \eqref{eq:opt-theta-l-form}, doing so requires access to $\theta^\sigma$, which is not available in practice. This highlights a \emph{key advantage of the neural network bounds}: although they may be looser than the linear model bounds, their approximation error is effectively negligible and can often be treated as zero. This makes the neural network bounds more applicable in practical settings where $\theta^\sigma$ is unknown.



\textit{Training vs. validation}: Next, we compare the efficacy of training and validation bounds (from \eqref{eq:2LNNMainResultRepresentation} and \eqref{eq:mmse-lower-bound-val-error}, respectively) in \cref{fig:xor_val} for various $n$ and $m=1$k. To compute the validation bounds, we use \eqref{eq:eps-g-bound} where $C(\hat h_\mathcal{H})$ is calculated by saving the model weights and compressing them in \texttt{pth.tar.gz} format. The average size for the neural networks is around 3KB, which results in $C(\hat h_\mathcal{H})\approx 24$k. Note that more sophisticated quantization methods such as those in \cite{lotfi2022pac} would yield a smaller constant, but current iterative approaches to distill model size may not achieve the smallest $\mathrm{MSE}_\text{train}$ in the hypothesis class. This would violate our assumptions and thus we are restricted to lossless compression of the model weights. However, this does suggest that other approaches to learn smaller models by restricting the hypothesis class directly could be promising to tighten these bounds.

We observe that the training-based bound is very tight, owing to the large number of samples, while the validation-based bound is vacuous due to the looseness of the compression-based generalization bound. Similar behavior is observed in the less separable, radius-1 setting, as shown in \cref{fig:xor_val_r_1}. This could potentially be improved by employing generalization bounds that more effectively exploit the structure of our loss function (square loss) and the dataset. Our empirical findings suggest that, in practice, training-based bounds—even when using shallow neural networks—may offer more reliable guarantees than validation-based bounds, given the current limitations of generalization bounds in the literature.


\textit{Computability of bounds}: Thus far, we have estimated the statistical quantity $\epsilon_A$ in the lower bound using a large number of samples (1M). However, it is still possible to obtain practical bounds in realistic settings. This is because $\epsilon_A$ can be made arbitrarily small by selecting a sufficiently expressive model class $\mathcal{H}$, reducing the bound in \cref{thm:mmse-bound-general} to depend only on $\mathrm{MSE}_\text{train}(\hat{h}_\mathcal{H})$ and $\epsilon_C$, both of which are computable from finite samples. The main drawback is that preventing overfitting with a large model class requires a correspondingly large number of training samples. Obtaining tight, computable bounds with only limited data remains challenging and is an important direction for future work.

\section{Concluding Remarks}
\label{sec:conclusion}
We have presented a framework for adversarially evaluating the inference of a sensitive feature from a noisy data release using MMSE as the inferential measure. The methodology we have introduced subsumes both regression and multi-class inference setting for bounded sensitive feature. 
We have shown that providing meaningful lower bounds on the MMSE depends on the number of samples $n$ available to perform the evaluation and the complex relationship between $X$ and $S$ as modeled by a chosen hypothesis class $\mathcal{H}$. In particular, these two quantities determine the gap between the true MMSE and the empirical training MSE lower bound with $n$ influencing the finite-sample estimation error and $\mathcal{H}$ determining the error in approximating the best MMSE estimator by the best in the class $\mathcal{H}$. For a given $n$, a complex hypothesis class (e.g., neural network) can significantly reduce the approximation error but may also drive the empirical training MSE to 0. To address this, we have additionally leveraged SOTA generalization results to provide a validation MSE based bounds. 

Our empirical results for synthetic datasets with increasingly complex relationships between $X$ and $S$ (i.e., predicting $S$ from $X$ requires more complex non-linear functions) suggest that a linear hypothesis class suffices in moderately complex data settings especially when obfuscation of $S$ is essential, i.e., large $\sigma$ (e.g., class-conditional Gaussian). Even for highly multi-modal data distributions (e.g., class-conditioned mixture of Gaussians), our results suggest that linear models provide tighter bounds but may not be computable in practice because of the the need for either distribution knowledge or large sample size to estimate the approximation error. In contrast, neural network-based bounds are nearly as tight and can be computed efficiently from finite samples and eliminate the need for computing the approximation error, if chosen appropriately. Interestingly, in both cases, the training MSE-based bound is better than the validation-based one owing to the looseness of even the SOTA generalization bounds.

More broadly, the lower bounds we have presented on the true MMSE are also relevant to communication systems ($S$ and $X$ are the transmitted bit and signal, respectively, over a Gaussian channel) where MMSE continues to be a key measure of goodness. In particular, as these systems employ complex non-linear encoders and decoders, lower bounds on the true MMSE can provide a simple means of evaluating the complex models for the finite blocklength setting.

\section*{Acknowledgment}
This work is supported in part by NSF grants CIF-1901243, CIF-2007688, CIF-2312666, and SCH-2205080. We lost one of our collaborators, Prof. Mario Diaz, halfway through this effort but he continues to be the inspiration for this work. We also wish to thank Prof. Vidya Muthukumar (Georgia Tech) for meaningful discussions on obtaining lower bounds using validation data and Prof. James Melbourne (CIMAT) for discussions on various bounds on the MMSE including Poincar\'e inequalities.

\bibliographystyle{IEEEtran}
\bibliography{Bibliography}

\appendix
\subsection{Proof of Theorem \ref{thm:mmse-bound-general}}
\label{appendix:thm1-proof}
For ease of notation, we define
\begin{equation}
\label{eq:Proof2LNNDelta}
    \Delta_\text{train} \coloneqq \mathrm{MSE}_\textup{train}(\hat{h}_\mathcal{H}) - \mathrm{MMSE}(S \vert X^\sigma).
\end{equation}
Note that we can rewrite $\Delta_\text{train}$ as follows:
\begin{equation}
\label{eq:delta_train_decomp}
    \Delta_\text{train} = \underset{\text{finite sample error}}{\underbrace{\mathrm{MSE}_\textup{train}(\hat{h}_\mathcal{H}) - \mathrm{MMSE}_\mathcal{H}(S \vert X^\sigma)}} + \underset{\text{approximation error}}{\underbrace{\mathrm{MMSE}_\mathcal{H}(S \vert X^\sigma) - \mathrm{MMSE}(S \vert X^\sigma)}},
\end{equation}
where $\mathrm{MMSE}_\mathcal{H}(S\vert X^\sigma)$ is the minimum mean squared-error achieved by a hypothesis $h \in \mathcal{H}$ as defined in \eqref{eq:mmse-restricted-h}.
Recall that $h^*_\mathcal{H}$ defined in \eqref{eq:opt-h} attains the $\mathrm{MMSE}_\mathcal{H}(S|X^\sigma)$.
Then, since $|\mathcal{D}_\text{train}| = n$,
\begin{align}       
\text{finite sample error} &= \frac{1}{n} \sum_{i=1}^{n} (S_{i} - \hat{h}_\mathcal{H}(X^\sigma_{i}))^{2} - \mathbb{E}[(S-h^*_\mathcal{H}(X^\sigma))^{2}] \\
& \le \frac{1}{n} \sum_{i=1}^{n} (S_{i} - h^*_\mathcal{H}(X^\sigma_{i}))^{2} - \mathbb{E}[(S-h^*_\mathcal{H}(X^\sigma))^{2}] \\
& = \mathrm{MSE}_\text{train}(h^*_\mathcal{H}) - \mathbb{E}[(S-h^*_\mathcal{H}(X^\sigma))^{2}] \\
& = \epsilon_C, \label{eq:finite-sample-bound-eps-c}
\end{align}
where $\epsilon_C$ is defined in \eqref{eq:eps-c}.


Recall that the $\mathrm{MMSE}(S \vert X^\sigma)$ is attained by the conditional expectation $\eta^\sigma$ in \eqref{eq:DefEta} and note that
\begin{align}
    \mathrm{MMSE}_\mathcal{H}(S \vert X^\sigma) &= \mathbb{E}[(S-h^*_\mathcal{H}(X))^{2}] \\
    & = \mathbb{E}[(S-h^*_\mathcal{H}(X^\sigma))^{2} - (S-\eta^\sigma(X^\sigma))^2 + (S-\eta^\sigma(X^\sigma))^2] \\
    & = \mathbb{E}[(\eta(X^\sigma) - h^*_\mathcal{H}(X^\sigma))^2 + 2(\eta^\sigma(X^\sigma)-h^*_\mathcal{H}(X^\sigma))(S-\eta^\sigma(X^\sigma)) + (S-\eta^\sigma(X^\sigma))^2] \\
    & = \mathbb{E}[(\eta(X^\sigma) - h^*_\mathcal{H}(X^\sigma))^2] + 2\mathbb{E}\left[(\eta^\sigma(X^\sigma)-h^*_\mathcal{H}(X^\sigma))\mathbb{E}[S-\eta^\sigma(X^\sigma)|X^\sigma]\right] + \mathbb{E}[(S-\eta^\sigma(X^\sigma))^2] \\
    & = \mathbb{E}[(\eta^\sigma(X^\sigma) - h^*_\mathcal{H}(X^\sigma))^2] + 2\mathbb{E}\left[(\eta^\sigma(X^\sigma)-h^*_\mathcal{H}(X^\sigma))(\mathbb{E}[S|X^\sigma]-\eta^\sigma(X^\sigma))\right] + \mathbb{E}[(S-\eta(X^\sigma))^2] \\
    & = \mathbb{E}[(\eta^\sigma(X^\sigma) - h^*_\mathcal{H}(X^\sigma))^2] + \mathbb{E}[(S-\eta^\sigma(X^\sigma))^2].
    \label{eq:mmse-h-decomp}
\end{align}
Observe that \eqref{eq:mmse-h-decomp} implies that $h^*_\mathcal{H} = \argmin_{h \in \mathcal{H}} \, \mathbb{E}[(\eta^\sigma(X^\sigma) - h(X^\sigma))^2]$. We then have that
\begin{align}
    \text{approximation error}  &= \mathbb{E}[(\eta^\sigma(X^\sigma) - h^*_\mathcal{H}(X^\sigma))^2] + \mathbb{E}[(S-\eta^\sigma(X^\sigma))^2] - \mathbb{E}[(S-\eta^\sigma(X^\sigma))^2] \\
    & = \mathbb{E}[(\eta^\sigma(X^\sigma) - h^*_\mathcal{H}(X^\sigma))^2] \\
    & = \epsilon_A,
    \label{eq:ProofGeneralApproximation}
\end{align}
where $\epsilon_A$ is the approximation error as defined in \eqref{eq:eps-a}.
Therefore, using \eqref{eq:finite-sample-bound-eps-c} and \eqref{eq:ProofGeneralApproximation}, we can upper bound \eqref{eq:Proof2LNNDelta} as
\begin{equation}
\label{inq:RepresentationApproximation}
    \Delta_\text{train} = \text{finite sample error} + \text{approximation error} \leq  \epsilon_C + \epsilon_A.
\end{equation}

\subsection{Proof of Theorem \ref{thm:mmse-bound-general-val}}
\label{appendix:mmse-bound-general-val-proof}
For ease of notation, we define
\begin{equation}
\label{eq:val-bound-diff}
    \Delta_\text{val} \coloneqq \mathrm{MSE}_\textup{val}(\hat{h}_\mathcal{H}) - \mathrm{MMSE}(S \vert X^\sigma).
\end{equation}
Note that we can rewrite $\Delta_\text{val}$ as follows:
\begin{equation}
\label{eq:val-bound-decomp}
     \Delta_\text{val} = \underset{\text{finite sample + generalization error}}{\underbrace{\mathrm{MSE}_\textup{val}(\hat{h}_\mathcal{H}) - \mathrm{MMSE}^\mathcal{H}(S \vert X^\sigma)}} + \underset{\text{approximation error}}{\underbrace{\mathrm{MMSE}^\mathcal{H}(S \vert X^\sigma) - \mathrm{MMSE}(S \vert X^\sigma)}},
\end{equation}
where $\mathrm{MMSE}_\mathcal{H}(S\vert X^\sigma)$ is the minimum mean squared-error achieved by a hypothesis $h \in \mathcal{H}$ as defined in \eqref{eq:mmse-restricted-h}. Recall that $h^*_\mathcal{H}$ defined in \eqref{eq:opt-h} attains the $\mathrm{MMSE}_\mathcal{H}(S|X^\sigma)$.
We can decompose the finite sample and generalization error term as follows:
\begin{align}
\label{eq:finite-sample-gen-decomp}
     \text{finite sample + generalization error} &= \mathrm{MSE}_\textup{val}(\hat{h}_\mathcal{H}) - \mathbb E[(S-\hat{h}_\mathcal{H}(X^\sigma))^2] + \mathbb E[(S-\hat{h}_\mathcal{H}(X^\sigma))^2] - \mathrm{MMSE}^\mathcal{H}(S \vert X^\sigma) \\
     & = \Tilde{\epsilon}_C + \underset{\textbf{I}}{\underbrace{\mathbb E[(S-\hat{h}_\mathcal{H}(X^\sigma))^2] - \mathrm{MMSE}^\mathcal{H}(S \vert X^\sigma)}},
\end{align}
where $\Tilde{\epsilon}_C$ is defined in \eqref{eq:eps-c-tilde}. Note that $\Tilde{\epsilon}_C$ can be bounded using large deviation inequalities since
the $|\mathcal{D}_\text{val}|=m$ samples used to compute $\mathrm{MSE}_\textup{val}(\hat{h}_\mathcal{H})$ are independent of the $|\mathcal{D}_\text{train}|=n$ samples used to train $\hat{h}_\mathcal{H}$.

Further decomposing $\textbf{I}$, we obtain
\begin{align}
    \textbf{I} &= \mathbb E[(S-\hat{h}_\mathcal{H}(X^\sigma))^2] - \mathrm{MSE}_\textup{train}(\hat{h}_\mathcal{H}) + \underset{\le 0}{\underbrace{\mathrm{MSE}_\textup{train}(\hat{h}_\mathcal{H}) - \mathrm{MSE}_\textup{train}({h}^*_\mathcal{H})}} + \mathrm{MSE}_\textup{train}({h}^*_\mathcal{H}) - \mathrm{MMSE}^\mathcal{H}(S \vert X^\sigma) \nonumber \\
    & \le \mathbb E[(S-\hat{h}_\mathcal{H}(X^\sigma))^2] - \mathrm{MSE}_\textup{train}(\hat{h}_\mathcal{H}) + \mathrm{MSE}_\textup{train}({h}^*_\mathcal{H}) - \mathrm{MMSE}^\mathcal{H}(S \vert X^\sigma) \nonumber \\
    & = \epsilon_G + \epsilon_C,
\end{align}
where $\epsilon_G$ and $\epsilon_C$ are defined in \eqref{eq:eps-g} and \eqref{eq:eps-c}, respectively.
We then have that
\begin{align}
    \text{finite sample + generalization error} &\le \Tilde{\epsilon}_C + \epsilon_G + \epsilon_C.
    \label{eq:gen-error-val}
\end{align}
As established in the proof of Theorem \ref{thm:mmse-bound-general}, the approximation error is given by
\begin{equation}
    \text{approximation error} = \epsilon_A,
    \label{eq:approx-error-val}
\end{equation}
where $\epsilon_A$ is defined in \eqref{eq:eps-a}.
By combining the bound from \eqref{eq:gen-error-val} with \eqref{eq:approx-error-val}, we obtain
\begin{equation}
     \Delta_\text{val} = \text{finite sample} + \text{generalization error} + \text{approximation error} \le \Tilde{\epsilon}_C + \epsilon_G +\epsilon_C +  \epsilon_A.
\end{equation}



\subsection{Proof of Proposition \ref{prop:delta-a-linear-optimal-example}}
\label{appendix:proof-delta-a-linear-optimal}
Consider setting (ii) first, where $(X\vert S=s) \sim N(\mu_s,\Sigma)$ for mean $\mu_s \in \mathbb{R}^d$ and a symmetric, positive definite covariance matrix $\Sigma \in \mathbb{R}^{d\times d}$. Let $\Tilde{\Sigma} \coloneqq \Sigma+\sigma^2I$. We have that for $i \in \{0,1\}$,
\begin{equation}
    f_{i}^{\sigma}(x) = \frac{1}{(2\pi)^{d/2}|\Tilde{\Sigma}|^{1/2}} e^{-\frac{1}{2}(x-\mu_i)^\mathrm{T}\Tilde{\Sigma}^{-1}(x-\mu_i)}.
\end{equation}
By \eqref{eq:DefTheta}, for all $x\in\mathbb{R}^d$,
\begin{equation}
    \theta^\sigma(x) = (\mu_1-\mu_0)^\mathrm{T}\Tilde{\Sigma}^{-1}x + \frac{1}{2}\left(\mu_0^\mathrm{T}\Tilde{\Sigma}^{-1}\mu_0-\mu_1^\mathrm{T}\Tilde{\Sigma}^{-1}\mu_1\right) + \log\left(\frac{p}{\bar{p}}\right).
\end{equation}
Computing each term in \eqref{eq:delta-a-decomp}, we get
\begin{align}
    \text{Var}(\theta^\sigma(X^\sigma)) = (\mu_1-\mu_0)^\mathrm{T}\Tilde{\Sigma}^{-1}(\text{Var}(X) + \sigma^2 I)\Tilde{\Sigma}^{-1}(\mu_1-\mu_0),
\end{align}
\begin{align}
    \text{Cov}(\theta^\sigma(X^\sigma),X^\sigma) = \text{Cov}(X^\sigma,\theta^\sigma(X^\sigma))^\mathrm{T} =(\mu_1-\mu_0)^\mathrm{T}\Tilde{\Sigma}^{-1}(\text{Var}(X) + \sigma^2 I),
\end{align}
\begin{align}
    \text{Var}(X^\sigma) = \text{Var}(X)+ {\sigma^2}I,
\end{align}
which yields
\begin{align}
    \epsilon_A &\le (\mu_1-\mu_0)^\mathrm{T}\Tilde{\Sigma}^{-1}(\text{Var}(X) + \sigma^2 I)\Tilde{\Sigma}^{-1}(\mu_1-\mu_0) \nonumber \\
    & \qquad \qquad- (\mu_1-\mu_0)^\mathrm{T}\Tilde{\Sigma}^{-1}(\text{Var}(X) + \sigma^2 I)(\text{Var}(X)+ {\sigma^2}I)^{-1}(\text{Var}(X) + \sigma^2 I)\Tilde{\Sigma}^{-1}(\mu_1-\mu_0) = 0.
\end{align}
This implies $\epsilon_A = 0$.

Note that once noise is added to $X$, setting (i), where $X = aS + b$ for $a,b \in \mathbb R^d$, becomes a special case of setting (ii) with $\mu_1=a+b$, $\mu_0=b$, and $\Tilde{\Sigma}=\sigma^2 I$. Therefore, the same result follows.

\subsection{Proof of Theorem \ref{thm:delta-a-bsc-example}}
\label{appendix:prrof-bsc-example}
Since $X = S \oplus N$, where $S\sim\text{Ber}(p)$ and $N\sim\text{Ber}(p_N)$, $X$ is also a Bernoulli random variable with the following PMF:
\begin{align}
    P_X(x) = \begin{cases}
        q \coloneqq p_N \overline{p} + \overline{p_N} p, & x = 1 \\
        \overline{q} \coloneqq 1-q, & x = 0.
    \end{cases}
\end{align}
The conditional densities of $X^\sigma$ given $S=i$ for $i \in \{0,1\}$ are given by
\begin{equation}
    f_{1}^{\sigma}(x) = \frac{1}{\sqrt{2\pi\sigma^{2}}} \left[p_N e^{-x^2/2\sigma^2}+ \overline{p_N} e^{-(x-1)^2/2\sigma^2} \right]
\end{equation}
and
\begin{equation}
    f_{0}^{\sigma}(x) = \frac{1}{\sqrt{2\pi\sigma^{2}}} \left[\overline{p_N} e^{-x^2/2\sigma^2}+ {p_N} e^{-(x-1)^2/2\sigma^2} \right].
\end{equation}
By \eqref{eq:DefTheta}, for all $x\in\mathbb{R}$,
\begin{align}
    \theta^\sigma(x) = \log\left(\frac{p\left[p_N e^{-x^2/2\sigma^2}+ \overline{p_N} e^{-(x-1)^2/2\sigma^2} \right]}{\overline{p}\left[\overline{p_N}e^{-x^2/2\sigma^2} + {p_N} e^{-(x-1)^2/2\sigma^2} \right]} \right)= \log\left(\frac{p\left[p_N + \overline{p_N} e^{(2x-1)/2\sigma^2} \right]}{\overline{p}\left[\overline{p_N} + {p_N} e^{(2x-1)/2\sigma^2} \right]} \right).
    \label{eq:theta-sig-bsc}
\end{align}
Rewriting \eqref{eq:theta-sig-bsc} as 
\begin{align}
    \theta^\sigma(x) = \log\left( \frac{p}{\overline{p}} \right) + \log\left(1+\overline{p_N}(\beta_x-1)\right) - \log\left(1+{p_N}(\beta_x-1)\right),
\end{align}
where $\beta_x = e^{(2x-1)/2\sigma^2}$, and using the Maclaurin series expansion of $\log(1+x)$ yields
\begin{align}
    \theta^\sigma(x) = \log\left( \frac{p}{\overline{p}} \right)+ \sum_{n=1}^\infty \frac{(-1)^n}{n} (\beta_x-1)^n [p_N^n - (1-p_N)^n].
    \label{eq:theta-sig-series-expansion}
\end{align}

Using the law of total variance, we have that
\begin{equation}
    \text{Var}(\theta^\sigma(X^\sigma)) =  \text{Var}(\mathbb E\left[ \theta^\sigma(X^\sigma)|X\right]) + \mathbb E\left[\text{Var}\left(\theta^\sigma(X^\sigma)|X\right)\right].
    \label{eq:law-of-total-variance}
\end{equation}
We can write the first term in \eqref{eq:law-of-total-variance} as
\begin{align}
    \text{Var}(\mathbb E\left[ \theta^\sigma(X^\sigma)|X\right]) = q \overline{q} (I_1 - I_0)^2,
\end{align}
where
\begin{align}
    I_0 &\coloneqq \mathbb E[\theta^\sigma(X^\sigma)|X=0] =  \int_{-\infty}^\infty \theta^\sigma(\sigma z) f_Z(z) dz, \label{eq:I0} \\
    I_1 &\coloneqq \mathbb E[\theta^\sigma(X^\sigma)|X=1] =  \int_{-\infty}^\infty \theta^\sigma(1+\sigma z) f_Z(z) dz, \label{eq:I1}
\end{align}
and $f_Z$ is the density of the standard normal random variable $Z$.

To better understand the behavior of \eqref{eq:I0} and \eqref{eq:I1} as functions of the noise parameter $\sigma$, we use the series expansion of $\theta^\sigma$ in \eqref{eq:theta-sig-series-expansion} and express $\Tilde{I_0} \coloneqq I_0 - \log( {p}/{\overline{p}} )$ as follows:
\begin{align*}
    \Tilde{I_0}  &= \int_{-\infty}^\infty\sum_{n=1}^\infty \frac{(-1)^n}{n} \left(e^{(2\sigma z-1)/2\sigma^2}-1\right)^n [p_N^n - (1-p_N)^n] \, \frac{e^{-z^2/2}}{\sqrt{2\pi}} \;dz \\
    & =  \sum_{n=1}^\infty \frac{(-1)^n}{n} [p_N^n - (1-p_N)^n] \int_{-\infty}^\infty\left(e^{(2\sigma z-1)/2\sigma^2}-1\right)^n  \, \frac{e^{-z^2/2}}{\sqrt{2\pi}} \;dz \\
    & \overset{(a)}{=} \sum_{n=1}^\infty \frac{(-1)^n}{n} [p_N^n - (1-p_N)^n] \int_{-\infty}^\infty \sum_{k=0}^n \binom{n}{k} e^{(n-k)(2\sigma z-1)/2\sigma^2}(-1)^k  \, \frac{e^{-z^2/2}}{\sqrt{2\pi}} \;dz \\
    & {=}  \sum_{n=1}^\infty \frac{(-1)^n}{n} [p_N^n - (1-p_N)^n]\sum_{k=0}^n \binom{n}{k} (-1)^k \int_{-\infty}^\infty \exp{\left(\frac{-(z\sigma-(n-k))^2+(n-k)^2-(n-k)}{2\sigma^2}\right)}\frac{1}{\sqrt{2\pi}} \;dz \\
    & {=} \sum_{n=1}^\infty \frac{(-1)^n}{n} [p_N^n - (1-p_N)^n]\sum_{k=0}^n \binom{n}{k} (-1)^k e^{(k-n)(1+k-n)/{2\sigma^2}} \\
    & \overset{(b)}{=} \sum_{n=1}^\infty \frac{(-1)^n}{n} [p_N^n - (1-p_N)^n]\sum_{k=0}^n \binom{n}{k} (-1)^k \sum_{m=0}^\infty \frac{(k-n)^m(1+k-n)^m}{m!2^m \sigma^{2m}} \\
     &{=} \sum_{m=0}^\infty\sum_{n=1}^\infty \frac{(-1)^n}{n} [p_N^n - (1-p_N)^n]\sum_{k=0}^n \binom{n}{k} (-1)^k  \frac{(k-n)^m(1+k-n)^m}{m!2^m \sigma^{2m}},
\end{align*}
where $(a)$ follows from the Binomial theorem and $(b)$ follows from the Maclaurin series expansion of $e^x$. Considering the first few values of $m$, we obtain
\begin{align}
    m=0:& \quad \sum_{n=1}^\infty \frac{(-1)^n}{n} [p_N^n - (1-p_N)^n]\sum_{k=0}^n \binom{n}{k} (-1)^k  = 0 \\
    m=1:& \quad \sum_{n=1}^\infty \frac{(-1)^n}{n} [p_N^n - (1-p_N)^n]\underset{=\begin{cases}
        \dfrac{1}{\sigma^2} & n=2, \\
        0 & \text{o.w.}
    \end{cases}}{\underbrace{\sum_{k=0}^n \binom{n}{k} (-1)^k  \frac{(k-n)(1+k-n)}{2 \sigma^{2}}}} = \frac{2p_N-1}{2\sigma^2} \\
    m=2:& \quad \sum_{n=1}^\infty \frac{(-1)^n}{n} [p_N^n - (1-p_N)^n]\underset{=\begin{cases}
        \dfrac{1}{2\sigma^4} & n=2, \\[1ex]
        \dfrac{3}{\sigma^4} & n=3,4, \\
        0 & \text{o.w.}
    \end{cases}}{\underbrace{\sum_{k=0}^n \binom{n}{k} (-1)^k  \frac{(k-n)^2(1+k-n)^2}{8 \sigma^{4}}}} = \frac{p_N(1-3p_N+2p_N^2)}{2\sigma^4}.
\end{align}
Therefore,
\begin{equation}
    I_0 \approx \log\left( \frac{p}{\overline{p}} \right)+\frac{2p_N-1}{2\sigma^2} + \frac{p_N(1-3p_N+2p_N^2)}{2\sigma^4}.
\end{equation}
Similarly, we can write $I_1$ as
\begin{align}
    I_1 {=} \log\left( \frac{p}{\overline{p}} \right)+ \sum_{m=0}^\infty\sum_{n=1}^\infty \frac{(-1)^n}{n} [p_N^n - (1-p_N)^n]\sum_{k=0}^n \binom{n}{k} (-1)^k  \frac{(n-k)^m(1+n-k)^m}{m!2^m \sigma^{2m}}.
\end{align}
Again considering only the first few values of $m$, we get
\begin{align}
    m=0:& \quad \sum_{n=1}^\infty \frac{(-1)^n}{n} [p_N^n - (1-p_N)^n]\sum_{k=0}^n \binom{n}{k} (-1)^k  = 0 \\
    m=1:& \quad \sum_{n=1}^\infty \frac{(-1)^n}{n} [p_N^n - (1-p_N)^n]\underset{=\begin{cases}
        \dfrac{1}{\sigma^2} & n=1,2, \\
        0 & \text{o.w.}
    \end{cases}}{\underbrace{\sum_{k=0}^n \binom{n}{k} (-1)^k  \frac{(n-k)(1+n-k)}{2 \sigma^{2}}}} = \frac{1-2p_N}{2\sigma^2} \\
    m=2:& \quad \sum_{n=1}^\infty \frac{(-1)^n}{n} [p_N^n - (1-p_N)^n]\underset{=\begin{cases}
        \dfrac{1}{2\sigma^4} & n=1, \\[1ex]
        \dfrac{7}{2\sigma^4} & n=2, \\[1ex]
        \dfrac{6}{\sigma^4} & n=3, \\[1ex]
        \dfrac{3}{\sigma^4} & n=4, \\
        0 & \text{o.w.}
    \end{cases}}{\underbrace{\sum_{k=0}^n \binom{n}{k} (-1)^k  \frac{(n-k)^2(1+n-k)^2}{8 \sigma^{4}}}} = \frac{-p_N(1-3p_N+2p_N^2)}{2\sigma^4}.
\end{align}
Therefore,
\begin{align}
    I_1 \approx \log\left( \frac{p}{\overline{p}} \right)+\frac{1-2p_N}{2\sigma^2} - \frac{p_N(1-3p_N+2p_N^2)}{2\sigma^4}.
\end{align}
Hence,
\begin{align}
    \text{Var}(\mathbb E\left[ \theta^\sigma(X^\sigma)|X\right]) &= q \overline{q} (I_1 - I_0)^2 \nonumber \\
    &\approx q \overline{q}\left(\frac{1-2p_N}{\sigma^2} - \frac{p_N(1-3p_N+2p_N^2)}{\sigma^4} \right)^2 \nonumber \\
    & = q \overline{q}\left(\frac{(1-2p_N)^2}{\sigma^4} - \frac{2p_N(1-2p_N)(1-3p_N+2p_N^2)}{\sigma^6} + \frac{p_N^2(1-3p_N+2p_N^2)^2}{\sigma^8} \right).
    \label{eq:VE-approx-bsc}
\end{align}
We can write the second term in \eqref{eq:law-of-total-variance} as
\begin{align}
    \mathbb E \left[\text{Var}(\theta^\sigma(X^\sigma)|X) \right] = \overline{q}I_{02} + q I_{12} - \overline{q} I_0^2 - q I_1^2 = q(I_{12}-I_1^2) + \overline{q}(I_{02}-I_0^2),
\end{align}
where
\begin{align}
    I_{02} &\coloneqq \mathbb E[\theta^\sigma(X^\sigma)^2|X=0] =  \int_{-\infty}^\infty \theta^\sigma(\sigma z)^2 f_Z(z) dz, \\
    I_{12} &\coloneqq \mathbb E[\theta^\sigma(X^\sigma)^2|X=1] =  \int_{-\infty}^\infty \theta^\sigma(1+\sigma z)^2 f_Z(z) dz,
\end{align}
and $f_Z$ is again the density of the standard Gaussian random variable $Z$.

Again using the series expansion of $\theta^\sigma$ in \eqref{eq:theta-sig-series-expansion}, we can write $I_{02}$ as
\begin{align}
    I_{02} &= \log^2\left(\frac{p}{\overline{p}} \right) + 2\log\left(\frac{p}{\overline{p}} \right)\left(I_0 - \log\left(\frac{p}{\overline{p}} \right)\right) 
\nonumber\\
    & \quad + \int_{-\infty}^\infty \left(\sum_{n=1}^\infty \frac{(-1)^n}{n} \left(e^{(2\sigma z-1)/2\sigma^2}-1\right)^n [p_N^n - (1-p_N)^n]\right)^2 \, \frac{e^{-z^2/2}}{\sqrt{2\pi}} \;dz \nonumber\\
    & = \log^2\left(\frac{p}{\overline{p}} \right) + 2\log\left(\frac{p}{\overline{p}} \right)\left(I_0 - \log\left(\frac{p}{\overline{p}} \right)\right) \nonumber\\
    & \quad + \sum_{n=1}^\infty \frac{[p_N^n - (1-p_N)^n]^2}{n^2} \int_{-\infty}^\infty \left(e^{(2\sigma z-1)/2\sigma^2}-1\right)^{2n} \, \frac{e^{-z^2/2}}{\sqrt{2\pi}} \;dz \nonumber\\
    & \quad + 2\sum_{j < n} \frac{(-1)^{n+j}[p_N^n - (1-p_N)^n][p_N^j - (1-p_N)^j]}{nj} \int_{-\infty}^\infty \left(e^{(2\sigma z-1)/2\sigma^2}-1\right)^{n+j} \, \frac{e^{-z^2/2}}{\sqrt{2\pi}} \;dz \nonumber\\
       & \overset{(a)}{=} \log^2\left(\frac{p}{\overline{p}} \right) + 2\log\left(\frac{p}{\overline{p}} \right)\left(I_0 - \log\left(\frac{p}{\overline{p}} \right)\right) \nonumber\\
    & \quad + \sum_{m=0}^\infty\sum_{n=1}^\infty \Bigg[ \frac{[p_N^n - (1-p_N)^n]^2}{n^2} \sum_{k=0}^{2n} \binom{2n}{k}(-1)^k \frac{(k-2n)^m(1+k-2n)^m}{m!2^m\sigma^{2m}} \nonumber\\
    & \quad + 2\sum_{j < n} \frac{(-1)^{n+j}[p_N^n - (1-p_N)^n][p_N^j - (1-p_N)^j]}{nj} \sum_{k=0}^{n+j} \binom{n+j}{k}(-1)^k \frac{(k-n-j)^m(1+k-n-j)^m}{m!2^m\sigma^{2m}} \Bigg],
    \label{eq:I02-expansion}
\end{align}
where $(a)$ follows from similar analysis to that used for $I_0$.
Considering just the third term (first summation term) in \eqref{eq:I02-expansion} and analyzing it for the first few values of $m$ yields
\begin{align}
    m=0:& \quad \sum_{n=1}^\infty \frac{[p_N^n - (1-p_N)^n]^2}{n^2} \sum_{k=0}^{2n} \binom{2n}{k} (-1)^k  = 0 \\
    m=1:& \quad \sum_{n=1}^\infty \frac{[p_N^n - (1-p_N)^n]^2}{n^2}{{\sum_{k=0}^{2n} \binom{2n}{k} (-1)^k  \frac{(k-2n)(1+k-2n)}{2 \sigma^{2}}}} = \frac{(2p_N-1)^2}{\sigma^2} \\
    m=2:& \quad \sum_{n=1}^\infty \frac{[p_N^n - (1-p_N)^n]^2}{n^2} {{\sum_{k=0}^{2n} \binom{2n}{k} (-1)^k  \frac{(k-2n)^2(1+k-2n)^2}{8 \sigma^{4}}}} = \frac{5(1-2p_N)^2}{4\sigma^4}.
\end{align}
Doing the same for the fourth term (second summation term) in \eqref{eq:I02-expansion} yields
\begin{align}
    m=0:& \quad 2\sum_{n=1}^\infty \sum_{j<n}\frac{(-1)^{n+j}}{nj} [p_N^n - (1-p_N)^n][p_N^j - (1-p_N)^j]\sum_{k=0}^{n+j} \binom{n+j}{k} (-1)^k  = 0 \\
    m=1:& \quad 2\sum_{n=1}^\infty \sum_{j<n}\frac{(-1)^{n+j}}{nj} [p_N^n - (1-p_N)^n][p_N^j - (1-p_N)^j]\sum_{k=0}^{n+j} \binom{n+j}{k} (-1)^k{{  \frac{(k-n-j)(1+k-n-j)}{2 \sigma^{2}}}} \nonumber \\
    & \qquad \qquad = \frac{2(2p_N-1)^2}{\sigma^2} \\
    m=2:& \quad 2\sum_{n=1}^\infty \sum_{j<n}\frac{(-1)^{n+j}}{nj} [p_N^n - (1-p_N)^n][p_N^j - (1-p_N)^j]\sum_{k=0}^{n+j} \binom{n+j}{k} (-1)^k{{  \frac{(k-n-j)^2(1+k-n-j)^2}{8 \sigma^{4}}}} \nonumber \\
    & \qquad \qquad= \frac{(1-2p_N)^2(3-4p_N+4p_N^2)}{2\sigma^4}.
\end{align}
Therefore,
\begin{align}
    I_{02} &\approx \log^2\left(\frac{p}{\overline{p}} \right) + 2\log\left(\frac{p}{\overline{p}} \right)\left(\frac{2p_N-1}{2\sigma^2}  + \frac{p_N(1-3p_N+2p_N^2)}{2\sigma^4}\right) \nonumber \\
    & \qquad + \frac{3(2p_N-1)^2}{\sigma^2} + \frac{(1-2p_N)^2(11-8p_N+8p_N^2)}{4\sigma^4}.
    \label{eq:I02-approximation}
\end{align}
Similarly, we can write $I_{12}$ as
\begin{align}
    I_{12} 
       & = \log^2\left(\frac{p}{\overline{p}} \right) + 2\log\left(\frac{p}{\overline{p}} \right)\left(I_1 - \log\left(\frac{p}{\overline{p}} \right)\right) \nonumber\\
    & \quad + \sum_{m=0}^\infty\sum_{n=1}^\infty \Bigg[ \frac{[p_N^n - (1-p_N)^n]^2}{n^2} \sum_{k=0}^{2n} \binom{2n}{k}(-1)^k \frac{(k-2n)^m(k-2n-1)^m}{m!2^m\sigma^{2m}} \nonumber\\
    & \quad + 2\sum_{j < n} \frac{(-1)^{n+j}[p_N^n - (1-p_N)^n][p_N^j - (1-p_N)^j]}{nj} \sum_{k=0}^{n+j} \binom{n+j}{k}(-1)^k \frac{(k-n-j)^m(k-n-j-1)^m}{m!2^m\sigma^{2m}} \Bigg],
    \label{eq:I12-expansion}
\end{align}
Considering just the third term (first summation term) in \eqref{eq:I12-expansion} and analyzing it for the first few values of $m$ yields
\begin{align}
    m=0:& \quad \sum_{n=1}^\infty \frac{[p_N^n - (1-p_N)^n]^2}{n^2} \sum_{k=0}^{2n} \binom{2n}{k} (-1)^k  = 0 \\
    m=1:& \quad \sum_{n=1}^\infty \frac{[p_N^n - (1-p_N)^n]^2}{n^2}{{\sum_{k=0}^{2n} \binom{2n}{k} (-1)^k  \frac{(k-2n)(k-2n-1)}{2 \sigma^{2}}}} = \frac{(2p_N-1)^2}{\sigma^2} \\
    m=2:& \quad \sum_{n=1}^\infty \frac{[p_N^n - (1-p_N)^n]^2}{n^2} {{\sum_{k=0}^{2n} \binom{2n}{k} (-1)^k  \frac{(k-2n)^2(k-2n-1)^2}{8 \sigma^{4}}}} = \frac{17(1-2p_N)^2}{4\sigma^4}.
\end{align}
Doing the same for the fourth term (second summation term) in \eqref{eq:I12-expansion} yields
\begin{align}
    m=0:& \quad 2\sum_{n=1}^\infty \sum_{j<n}\frac{(-1)^{n+j}}{nj} [p_N^n - (1-p_N)^n][p_N^j - (1-p_N)^j]\sum_{k=0}^{n+j} \binom{n+j}{k} (-1)^k  = 0 \\
    m=1:& \quad 2\sum_{n=1}^\infty \sum_{j<n}\frac{(-1)^{n+j}}{nj} [p_N^n - (1-p_N)^n][p_N^j - (1-p_N)^j]\sum_{k=0}^{n+j} \binom{n+j}{k} (-1)^k{{  \frac{(k-n-j)(k-n-j-1)}{2 \sigma^{2}}}} \nonumber\\
    & \qquad \qquad = \frac{2(2p_N-1)^2}{\sigma^2} \\
    m=2:& \quad 2\sum_{n=1}^\infty \sum_{j<n}\frac{(-1)^{n+j}}{nj} [p_N^n - (1-p_N)^n][p_N^j - (1-p_N)^j]\sum_{k=0}^{n+j} \binom{n+j}{k} (-1)^k{{  \frac{(k-n-j)^2(k-n-j-1)^2}{8 \sigma^{4}}}} \nonumber \\
    & \qquad \qquad= \frac{(1-2p_N)^2(9-4p_N+4p_N^2)}{2\sigma^4}.
\end{align}
Therefore,
\begin{align}
        I_{12} &\approx \log^2\left(\frac{p}{\overline{p}} \right) + 2\log\left(\frac{p}{\overline{p}} \right)\left(\frac{1-2p_N}{2\sigma^2} - \frac{p_N(1-3p_N+2p_N^2)}{2\sigma^4}\right) \nonumber\\
    & \qquad + \frac{3(2p_N-1)^2}{\sigma^2}   + \frac{(1-2p_N)^2(35-8p_N+8p_N^2)}{4\sigma^4}.
    \label{eq:I12-approximation}
\end{align}
Hence,
\begin{align}
    \mathbb E \left[\text{Var}(\theta^\sigma(X^\sigma)|X) \right] &= \overline{q}(I_{02} - I_0^2) + q (I_{12} - I_1^2) \nonumber\\
    & \approx  \frac{3(1-2p_N)^2}{\sigma^2} + \frac{(1-2p_N)^2(5-4p_N+4p_N^2+12q)}{2\sigma^4} + \frac{p_N\overline{p_N}(1-2p_N)^2}{2\sigma^6} - \frac{p_N^2(1-3p_N+2p_N^2)^2}{4\sigma^8}.
    \label{eq:EV-approximation-bsc}
\end{align}
Finally, we have that
\begin{align}
    \text{Cov}(\theta^\sigma(X^\sigma),X^\sigma) &= \text{Cov}(\theta^\sigma(X^\sigma),X) + \sigma \text{Cov}(\theta^\sigma(X^\sigma),Z) \\
    & = q \overline{q}(I_1-I_0) + \sigma\left( \overline{q} \mathbb E[Z\theta^\sigma(\sigma Z)]+ q \mathbb E[Z\theta^\sigma(1+\sigma Z)] \right).
\end{align}
Observe that
\begin{align}
    E[Z\theta^\sigma(\sigma Z)] & = \int_{-\infty}^\infty z \theta^\sigma(\sigma z) f_Z(z) dz \\
    & = \int_{0}^\infty z \theta^\sigma(\sigma z) f_Z(z) dz - \int_{0}^\infty z \theta^\sigma(-\sigma z) f_Z(-z) dz \\
    & = \int_{0}^\infty z (\theta^\sigma(\sigma z)-\theta^\sigma(-\sigma z)) f_Z(z) dz \\
    & = \int_{0}^\infty z \left(\log\left(\frac{\left[p_N+\overline{p_N}e^{(2\sigma z - 1)/2\sigma^2}\right]\left[\overline{p_N}+{p_N}e^{(-2\sigma z - 1)/2\sigma^2}\right]}{\left[p_N+\overline{p_N}e^{(-2\sigma z - 1)/2\sigma^2}\right]\left[\overline{p_N}+{p_N}e^{(2\sigma z - 1)/2\sigma^2}\right]} \right)\right) f_Z(z) dz \\
     & = \int_{0}^\infty z \left(\log\left(\frac{p_N\overline{p_N}(1+e^{-1/\sigma^2})+ p_N^2 e^{(-2\sigma z - 1)/2\sigma^2} +\overline{p_N}^2e^{(2\sigma z - 1)/2\sigma^2}}{p_N\overline{p_N}(1+e^{-1/\sigma^2})+ p_N^2 e^{(2\sigma z - 1)/2\sigma^2} +\overline{p_N}^2e^{(-2\sigma z - 1)/2\sigma^2}} \right)\right) f_Z(z) dz \\
     & = \int_{0}^\infty z \left(\log\left(\frac{p_N\overline{p_N}(1+e^{1/\sigma^2})+ p_N^2 e^{(-2\sigma z + 1)/2\sigma^2} +\overline{p_N}^2e^{(2\sigma z + 1)/2\sigma^2}}{p_N\overline{p_N}(1+e^{1/\sigma^2})+ p_N^2 e^{(2\sigma z + 1)/2\sigma^2} +\overline{p_N}^2e^{(-2\sigma z + 1)/2\sigma^2}} \right)\right) f_Z(z) dz \\
     & = \int_{0}^\infty z \left(\log\left(\frac{\left[p_N+\overline{p_N}e^{(2\sigma z + 1)/2\sigma^2}\right]\left[\overline{p_N}+{p_N}e^{(-2\sigma z + 1)/2\sigma^2}\right]}{\left[p_N+\overline{p_N}e^{(-2\sigma z + 1)/2\sigma^2}\right]\left[\overline{p_N}+{p_N}e^{(2\sigma z + 1)/2\sigma^2}\right]} \right)\right) f_Z(z) dz \\
     & = \int_{0}^\infty z (\theta^\sigma(\sigma z+1)-\theta^\sigma(-\sigma z+1))  f_Z(z) dz \\
     & = \int_{-\infty}^\infty z \theta^\sigma(\sigma z+1) f_Z(z) dz \\
     & = E[Z\theta^\sigma(\sigma Z+1)].
\end{align}
Therefore,
\begin{align}
    \text{Cov}(\theta^\sigma(X^\sigma),X^\sigma)
    & = q \overline{q}(I_1-I_0) + \sigma\left( \overline{q} \mathbb E[Z\theta^\sigma(\sigma Z)]+ q \mathbb E[Z\theta^\sigma(1+\sigma Z)] \right) \\
    & = q \overline{q}(I_1-I_0) + \sigma \mathbb E[Z\theta^\sigma(\sigma Z)].
\end{align}
Using similar analysis to that above and the series expansion in \eqref{eq:theta-sig-series-expansion}, we can write $\mathbb E[Z\theta^\sigma(\sigma Z)]$ as
\begin{align*}
    \mathbb E[Z\theta^\sigma(\sigma Z)] & =  \int_{-\infty}^\infty z \sum_{n=1}^\infty \frac{(-1)^n}{n} \left(e^{(2\sigma z-1)/2\sigma^2}-1\right)^n [p_N^n - (1-p_N)^n] \, \frac{e^{-z^2/2}}{\sqrt{2\pi}} \;dz \\
    & =  \sum_{n=1}^\infty \frac{(-1)^n}{n} [p_N^n - (1-p_N)^n] \, \sum_{k=0}^n \binom{n}{k}(-1)^k \int_{-\infty}^\infty \exp{\left(\frac{-(z\sigma-(n-k))^2+(n-k)^2-(n-k)}{2\sigma^2} \right)}  \frac{z}{\sqrt{2\pi}} \;dz \\
    & = \sum_{n=1}^\infty \frac{(-1)^n}{n} [p_N^n - (1-p_N)^n] \, \sum_{k=0}^n \binom{n}{k}(-1)^k \frac{n-k}{\sigma}e^{(k-n)(1+k-n)/2\sigma^2} \\
    & = \sum_{m=0}^\infty\sum_{n=1}^\infty \frac{(-1)^n}{n} [p_N^n - (1-p_N)^n] \, \sum_{k=0}^n \binom{n}{k}(-1)^k \frac{n-k}{\sigma}\frac{(k-n)^m(1+k-n)^m}{m!2^m\sigma^{2m}}.
\end{align*}
Considering only the first few values of $m$, we get
\begin{align}
    m=0:& \quad \sum_{n=1}^\infty \frac{(-1)^n}{n} [p_N^n - (1-p_N)^n]\underset{=\begin{cases}
        \dfrac{1}{\sigma} & n=1, \\
        0 & \text{o.w.}
    \end{cases}}{\underbrace{\sum_{k=0}^n \binom{n}{k} (-1)^k \frac{(n-k)}{\sigma}}} = \frac{1-2p_N}{\sigma} \\
    m=1:& \quad \sum_{n=1}^\infty \frac{(-1)^n}{n} [p_N^n - (1-p_N)^n]\underset{=\begin{cases}
        \dfrac{2}{\sigma^3} & n=2, \\[1ex]
        \dfrac{3}{\sigma^3} & n=3, \\
        0 & \text{o.w.}
    \end{cases}}{\underbrace{\sum_{k=0}^n \binom{n}{k} (-1)^k \frac{(n-k)}{\sigma} \frac{(k-n)(1+k-n)}{2 \sigma^{2}}}} = \frac{p_N(-1+3p_N-2p_N^2)}{\sigma^3} \\
    m=2:& \quad \sum_{n=1}^\infty \frac{(-1)^n}{n} [p_N^n - (1-p_N)^n]\underset{=\begin{cases}
        \dfrac{1}{\sigma^5} & n=2, \\[1ex]
        \dfrac{21}{2\sigma^5} & n=3, \\[1ex]
        \dfrac{24}{\sigma^5} & n=4, \\[1ex]
        \dfrac{15}{\sigma^5} & n=5, \\
        0 & \text{o.w.}
    \end{cases}}{\underbrace{\sum_{k=0}^n \binom{n}{k} (-1)^k \frac{(n-k)}{\sigma^2}  \frac{(k-n)^2(1+k-n)^2}{8 \sigma^{4}}}} = \frac{-p_N(1-9p_N+26p_N^2-30p_N^3+12p_N^4)}{2\sigma^5}.
\end{align}
Therefore,
\begin{equation}
    \mathbb E[Z\theta^\sigma(\sigma Z)] \approx \frac{1-2p_N}{\sigma} + \frac{p_N(-1+3p_N-2p_N^2)}{\sigma^3} - \frac{p_N(1-9p_N+26p_N^2-30p_N^3+12p_N^4)}{2\sigma^5}.
\end{equation}
Hence,
\begin{align}
    \text{Cov}(\theta^\sigma(X^\sigma),X^\sigma)
    & = q \overline{q}(I_1-I_0) + \sigma \mathbb E[Z\theta^\sigma(\sigma Z)] \nonumber \\
    & \approx  1-2p_N + \frac{p_N(-1+3p_N-2p_N^2)+q \overline{q}(1-2p_N)}{\sigma^2} \nonumber \\
    & \qquad - \frac{p_N(1-9p_N+26p_N^2-30p_N^3+12p_N^4)+2q \overline{q}p_N(1-3p_N+2p_N^2)}{2\sigma^4}.
    \label{eq:cov-approximation}
\end{align}
Note that
\begin{equation}
    \text{Var}(X^\sigma) = \text{Var}(X) +\sigma^2 = q \overline{q} + \sigma^2.
    \label{eq:var-x-sig-bsc}
\end{equation}
By combining \eqref{eq:VE-approx-bsc}, \eqref{eq:EV-approximation-bsc}, \eqref{eq:cov-approximation}, and \eqref{eq:var-x-sig-bsc}, and incorporating additional terms in each summation for larger values of $m$, we obtain the desired result.

\subsection{Proof of Theorem \ref{thm:CCG-vector-general}}
\label{appendix:proof-CCG-vector-general}
   For $s \in \{0,1\}$, let $\Tilde{\Sigma}_s \coloneqq \Sigma_s +\sigma^2 I$. Then
\begin{equation}
    f_{s}^{\sigma}(x) = \frac{1}{(2\pi)^{d/2}|\Tilde{\Sigma}_s|^{1/2}} e^{-\frac{1}{2}(x-\mu_s)^\mathrm{T}\Tilde{\Sigma}_s^{-1}(x-\mu_s)}.
\end{equation}
By \eqref{eq:DefTheta}, for all $x\in\mathbb{R}^d$,
\begin{align}
    \theta^\sigma(x) = x^\mathrm{T}Ax + b^\mathrm{T}x+c,
    \label{eq:optimal-sep-ccg-diff-cov}
\end{align}
where
\begin{align*}
    A &= \frac{1}{2}\left(\Tilde{\Sigma}_0^{-1} - \Tilde{\Sigma}_1^{-1}\right), \\
    b &= \Tilde{\Sigma}_1^{-1}\mu_1 - \Tilde{\Sigma}_0^{-1}\mu_0, \\
    c & = \frac{1}{2}\mu_0^\mathrm{T}\Tilde{\Sigma}_0^{-1}\mu_0 - \frac{1}{2}\mu_1^\mathrm{T}\Tilde{\Sigma}_1^{-1}\mu_1+ \frac{1}{2}\log\left(\frac{|\Tilde{\Sigma}_0|}{|\Tilde{\Sigma}_1|} \right)+ \log\left( \frac{p}{1-p}\right).
\end{align*}
By the law of total variance,
\begin{align}
    \text{Var}(\theta^\sigma(X^\sigma)) = \mathbb E[\text{Var}(\theta^\sigma(X^\sigma)|S)] + \text{Var}(\mathbb E[\theta^\sigma(X^\sigma)|S]).
    \label{eq:var-decomp}
\end{align}
From \cite[Thm. 3.2b.3]{mathai1992quadratic}, we have that for $s \in \{0,1\}$,
\begin{align}
    M_s &\coloneqq \mathbb E[\theta^\sigma(X^\sigma)|S=s] = \sum_{j=1}^d (\lambda_j^{(s)})+(c+b^\mathrm{T}\mu_s + \mu_s^\mathrm{T}A\mu_s), \\
    V_s &\coloneqq \text{Var}(\theta^\sigma(X^\sigma)|S=s)  = \sum_{j=1}^d 2(\lambda_j^{(s)})^2+ (u_j^{(s)})^2,
\end{align}
where $\lambda_j^{(s)}$, $j \in \{1,\dots,d\}$, are the eigenvalues of $\Tilde{\Sigma}_s^{1/2}A\Tilde{\Sigma}_s^{1/2}$ with corresponding eigenvectors as the columns of a matrix $Q_s$, i.e., $Q_s^\mathrm{T}\Tilde{\Sigma}_s^{1/2}A\Tilde{\Sigma}_s^{1/2}Q_s = \text{diag}(\lambda_1^{(s)},\dots,\lambda_d^{(s)})$, and
\begin{equation}
    u^{(s)} = (u_1^{(s)},\dots,u_d^{(s)})^\mathrm{T} = Q_s^\mathrm{T}(\Tilde{\Sigma}_s^{1/2}b+2\Tilde{\Sigma}_s^{1/2}A\mu_s).
\end{equation}
We can therefore compute the first term in \eqref{eq:var-decomp} as 
\begin{align}
    \mathbb E[\text{Var}(\theta^\sigma(X^\sigma)|S)] &= \mathbb E\left[SV_1+(1-S)V_0\right] = pV_1 + (1-p)V_0
\end{align}
and the second term as
\begin{align}
    \text{Var}(\mathbb E[\theta^\sigma(X^\sigma)|S]) &= \text{Var}(S)\left(M_1-M_0\right)^2 = p(1-p)\left(M_1-M_0\right)^2.
\end{align}
Therefore,
\begin{equation}
    \text{Var}(\theta^\sigma(X^\sigma)) = pV_1 + (1-p)V_0 + p(1-p)\left(M_1-M_0\right)^2.
    \label{eq:var-theta-sigma-vector-gaussian}
\end{equation}
Next, we compute $\text{Cov}(\theta^\sigma(X^\sigma),X^\sigma)$ as
\begin{align}
    \text{Cov}(\theta^\sigma(X^\sigma),X^\sigma) &= \text{Cov}(X^\mathrm{T}AX + 2\sigma Z^\mathrm{T}AX+\sigma^2 Z^\mathrm{T}AZ + b^\mathrm{T}X+\sigma b^\mathrm{T}Z ,X+\sigma Z) \\
    & = \text{Cov}(X^\mathrm{T}AX,X) + b^\mathrm{T}\text{Var}(X) + 2\sigma^2 \text{Cov}(Z^\mathrm{T}AX,Z) + \sigma^2 b^\mathrm{T}\text{Var}(Z) \\
    & = \text{Cov}(X^\mathrm{T}AX,X) + b^\mathrm{T}\text{Var}(X) + 2\sigma^2 \text{Cov}(Z^\mathrm{T}AX,Z) + \sigma^2 b^\mathrm{T}.
\end{align}
Note that
\begin{align}
    \text{Cov}(Z^\mathrm{T}AX,Z)  = \mathbb E[Z^\mathrm{T}AXZ^\mathrm{T}] = \mathbb E[X^\mathrm{T}AZZ^\mathrm{T}] = \mathbb E[X]^\mathrm{T} A \mathbb E[ZZ^\mathrm{T}] = [p\mu_1+(1-p)\mu_0]^\mathrm{T} A
\end{align}
and
\begin{align}
    \text{Var}(X) &= \mathbb E[\text{Var}(X|S)] + \text{Var}(\mathbb E[X|S]) \\
    & = \mathbb E[S\Sigma_1 + (1-S)\Sigma_0] + \text{Var}\left(S\mu_1+(1-S)\mu_0\right) \\
    & = p\Sigma_1 + (1-p)\Sigma_0 + p(1-p)(\mu_1-\mu_0)(\mu_1-\mu_0)^\mathrm{T}.
\end{align}
By the law of total covariance, we also have that
\begin{equation}
    \text{Cov}(X^\mathrm{T}AX,X) = \mathbb E[\text{Cov}(X^\mathrm{T}AX,X|S)] + \text{Cov}(\mathbb E[X^\mathrm{T}AX|S],\mathbb E[X|S]).
    \label{eq:total-cov}
\end{equation}
Focusing on the first term, let $\Tilde{X}_s \coloneqq (X|S=s) - \mu_s$ for $s \in \{0,1\}$. Then $\Tilde{X}_s \sim \mathcal{N}(0,\Sigma_s)$, and hence
\begin{align}
    \text{Cov}(X^\mathrm{T}AX,X|S=s) &= \text{Cov}\left((\Tilde{X}_s+\mu_s)^\mathrm{T}A(\Tilde{X}_s+\mu_s),\Tilde{X}_s+\mu_s\right) \\
    & = \text{Cov}\left(\Tilde{X}_s^\mathrm{T}A\Tilde{X}_s, \Tilde{X}_s\right) + 2\mu_s^\mathrm{T}A\text{Var}\left(\Tilde{X}_s\right) \\
    & = 2\mu_s^\mathrm{T}A\Sigma_s.
\end{align}
Therefore,
\begin{equation}
    \mathbb E[\text{Cov}(X^\mathrm{T}AX,X|S)] = 2p\mu_1^\mathrm{T}A\Sigma_1 + 2(1-p)\mu_0^\mathrm{T}A\Sigma_0.
\end{equation}
Now, focusing on the second term in \eqref{eq:total-cov}, from \cite[Thm. 3.2b.2]{mathai1992quadratic}, we have that
\begin{equation}
    \mathbb E[X^\mathrm{T}A X|S=s] = \text{tr}(A\Sigma_s) + \mu_s^\mathrm{T}A\mu_s.
\end{equation}
Therefore,
\begin{equation}
    \mathbb E[X^\mathrm{T}AX|S] = S\left(\text{tr}(A(\Sigma_1 - \Sigma_0)) + \mu_1^\mathrm{T}A\mu_1 - \mu_0^\mathrm{T}A\mu_0 \right) + \text{tr}(A\Sigma_0) + \mu_0^\mathrm{T}A\mu_0,
\end{equation}
and hence
\begin{align}
    \text{Cov}(\mathbb E[X^\mathrm{T}AX|S],\mathbb E[X|S]) & = \text{Cov}\left(S\left(\text{tr}(A(\Sigma_1 - \Sigma_0)) + \mu_1^\mathrm{T}A\mu_1 - \mu_0^\mathrm{T}A\mu_0 \right), S(\mu_1-\mu_0) \right) \\
    & = \left(\text{tr}(A(\Sigma_1 - \Sigma_0)) + \mu_1^\mathrm{T}A\mu_1 - \mu_0^\mathrm{T}A\mu_0 \right)p(1-p)(\mu_1\mu_0)^\mathrm{T}.
\end{align}
Thus, 
\begin{align}
    \text{Cov}(\theta^\sigma(X^\sigma),X^\sigma) & = 2p\mu_1^\mathrm{T}A\Sigma_1 + 2(1-p)\mu_0^\mathrm{T}A\Sigma_0 \nonumber \\
    & \quad + \left(\text{tr}(A(\Sigma_1 - \Sigma_0)) + \mu_1^\mathrm{T}A\mu_1 - \mu_0^\mathrm{T}A\mu_0 \right)p(1-p)(\mu_1-\mu_0)^\mathrm{T} \nonumber\\
    & \quad + b^\mathrm{T}\left[p\Sigma_1 + (1-p)\Sigma_0 + p(1-p)(\mu_1-\mu_0)(\mu_1-\mu_0)^\mathrm{T} \right] \nonumber\\
    & \quad + 2\sigma^2 \left[p\mu_1+(1-p)\mu_0\right]^\mathrm{T} A +\sigma^2b^\mathrm{T}.
    \label{eq:cov-term-vector-gaussian}
\end{align}
Note that $\text{Cov}(X^\sigma,\theta^\sigma(X^\sigma)) = \text{Cov}(\theta^\sigma(X^\sigma),X^\sigma)^\mathrm{T}$.
Finally, we have that
\begin{equation}
    \text{Var}(X^\sigma) = \text{Var}(X) + \sigma^2\text{Var}(Z) = p\Sigma_1 + (1-p)\Sigma_0 +p(1-p)(\mu_1-\mu_0)(\mu_1-\mu_0)^\mathrm{T} + \sigma^2 I.
    \label{eq:var-x-sigma-vector-gaussian}
\end{equation}
Combining \eqref{eq:var-theta-sigma-vector-gaussian}, \eqref{eq:cov-term-vector-gaussian} and \eqref{eq:var-x-sigma-vector-gaussian}, we obtain the result.

\subsection{Proof of Corollary \ref{cor:CCG-vector-diag-cov}}
\label{appendix:proof-CCG-vector-diag-cov}
When $\Sigma_s = \sigma_s^2 I$ for $s \in \{0,1\}$, \eqref{eq:optimal-sep-ccg-diff-cov} simplifies to 
\begin{align}
    \theta^\sigma(x) = x^\mathrm{T}Ax + b^\mathrm{T}x+c,
\end{align}
where
\begin{align*}
    A &= \frac{1}{2}\left((\sigma_0^2+\sigma^2)^{-1} - (\sigma_1^2 + \sigma^2)^{-1}\right)I = \frac{\sigma_1^2-\sigma_0^2}{2(\sigma_1^2+\sigma^2)(\sigma_0^2+\sigma^2)} I \eqqcolon aI, \\
    b &= (\sigma_1^2 + \sigma^2)^{-1}\mu_1 - (\sigma_0^2 + \sigma^2)^{-1}\mu_0 = \frac{(\sigma_0^2+\sigma^2)\mu_1-(\sigma_1^2+\sigma^2)\mu_0}{(\sigma_1^2+\sigma^2)(\sigma_0^2+\sigma^2)}, \\
    c & = \frac{1}{2}(\sigma_0^2 + \sigma^2)^{-1}\|\mu_0\|^2_2 - \frac{1}{2}(\sigma_1^2 + \sigma^2)^{-1}\|\mu_1\|^2_2+ \frac{1}{2}d\log\left(\frac{\sigma_0^2+\sigma^2}{\sigma_1^2+\sigma^2} \right)+ \log\left( \frac{p}{1-p}\right).
\end{align*}
Since $\Tilde{\Sigma}_s = (\sigma_s^2+\sigma^2)I$, $\Tilde{\Sigma}_s^{1/2} = \sqrt{\sigma_s^2+\sigma^2}I$, and hence \eqref{eq:cond-exp-ccg} simplifies to
\begin{align}
    M_0 &= \frac{d}{2}\left( \frac{\sigma_1^2-\sigma_0^2}{\sigma_1^2+\sigma^2}\right) + c + b^\mathrm{T}\mu_0 + \mu_0^\mathrm{T}A\mu_0 = \frac{d}{2}\left( \frac{\sigma_1^2-\sigma_0^2}{\sigma_1^2+\sigma^2}\right) + c + \frac{2(\sigma_0^2+\sigma^2)\mu_1^\mathrm{T}\mu_0-(\sigma_1^2+\sigma_0^2+2\sigma^2)\|\mu_0\|^2_2}{2(\sigma_1^2+\sigma^2)(\sigma_0^2+\sigma^2)}, \label{eq:cond-mean-class-0} \\
    M_1 & = \frac{d}{2}\left( \frac{\sigma_1^2-\sigma_0^2}{\sigma_0^2+\sigma^2}\right) + c + b^\mathrm{T}\mu_1 + \mu_1^\mathrm{T}A\mu_1 = \frac{d}{2}\left( \frac{\sigma_1^2-\sigma_0^2}{\sigma_0^2+\sigma^2}\right) + c + \frac{(\sigma_1^2+\sigma_0^2+2\sigma^2)\|\mu_1\|^2_2 - 2(\sigma_1^2+\sigma^2)\mu_1^\mathrm{T}\mu_0}{2(\sigma_1^2+\sigma^2)(\sigma_0^2+\sigma^2)},
    \label{eq:cond-mean-class-1}
\end{align}
for $S=0$ and $S=1$, respectively. We also have that
\begin{align}
    u^{(0)} &= \Tilde{\Sigma}_0^{1/2}(b + 2 A\mu_0) \nonumber \\
    & = \frac{\sqrt{\sigma_0^2+\sigma^2}}{\sigma_1^2+\sigma^2}(\mu_1-\mu_0), \\
    u^{(1)} &= \Tilde{\Sigma}_1^{1/2}(b + 2 A\mu_1) \nonumber \\
    & = \frac{\sqrt{\sigma_1^2+\sigma^2}}{\sigma_0^2+\sigma^2}(\mu_1-\mu_0).
\end{align}
Therefore, \eqref{eq:cond-var-ccg} simplifies to
\begin{align}
    V_0 &= \frac{d}{2} \left(\frac{\sigma_1^2-\sigma_0^2}{\sigma^2_1 + \sigma^2} \right)^2 + \frac{\sigma^2_0 + \sigma^2}{(\sigma^2_1 + \sigma^2)^2}\lVert \mu_1-\mu_0 \rVert_2^2 = \frac{d(\sigma_1^2-\sigma_0^2)^2+2(\sigma_0^2+\sigma^2)\lVert \mu_1-\mu_0 \rVert_2^2}{2(\sigma_1^2+\sigma^2)^2}, \label{eq:cond-var-class-0} \\
    V_1 &= \frac{d}{2} \left(\frac{\sigma_1^2-\sigma_0^2}{\sigma^2_0 + \sigma^2} \right)^2 + \frac{\sigma^2_1 + \sigma^2}{(\sigma^2_0 + \sigma^2)^2}\lVert \mu_1-\mu_0 \rVert_2^2 = \frac{d(\sigma_1^2-\sigma_0^2)^2+2(\sigma_1^2+\sigma^2)\lVert \mu_1-\mu_0 \rVert_2^2}{2(\sigma_0^2+\sigma^2)^2},
    \label{eq:cond-var-class-1}
\end{align}
for $S= 0$ and $S=1$, respectively.
Substituting \eqref{eq:cond-mean-class-0}, \eqref{eq:cond-mean-class-1}, \eqref{eq:cond-var-class-0}, and \eqref{eq:cond-var-class-1} into \eqref{eq:var-theta-sig-ccg} yields
\begin{align}
    \text{Var}(\theta^\sigma(X^\sigma)) &= p\left(\frac{d(\sigma_1^2-\sigma_0^2)^2+2(\sigma_1^2+\sigma^2)\lVert \mu_1-\mu_0 \rVert_2^2}{2(\sigma_0^2+\sigma^2)^2} \right) + (1-p)\left(\frac{d(\sigma_1^2-\sigma_0^2)^2+2(\sigma_0^2+\sigma^2)\lVert \mu_1-\mu_0 \rVert_2^2}{2(\sigma_1^2+\sigma^2)^2} \right) \\
    & \qquad + p(1-p)\left(\frac{d(\sigma_1^2-\sigma_0^2)^2 + (\sigma_1^2+\sigma_0^2+2\sigma^2)\lVert \mu_1-\mu_0 \rVert_2^2}{2(\sigma_1^2+\sigma^2)(\sigma_0^2+\sigma^2)} \right)^2.
\end{align}
Making appropriate substitutions, \eqref{eq:cov-theta-x-ccg} simplifies to
\begin{align}
    \text{Cov}(\theta^\sigma(X^\sigma),X^\sigma) & = a\left[ 2p(\sigma_1^2+\sigma^2)\mu_1^\mathrm{T}+ 2(1-p)(\sigma_0^2+\sigma^2)\mu_0^\mathrm{T} +p(1-p)(d(\sigma_1^2-\sigma_0^2)+\lVert\mu_1\lVert_2^2 - \lVert \mu_0 \lVert_2^2)(\mu_1-\mu_0)^\mathrm{T}\right] \nonumber \\
    & \quad + b^\mathrm{T}\left[\alpha I +p(1-p)(\mu_1-\mu_0)(\mu_1-\mu_0)^\mathrm{T} \right] \\
    & = \gamma_1 \mu_1^\mathrm{T} + \gamma_0 \mu_0^\mathrm{T}.
\end{align}
where 
\begin{equation}
    \alpha \coloneqq p\sigma_1^2 + (1-p)\sigma_0^2 +\sigma^2
    \label{eq:alpha-def}
\end{equation} and
\begin{align}
    \gamma_1 &\coloneqq a\left[ 2p(\sigma_1^2+\sigma^2) +p(1-p)(d(\sigma_1^2-\sigma_0^2)+\lVert\mu_1\lVert_2^2 - \lVert \mu_0 \lVert_2^2)\right]+\frac{\alpha}{\sigma_1^2+\sigma^2} \nonumber\\
    & \qquad +p(1-p)\frac{(\sigma_0^2+\sigma^2)\lVert\mu_1\lVert_2^2 - (\sigma_1^2+\sigma_0^2+2\sigma^2)\mu_1^\mathrm{T}\mu_0 + (\sigma_1^2+\sigma^2)\lVert\mu_0\lVert_2^2}{(\sigma_1^2+\sigma^2)(\sigma_0^2+\sigma^2)}, \\
    \gamma_0 &\coloneqq a\left[ 2(1-p)(\sigma_0^2+\sigma^2) -p(1-p)(d(\sigma_1^2-\sigma_0^2)+\lVert\mu_1\lVert_2^2 - \lVert \mu_0 \lVert_2^2)\right]-\frac{\alpha}{\sigma_0^2+\sigma^2} \nonumber \\
    & \qquad -p(1-p)\frac{(\sigma_0^2+\sigma^2)\lVert\mu_1\lVert_2^2 - (\sigma_1^2+\sigma_0^2+2\sigma^2)\mu_1^\mathrm{T}\mu_0 + (\sigma_1^2+\sigma^2)\lVert\mu_0\lVert_2^2}{(\sigma_1^2+\sigma^2)(\sigma_0^2+\sigma^2)}.
\end{align}
Again making the appropriate substitutions, \eqref{eq:var-x-sig-ccg} reduces to
\begin{align}
    \text{Var}(X^\sigma) = \alpha I +p(1-p)(\mu_1-\mu_0)(\mu_1-\mu_0)^\mathrm{T},
\end{align}
for $\alpha$ defined in \eqref{eq:alpha-def}. Using the Sherman-Morrison formula, we can compute $\text{Var}^{-1}(X^\sigma)$ as follows:
\begin{align}
   \text{Var}^{-1}(X^\sigma) &= \alpha^{-1}I - \frac{\alpha^{-1}p(1-p)(\mu_1-\mu_0)(\mu_1-\mu_0)^\mathrm{T}}{\alpha + p(1-p)\lVert \mu_1-\mu_0\lVert_2^2}.
\end{align}
Therefore,
\begin{align}
    \text{Cov}(\theta^\sigma(X^\sigma),X^\sigma)\text{Var}^{-1}(X^\sigma)\text{Cov}(X^\sigma,\theta^\sigma(X^\sigma)) & = \alpha^{-1}(\gamma_1^2 \lVert \mu_1 \lVert_2^2 + 2\gamma_1\gamma_0\mu_1^\mathrm{T}\mu_0 + \gamma_0^2\lVert \mu_0\lVert_2^2) \nonumber \\
    & \quad - \frac{\alpha^{-1}p(1-p)(\gamma_1 \lVert \mu_1 \lVert_2^2 + (\gamma_0-\gamma_1)\mu_1^\mathrm{T}\mu_0 - \gamma_0\lVert \mu_0\lVert_2^2)^2}{\alpha+p(1-p)\lVert \mu_1-\mu_0\lVert_2^2},
\end{align}
and hence
\begin{align}
    \text{Var}(\theta^\sigma(X^\sigma))-\text{Cov}(\theta^\sigma(X^\sigma),X^\sigma)\text{Var}^{-1}(X^\sigma)\text{Cov}(X^\sigma,\theta^\sigma(X^\sigma)) = \frac{(\sigma_1^2-\sigma_0^2)^2(q_1 + q_2\sigma^2+q_3\sigma^4+2d\sigma^6)}{4(r_1+r_2\sigma^2+r_3\sigma^4+r_4\sigma^6+r_5\sigma^8+\sigma^{10})},
\end{align}
where
\begin{align*}
    q_1 & = \lVert \mu_1 - \mu_0 \rVert_2^4\left(p^2(1-p)\sigma_1^2 +p(1-p)^2\sigma_0^2 \right) - d^2p^3(\sigma_1^2-\sigma_0^2)^3+2p(1-p)(2+d)\sigma_0^2\sigma_1^2\lVert \mu_1 - \mu_0 \rVert_2^2 \nonumber \\
    & \quad + p^2\left(d(5d-2)\sigma_0^4\sigma_1^2 - 2d(2d+1)\sigma_0^2\sigma_1^4+d(2+d)\sigma_1^6 - 2d(d-1)\sigma_0^6\right) \nonumber \\
    & \quad +p\left(d(d-4)\sigma_0^6 - 2d(d-1)\sigma_0^4\sigma_1^2 + d(2+d) \sigma_0^2\sigma_1^4 \right) + 2d\sigma_0^6, \\
    q_2 & = p(1-p)\lVert \mu_1 - \mu_0 \rVert_2^2\left(\lVert \mu_1 - \mu_0 \rVert_2^2 + 2(2+d)(\sigma_0^2+\sigma_1^2) \right) - p^2d(d-4)(\sigma_1^2-\sigma_0^2)^2 \nonumber \\
    & \quad +p\left(d(d-10)\sigma_0^4-2d(d-4)\sigma_0^2\sigma_1^2+d(2+d)\sigma_1^4 \right) + 6d\sigma_0^4, \\
    q_3 & = 2p(1-p)(2+d)\lVert \mu_1 - \mu_0 \rVert_2^2+6d(p\sigma_1^2+(1-p)\sigma_0^2), \\
    r_1 & = \sigma_0^4\sigma_1^4\left(p(1-p)\lVert \mu_1 - \mu_0 \rVert_2^2+p\sigma_1^2+(1-p)\sigma_0^2 \right), \\
    r_2 & = \sigma_0^2\sigma_1^2 \left(2p(1-p)(\sigma_1^2+\sigma_0^2)\lVert \mu_1 - \mu_0 \rVert_2^2 +2p\sigma_1^4 +2(1-p)\sigma_0^4 + 3\sigma_0^2\sigma_1^2 \right), \\
    r_3 & = p(1-p)(\sigma_0^4 + 4\sigma_0^2\sigma_1^2 + \sigma_1^4)\lVert \mu_1 - \mu_0 \rVert_2^2 + p\sigma_1^6 + 3(1+p)\sigma_0^2\sigma_1^4 + 3(2-p)\sigma_0^4\sigma_1^2+ (1-p)\sigma_0^6, \\
    r_4 & = 2p(1-p)(\sigma_1^2+\sigma_0^2)\lVert \mu_1 - \mu_0 \rVert_2^2 + (2p+1)\sigma_1^4 + (3-2p)\sigma_0^4+6\sigma_0^2\sigma_1^2, \\
    r_5 & = p(1-p)\lVert \mu_1 - \mu_0 \rVert_2^2 + (3-p)\sigma_0^2 + (2+p)\sigma_1^2.
\end{align*}

\subsection{Additional Plots}
\label{sec:additional_plots}
We include several additional plots in this section to illustrate further observations.
In \cref{fig:ccg-noise-density-progression}, we illustrate how the optimal $\theta^\sigma$ for the class-conditional Gaussian setting of \cref{cor:CCG-vector-diag-cov} reduces in curvature from a quadratic to a linear function with increasing $\sigma$. \cref{fig:bsc_pn} shows a plot of the training MSE-based bounds for a linear model as a function of the crossover probability $p_n$ in the BSC setting described in \cref{thm:delta-a-bsc-example}. In \cref{fig:epsilon_a_dimension}, we illustrate how the approximation error for both linear and neural network hypothesis classes varies with the data dimension $d$ in the class-conditional Gaussian setting. Finally, \cref{fig:linear-vs-nn-bounds-dimension} shows training-based bounds for linear and neural network hypothesis classes as a function of the data dimension $d$, evaluated across different hidden layer widths $d_w$ for the neural network. In \cref{fig:xor-sample-3-modes-radius-1,fig:xor-sample-4-modes-radius-1}, we show visualizations of $X$ and the corresponding noisy $X^\sigma$ from our class-conditional mixture dataset for $n_m=3$ and $n_m=4$, respectively, using a smaller radius of 1. The corresponding results for the training MSE– and validation MSE–based lower bounds are presented in \cref{fig:xor_train_r_1,fig:xor_val_r_1}, respectively. Despite the reduced separability in this setting, the results remain qualitatively similar to those observed with radius 2.

\begin{figure}[h]
    \centering
    \begin{subfigure}{0.33\textwidth}
        \centering
        \includegraphics[width=\textwidth,page=2]{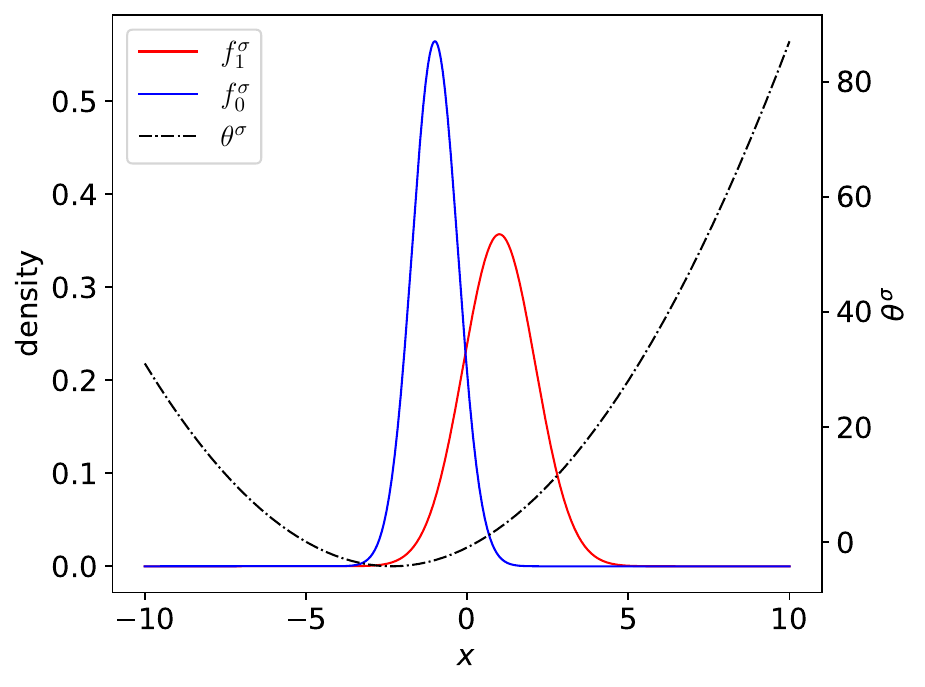}
        \caption{$\sigma = 0.5$}
    \end{subfigure}\hfill\hfill%
    \begin{subfigure}{0.33\textwidth}
        \centering
        \includegraphics[width=\textwidth,page=3]{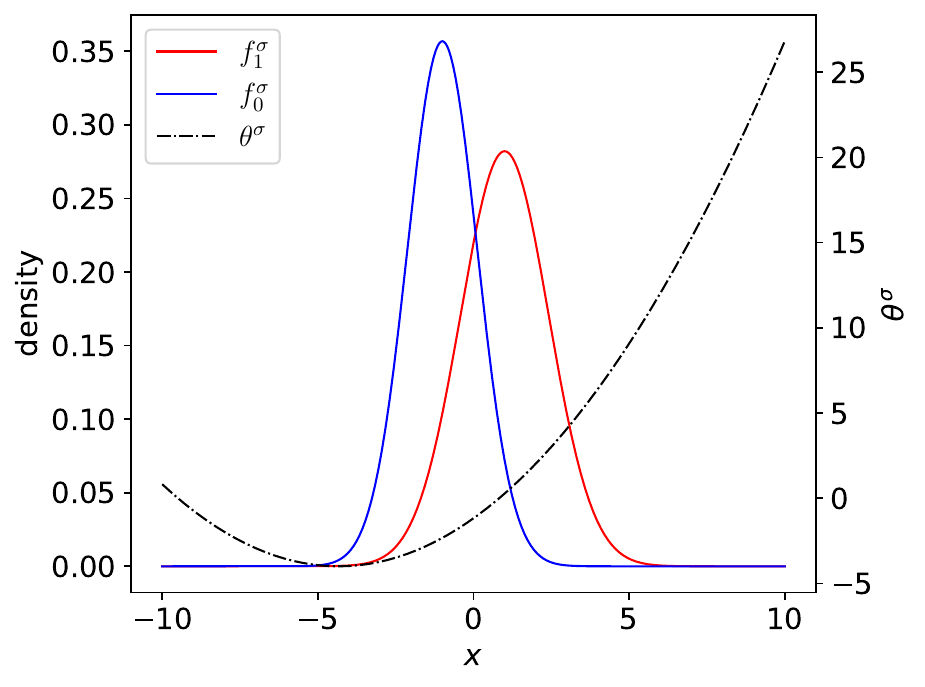}
        \caption{$\sigma = 1$}
    \end{subfigure}\hfill\hfill%
    \begin{subfigure}{0.33\textwidth}
        \centering
    \includegraphics[width=\textwidth,page=6]{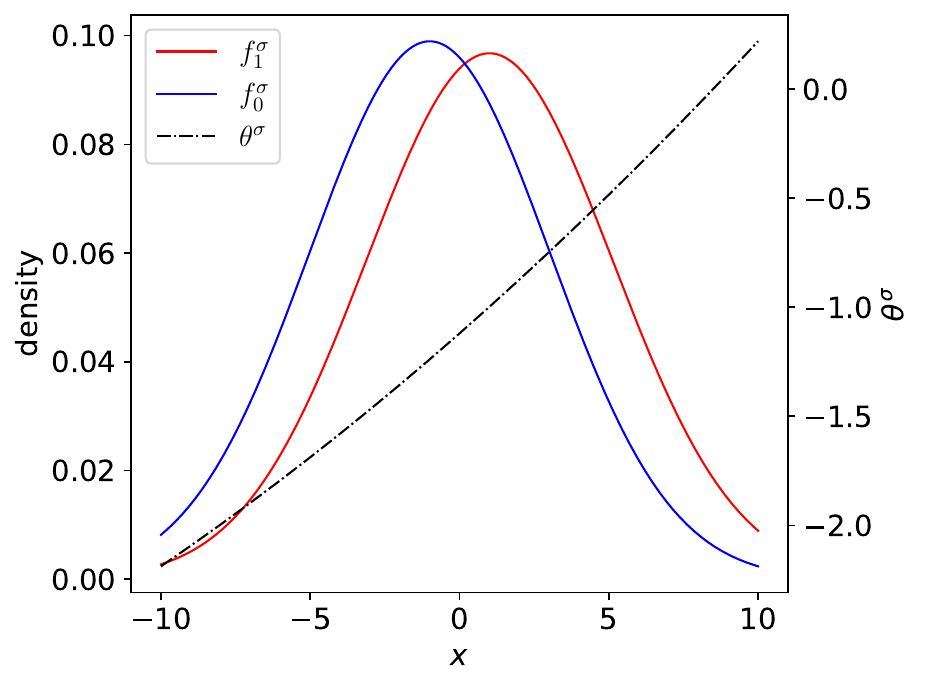}
        \caption{$\sigma = 4$}
    \end{subfigure}
    \caption{An illustration of how the class-conditional densities $f_i^\sigma$
  of $X^\sigma\vert S=i$, for $i\in\{0,1\}$, and the optimal estimator $\theta^\sigma$ evolve in the class-conditional Gaussian setting of \cref{cor:CCG-vector-diag-cov}, as the noise level $\sigma$ increases. Here, $p=1/4$, $\mu_0=-1$, $\mu_1=1$, $\sigma_0=0.5$, $\sigma_1=1$, and $d=1$. As $\sigma$ increases, the variances of the conditional distributions become more alike, leading $\theta^\sigma$
  to exhibit increasingly linear behavior and ultimately converge to a linear function in the limit.}
    \label{fig:ccg-noise-density-progression}
\end{figure}

\begin{figure}[h]
    \centering
    \includegraphics[width=0.5\linewidth]{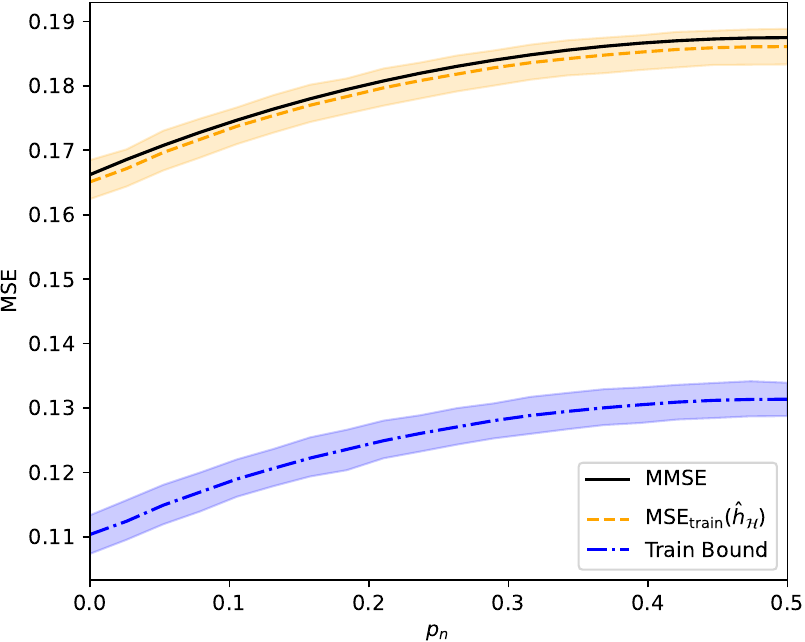}
    \caption{Training MSE-based bounds for a linear adversarial evaluation model as a function of the crossover probability $p_n$
  in the binary symmetric channel setting of \cref{thm:delta-a-bsc-example}. The bounds remain relatively tight across the full range of 
$p_n$ values.}
    \label{fig:bsc_pn}
\end{figure}

\begin{figure}[h]
    \centering
    \begin{subfigure}{0.48\linewidth}
        
        \includegraphics[width=\linewidth]{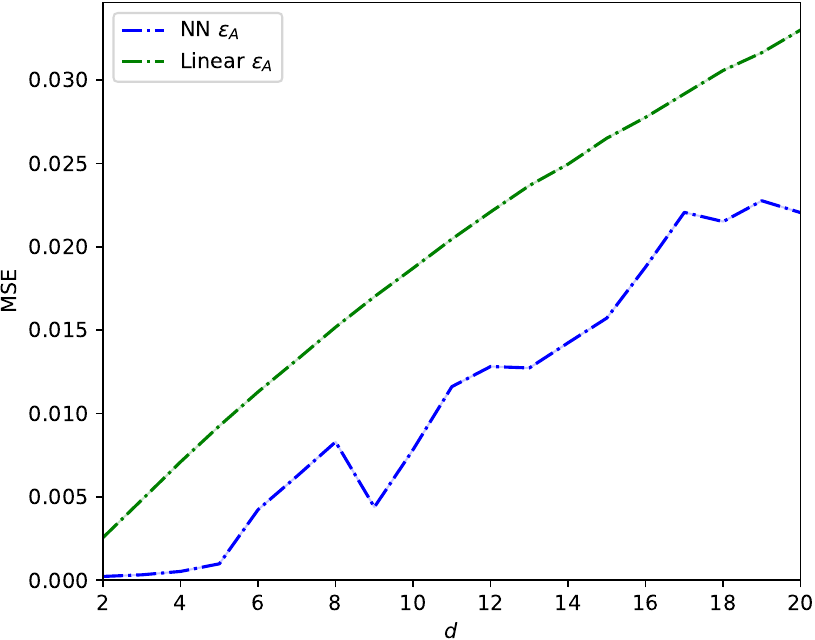}
        \caption{$d_w=10$}
    \end{subfigure}\hfill%
    \begin{subfigure}{0.48\linewidth}
        
        \includegraphics[width=\linewidth]{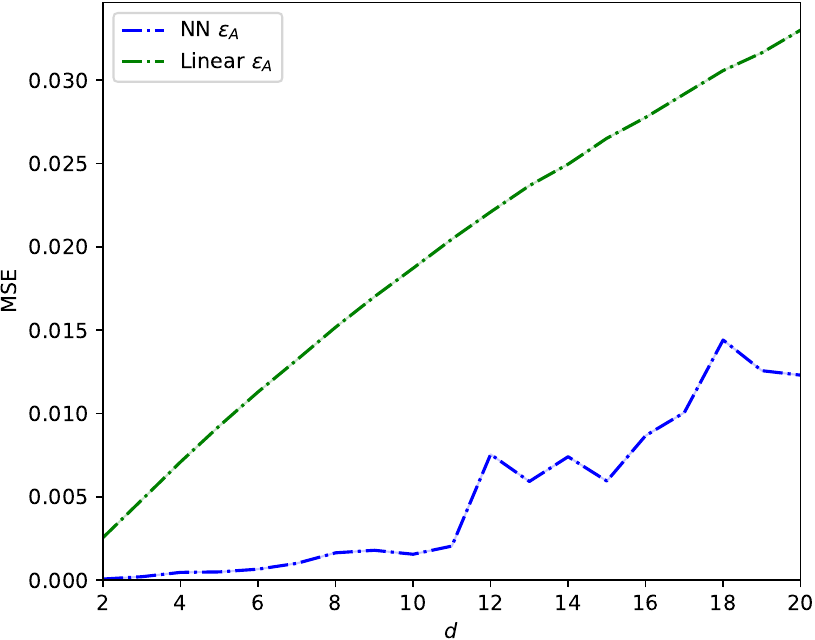}
        \caption{$d_w=20$}
    \end{subfigure}
    \caption{Empirical $\epsilon_A$ for the class-conditional Gaussian dataset for linear models and neural networks. For both hypothesis classes, the approximation error grows with the dimensionality of the dataset but grows more slowly for neural networks.}
    \label{fig:epsilon_a_dimension}
\end{figure}
\begin{figure}[h]
    \centering
    \begin{subfigure}{0.3\linewidth}
    \includegraphics[width=\linewidth]{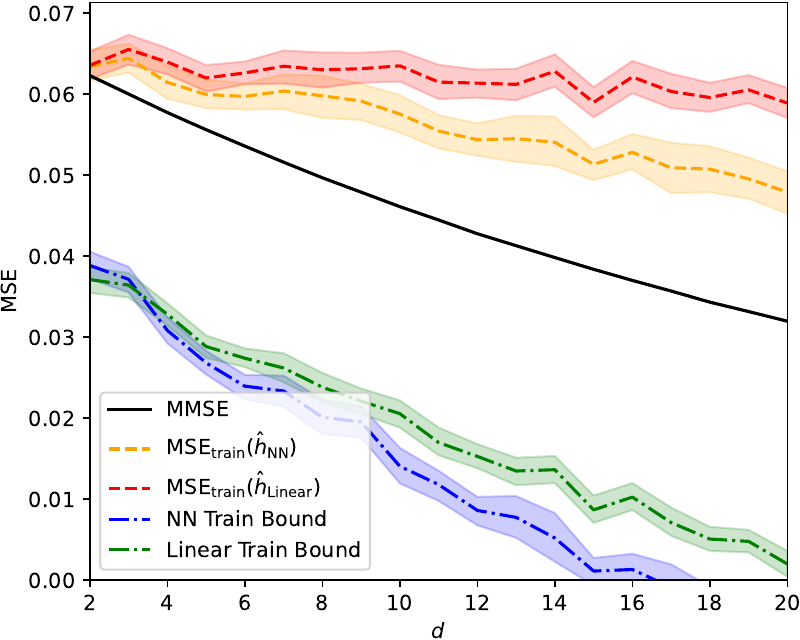}
    \caption{$d_w=2$}
    \label{fig:ccg_h2}
    \end{subfigure}\hfill%
    \begin{subfigure}{0.3\linewidth}
    \includegraphics[width=\linewidth]{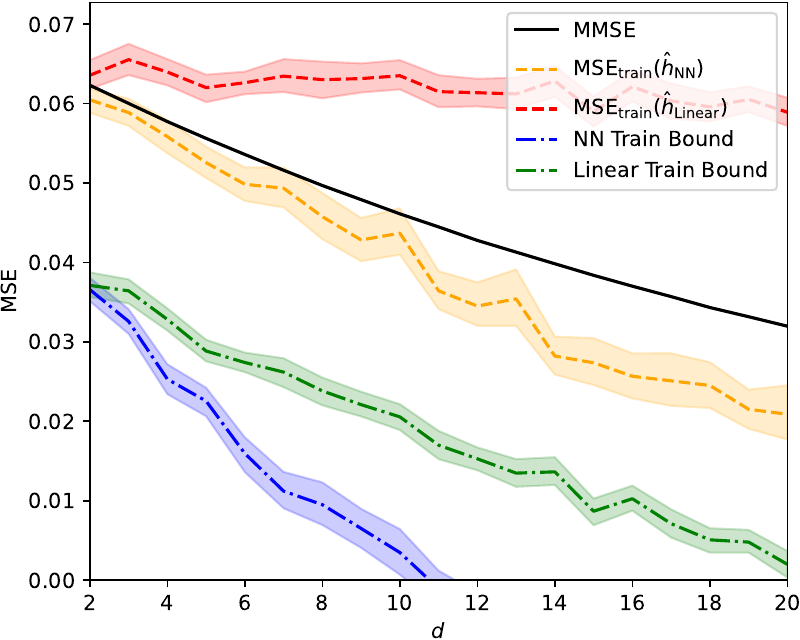}
    \caption{$d_w=5$}
    \label{fig:ccg_h5}
    \end{subfigure}\hfill%
    \begin{subfigure}{0.3\linewidth}
    \includegraphics[width=\linewidth]{figures/ccg_d_nn_vs_linear_h_10.pdf}
    \caption{$d_w=10$}
    \label{fig:ccg_h10}
    \end{subfigure}
    \caption{Training-based bounds for linear and neural network hypothesis classes as a function of the data dimension $d$, shown for various hidden layer widths $d_w$ of the neural network for the class-conditional Gaussian setting of \cref{cor:CCG-vector-diag-cov}. The bounds for the neural network hypothesis class degrade as the hidden layer width $d_w$ increases, primarily due to overfitting at higher data dimensions $d$. {The bounds are evaluated using $n=1$k samples of $(X^\sigma,S)$ and the parameters $p=1/4, \sigma_0^2 = 1, \sigma_1^2 = 3, \sigma=1$, and 
    $\lVert \mu_1 - \mu_0 \rVert^2 = 4$.} Although increasing the hidden layer width $d_w$ reduces the approximation error $\epsilon_A$, as shown in \cref{fig:epsilon_a_dimension}, the resulting overfitting leads to a much lower training MSE and ultimately looser lower bounds.}
    \label{fig:linear-vs-nn-bounds-dimension}
\end{figure}

\begin{figure}
    \centering
    \begin{subfigure}{0.48\linewidth}
    \centering
    \includegraphics[width=\linewidth]{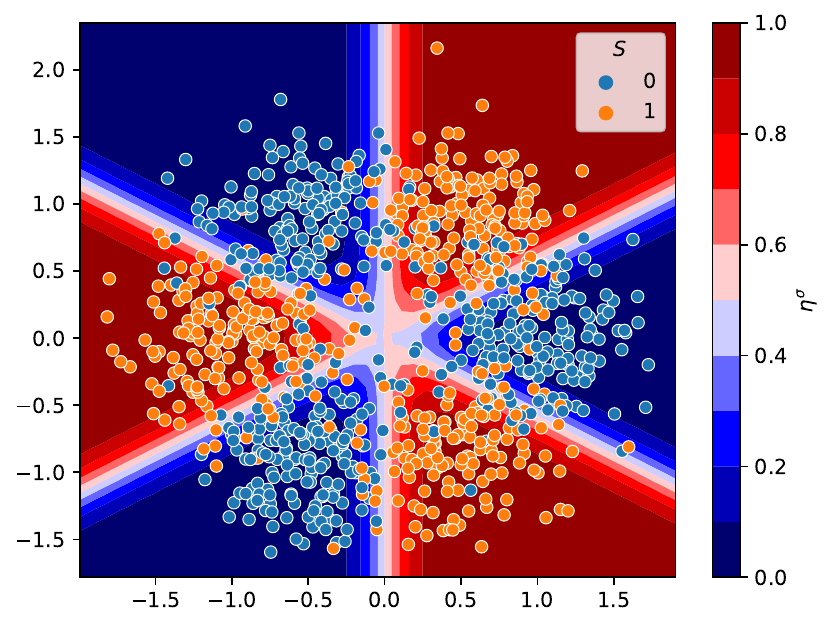}
    \caption{$X$}
    \label{fig:xor-sample-3-no-noise-radius-1}
    \end{subfigure}\hfill\hfill%
    \begin{subfigure}{0.48\linewidth}
    \centering
    \includegraphics[width=\linewidth]{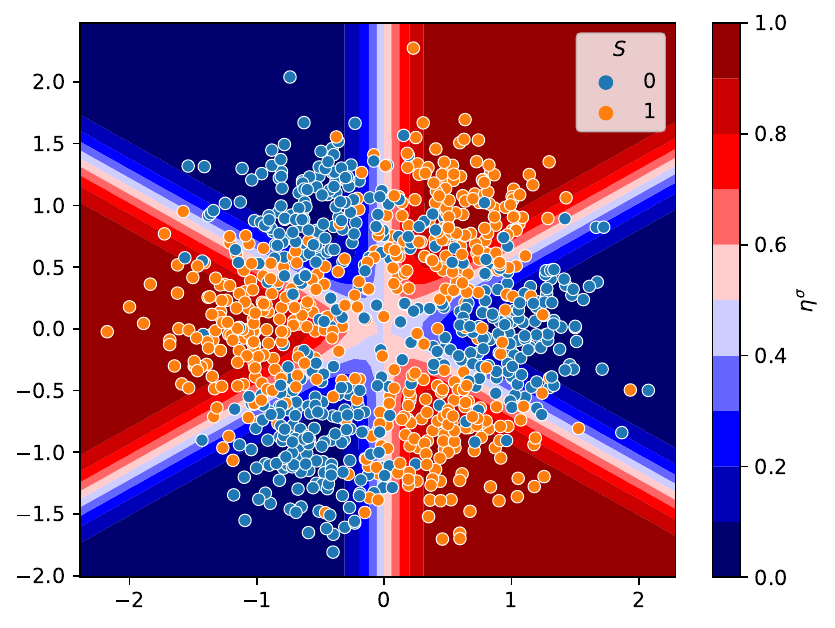}
    \caption{$X^\sigma$}
    \label{fig:xor-sample-3-noise-radius-1}
    \end{subfigure}
    \caption{A sample of (a) $X$ and (b) $X^\sigma$ with $\sigma = 0.5$ from our class-conditional mixture dataset with 3 modes per class, where the modes are placed on a circle of radius 1.
    The heat maps (i.e., contours) of $\eta^\sigma$ (the color legend is shown to the right of the figure) are derived from the true statistics of the data.}
    \label{fig:xor-sample-3-modes-radius-1}
\end{figure}

\begin{figure}
    \centering
    \begin{subfigure}{0.48\linewidth}
    \centering
    \includegraphics[width=\linewidth]{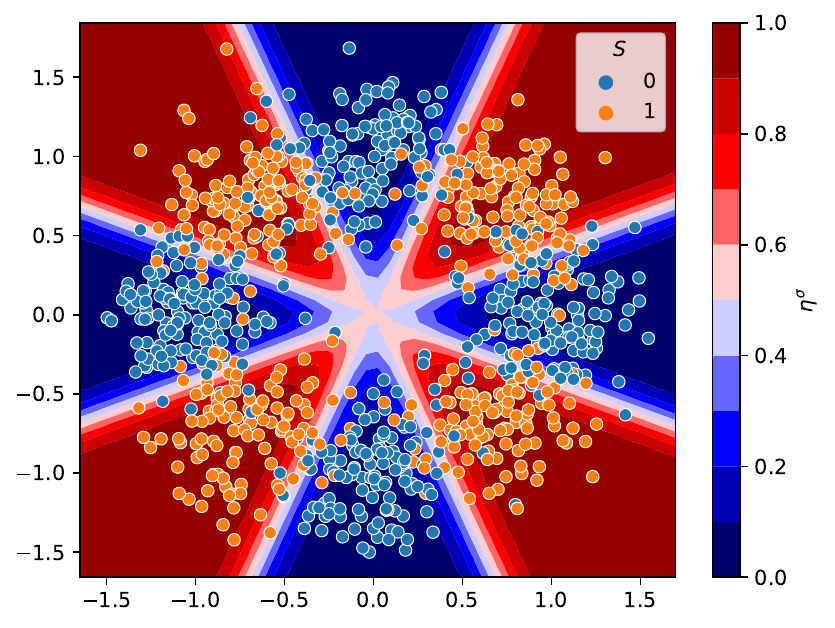}
    \caption{$X$}
    \label{fig:xor-sample-4-no-noise-radius-1}
    \end{subfigure}\hfill\hfill%
    \begin{subfigure}{0.48\linewidth}
    \centering
    \includegraphics[width=\linewidth]{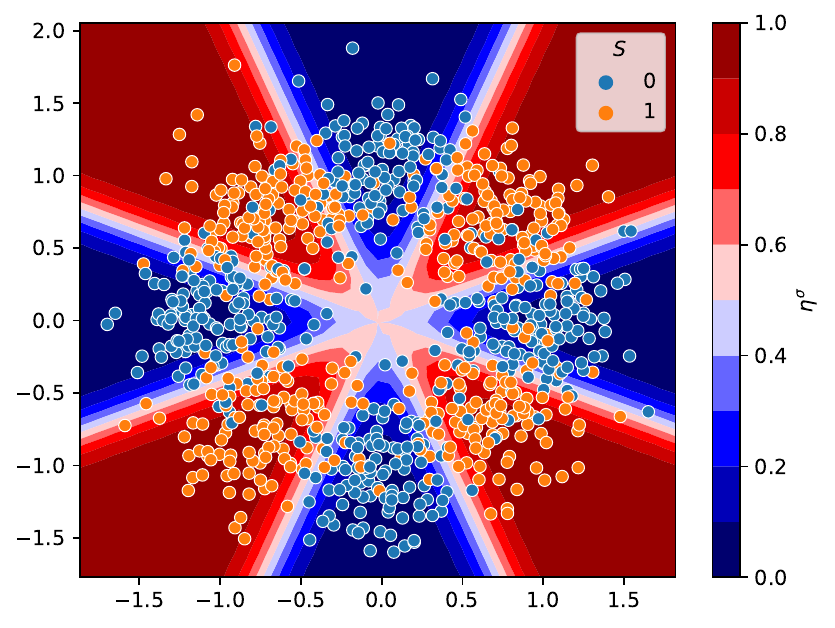}
    \caption{$X^\sigma$}
    \label{fig:xor-sample-4-noise-radius-1}
    \end{subfigure}
    \caption{A sample of (a) $X$ and (b) $X^\sigma$ with $\sigma = 0.5$ from our class-conditional mixture dataset with 4 modes per class, where the modes are placed on a circle of radius 1.
    The heat maps (i.e., contours) of $\eta^\sigma$ (the color legend is shown to the right of the figure) are derived from the true statistics of the data.}
    \label{fig:xor-sample-4-modes-radius-1}
\end{figure}
\begin{figure}[h]
    \centering
    \begin{subfigure}[t]{0.48\linewidth}
    \includegraphics[width=\linewidth]{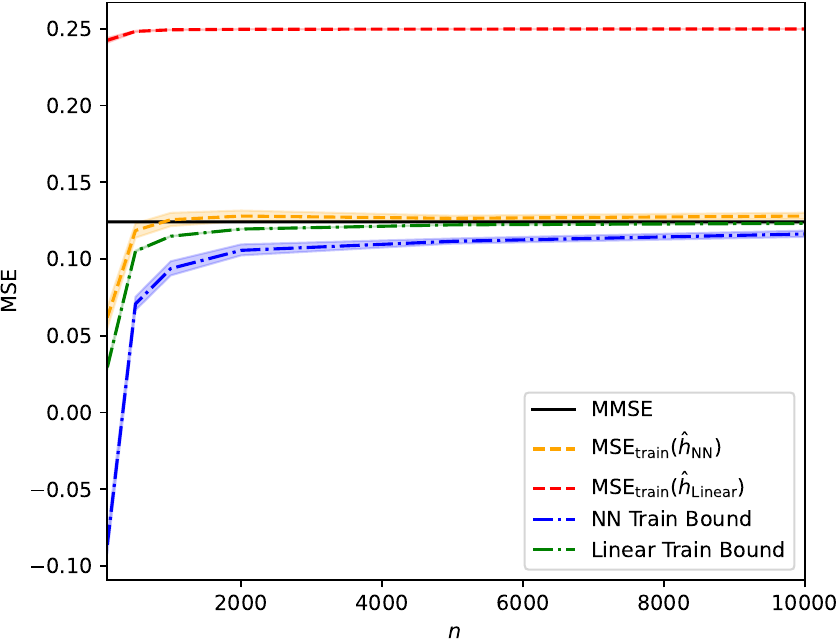}
    \caption{3 modes per class.}
    \label{fig:xor_train_3_r_1}
    \end{subfigure}\hfill\hfill%
    \begin{subfigure}[t]{0.48\linewidth}
    \includegraphics[width=\linewidth]{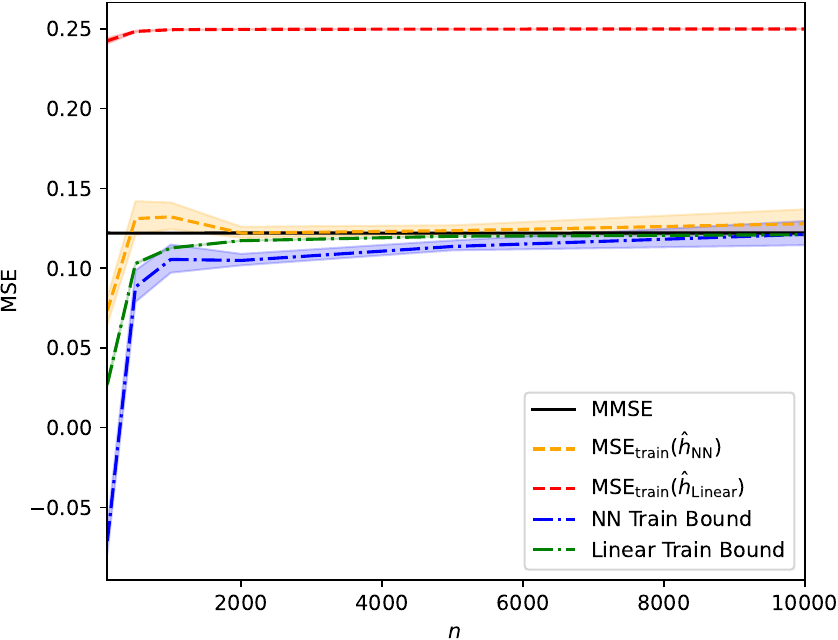}
    \caption{4 modes per class.}
    \label{fig:xor_train_4_r_1}
    \end{subfigure}
    \caption{Training MSE-based lower bounds for the class-conditional mixture of Gaussians dataset for radius 1, $\sigma=0.5$, and $d_w = 10$. For both settings ($n_m=3$ and $n_m=4$), linear models yield tighter bounds for small $n$, despite having a significantly larger approximation error $\epsilon_A$. Neural networks, while initially affected by overfitting, exhibit much smaller approximation error and achieve competitive bounds as $n$ increases.}
    \label{fig:xor_train_r_1}
\end{figure}
\begin{figure}[h]
    \centering
    \begin{subfigure}[t]{0.48\linewidth}
    \includegraphics[width=\linewidth]{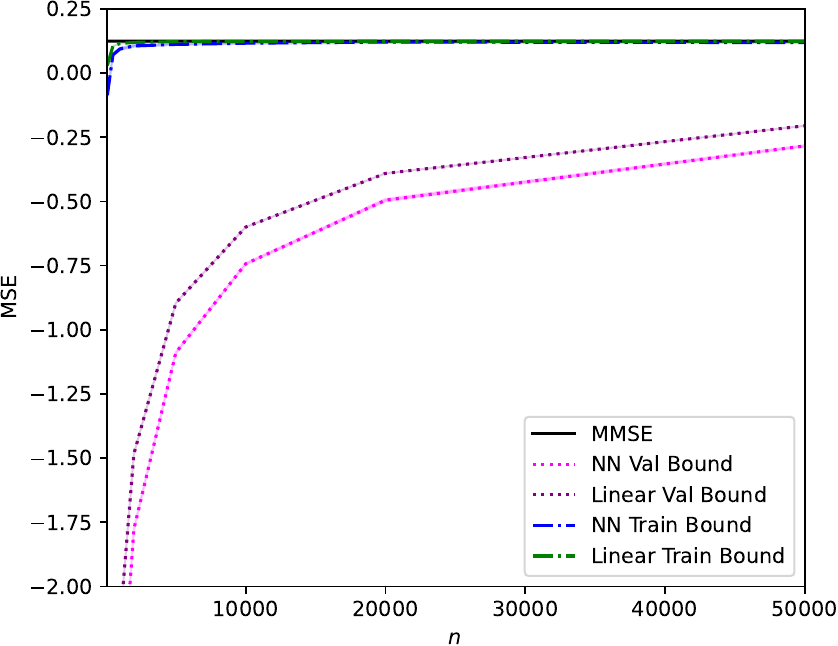}
    \caption{3 modes per class.}
    \label{fig:xor_val_3_r_1}
    \end{subfigure}\hfill\hfill%
    \begin{subfigure}[t]{0.48\linewidth}
    \includegraphics[width=\linewidth]{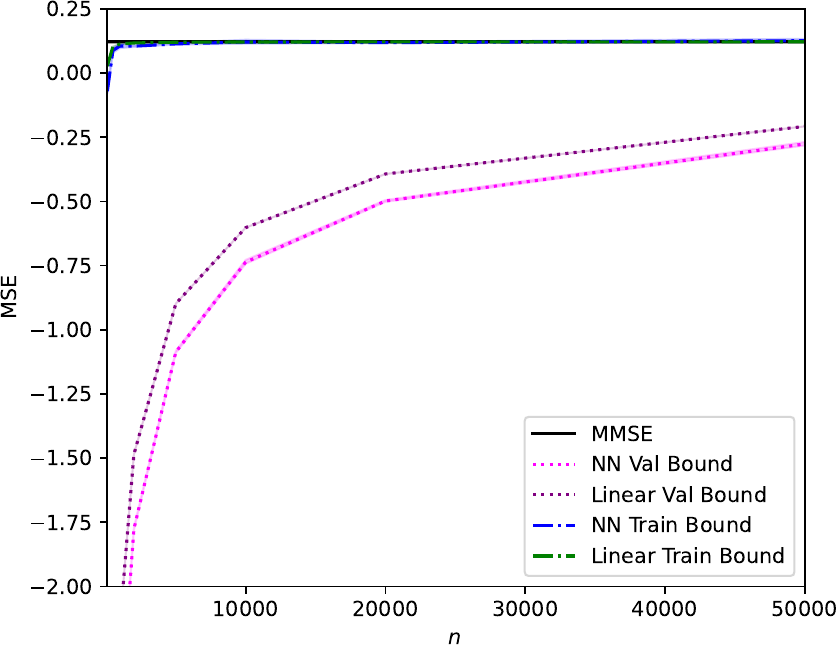}
    \caption{4 modes per class.}
    \label{fig:xor_val_4_r_1}
    \end{subfigure}
    \caption{Comparison of training MSE- and validation MSE-based lower bounds for the class-conditional mixture of Gaussians dataset with radius 1, $\sigma=0.5$ and the neural network hidden-layer width $d_w = 10$. Results are shown for both linear models and shallow neural networks. Despite using as many as $n=50$k training and $m=1$k validation samples,  the validation bounds remain vacuous due to the looseness of the generalization error bound on $\epsilon_G$ in \eqref{eq:eps-g-bound}, which accounts for the complexity of the hypothesis class.}
    \label{fig:xor_val_r_1}
\end{figure}

\subsection{Experimental Details}
\label{appendix:experimental_setup}
See \cref{tab:nn,tab:logistic} for hyperparameter details for the training of each model type.

\begin{table}[h]
\parbox{.45\linewidth}{
\centering
    \begin{tabular}{c|c}
        \hline
        Parameter & Value \\
        \hline
        Output Activation & Sigmoid \\ 
        \hline
        Optimizer & AdamW \\ 
        $\beta_1$ & 0.9 \\
        $\beta_2$ & 0.999 \\
        Weight Decay & 0.01 \\
        Learning Rate & 0.1 \\ 
        LR Scheduler & cosine \\ 
        Epochs & 5000 \\ 
        \hline
    \end{tabular}
    \caption{Hyperparameters for logistic model training.}
    \label{tab:logistic}
}
\hfill
\parbox{.45\linewidth}{
\centering
    \begin{tabular}{c|c}
        \hline
        Parameter & Value \\
        \hline
        Hidden Activation & ReLU \\ 
        Output Activation & Sigmoid \\ 
        $d_w$ & \{ 2, 5, 10, 20 \} \\
        \hline
        Optimizer & AdamW \\ 
        $\beta_1$ & 0.9 \\
        $\beta_2$ & 0.999 \\
        Weight Decay & 0.01 \\
        Learning Rate & 0.01 \\ 
        LR Scheduler & cosine \\ 
        CCG Epochs & 5000 \\ 
        XOR Epochs & 10000 \\ 
        \hline
    \end{tabular}
    \caption{Hyperparameters for neural network training.}
    \label{tab:nn}
}
\end{table}

\end{document}